%% file: neurips_2026.tex
\documentclass{article}

\PassOptionsToPackage{numbers,compress}{natbib}
\usepackage[preprint]{neurips_2026}
\usepackage{float}
\usepackage[utf8]{inputenc}
\usepackage[T1]{fontenc}
\usepackage{hyperref}
\usepackage{url}
\usepackage{booktabs}
\usepackage{amsfonts}
\usepackage{nicefrac}
\usepackage{microtype}
\usepackage{pifont}
\usepackage{xcolor}
\usepackage{graphicx}
\usepackage{subcaption}
\usepackage{adjustbox}
\usepackage{amsmath}
\usepackage{amssymb}
\usepackage{mathtools}
\usepackage{amsthm}
\usepackage{algorithm}
\usepackage{algorithmic}
\usepackage{makecell}
\usepackage{enumitem}
\usepackage{multirow}
\usepackage{array}
\usepackage{pifont}
\usepackage[capitalize,noabbrev]{cleveref}
\usepackage{wrapfig} 
\floatstyle{ruled}
\restylefloat{algorithm}
\usepackage[table]{xcolor}   
\usepackage{booktabs, colortbl, amssymb, graphicx}
\usepackage{wrapfig}
\newcommand{\colhead}[1]{\textbf{#1}}
\newcommand{\grouplab}[1]{\rotatebox[origin=c]{90}{\bfseries #1}}

\newcolumntype{L}[1]{>{\centering\arraybackslash}p{#1}}
\DeclareMathOperator*{\argmin}{arg\,min}

\theoremstyle{plain}
\newtheorem{theorem}{Theorem}[section]

\newtheorem{lemma}[theorem]{Lemma}

\theoremstyle{definition}
\newtheorem{definition}[theorem]{Definition}
\newtheorem{assumption}[theorem]{Assumption}
\theoremstyle{remark}

\title{SPOT: Selective Prompt Projection via Total Variation for Inference-Only Safe Text-to-Image Generation}
\author{
  Minhyuk Lee$^*$ \\
  Seoul National University \\
  \texttt{minhyuk@snu.ac.kr}
  \And
  Hyekyung Yoon$^*$ \\
  Seoul National University \\
  \texttt{yhk04150@snu.ac.kr}
  \And
  Myungjoo Kang$^\dagger$ \\
  Seoul National University \\
  \texttt{mkang@snu.ac.kr}
  \thanks{Equal contribution. \quad $^{\dagger}$Corresponding author.}
}

\begin{document}
\maketitle

\begin{abstract}
Text-to-Image (T2I) diffusion models enable high quality open ended synthesis, but practical use requires suppressing unsafe generations while preserving behavior on benign prompts. We study this tension relative to the frozen generator, using its prompt conditioned distribution as the preservation reference. Since T2I safety is commonly evaluated by bounded risk scores on generated images, total variation (TV) bounds how much expected risk can change from this reference. We call this fixed reference constraint the \textbf{S}afety--\textbf{P}rompt \textbf{A}lignment \textbf{T}radeoff (\textbf{SPAT}): reducing expected unsafety requires prompt conditioned distributional deviation. To make this deviation selective and adjustable, we define the $\tau$ safe set as prompts whose reference risk is at most $\tau$, and cast intervention as projection toward nearby prompts in this set. We propose \textbf{S}elective \textbf{P}rompt pr\textbf{O}jec\textbf{T}ion (\textbf{SPOT}), an inference time framework that approximates this projection without retraining the generator or learning a category specific rewriter. SPOT uses an LLM to rank candidate rewrites and a safeguard VLM to accept generated images under the same $\tau$. Across four datasets and three diffusion backbones, SPOT achieves relative inappropriate (IP) score reductions from 14.2\% to 44.4\% over strong safety alignment baselines while keeping benign prompt behavior close to the fixed reference.

\textcolor{red}{WARNING: This paper contains model outputs that may be offensive in nature.}
\end{abstract}

\section{Introduction}

Recent Text-to-Image (T2I) diffusion models have greatly improved photorealism and prompt following, enabling broad open ended image generation~\citep{podellsdxl,rombach2022high}, but also creating safety risks where users can elicit harmful or policy violating imagery~\citep{bird2023typology,dalle}. Prior work has made significant efforts toward safety alignment~\citep{2023ESDU,2024UCE,2024latentguard,2025alignguard,2023SLD}. Yet global model intervention can perturb behavior even on safe prompts~\citep{2025alignguard}, motivating selective safety alignment that reduces unsafety only where intervention is needed.

We use the frozen T2I generator as the preservation reference, denoting its prompt conditioned image distribution by $G^*(\cdot\mid c)$. Preserving benign prompts means staying close to this reference, while improving safety means lowering normalized risk scores on generated samples, as is common in T2I safety evaluation~\citep{Nudenet2022,schramowski2022can}. Since these scores are bounded expectations over generated images, total variation (TV) bounds how much expected risk can change from the reference. We call this fixed reference constraint the \textbf{S}afety--\textbf{P}rompt \textbf{A}lignment \textbf{T}radeoff (\textbf{SPAT}): reducing expected unsafety requires prompt conditioned distributional deviation. The key question is therefore how to localize this deviation to risky prompts.

To give selective deviation an adjustable target, we define the $\tau$ safe set as prompts whose reference risk under $G^*$ is at most $\tau$. The threshold $\tau$ sets the operating safety level, with smaller values imposing stricter safety and larger values relaxing the constraint. This gives a prompt projection view: leave safe prompts unchanged and map risky prompts to nearby $\tau$ safe prompts. We study this view at inference time, so changing $\tau$, datasets, or backbones requires no generator retraining or category specific rewriter.

We propose \textbf{S}elective \textbf{P}rompt pr\textbf{O}jec\textbf{T}ion (\textbf{SPOT}) to approximate this projection efficiently. Exact projection onto the $\tau$ safe set is intractable because membership depends on expected image level risk under $G^*$. SPOT therefore approximates hard projection over finite LLM proposed candidates, balancing prompt proximity against violation of the same $\tau$. A second practical issue is that $\tau$ safety is defined over generated images, reliable verification should use a safeguard VLM, but checking every rewrite candidate is costly. SPOT therefore separates cheap search from final verification: a prompt only LLM ranks candidate rewrites, while the safeguard VLM accepts or rejects the generated image. This design avoids exhaustive VLM evaluation while preserving final image based verification.
\vspace{-2mm}
\begin{figure}[t]

  \centering
  \includegraphics[width=\linewidth, height=5cm, keepaspectratio]{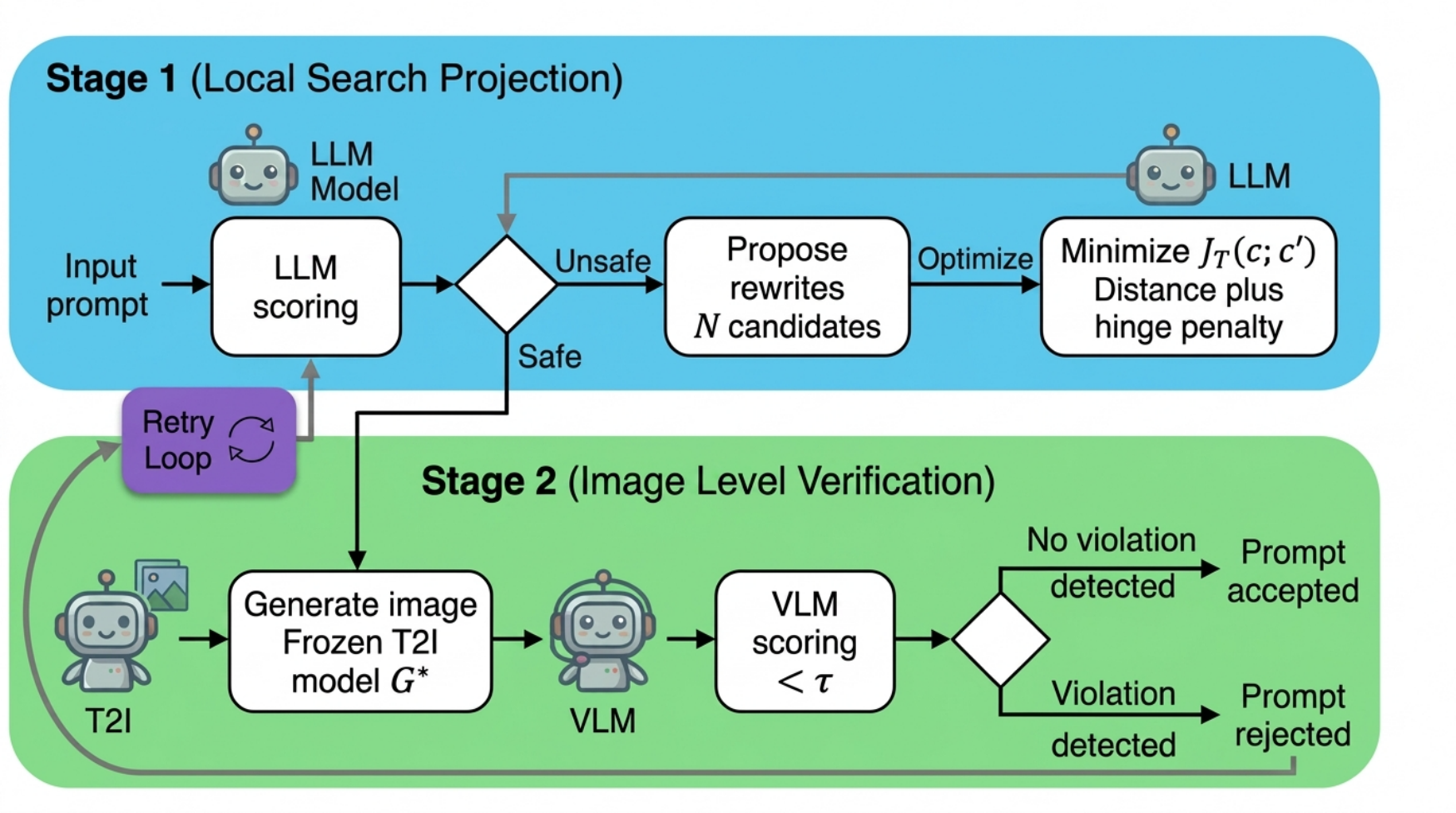}
  \caption{Overview of SPOT depicting two stage cascade protocol.}
  \label{fig:spot}
\end{figure}

\vspace{-2mm}
Our contributions are as follows:
\begin{itemize}[leftmargin=1.2em]
\item \textbf{Reference-based safety tradeoff.} We formalize SPAT relative to the fixed reference $G^*$, showing that reducing population unsafety below the reference safety level requires prompt conditioned distributional deviation. This motivates selective intervention only on risky prompts.
\item \textbf{Selective prompt projection.} We define a $\tau$ safe set and formulate intervention as projecting risky prompts toward nearby prompts in this set while preserving safe prompts.
\item \textbf{Inference time LLM and VLM cascade.} We instantiate SPOT with an LLM router for candidate rewrite ranking and a safeguard VLM for final acceptance of generated images.
\item \textbf{Strong safety and benign preservation.} Across four datasets and three diffusion backbones, SPOT achieves 14.2\% to 44.4\% relative IP reductions over strong baselines while preserving COCO benign alignment near the unaligned reference.
\end{itemize}

\vspace{-2mm}
\section{Related Work}

\newcommand{\cyes}{\textcolor[RGB]{0,150,0}{\ding{51}}}
\newcommand{\cno}{\textcolor[RGB]{200,0,0}{\ding{55}}}

\vspace{-2mm}
\paragraph{Prompt side moderation and rewriting.}
A growing line of work improves T2I safety through prompts or text embeddings while keeping the diffusion backbone fixed. Earlier methods use keyword blocklists~\citep{midjourney,azure2024}, learned prompt classifiers such as LLM guards~\citep{llamaguard2023}, or guardrail pipelines with prompt revision~\citep{dalle3}. More recent methods learn prompt or embedding interventions: LatentGuard~\citep{2024latentguard} learns an auxiliary latent space above the text encoder, PromptGuard~\citep{promptguard2025} optimizes a universal safety soft prompt, POSI~\citep{2024Posi} finetunes an unsafe prompt rewriter, and GuardT2I~\citep{2024guardt2i} decodes text encoder latents into plain text for moderation. These methods preserve the diffusion backbone, but most decisions are made before observing the final generated image. Image feedback methods such as IPR~\citep{ipr} and VALOR~\citep{zhao2025value} address this by verifying generated images during refinement, but repeatedly invoke the diffusion backbone during search. This motivates separating cheap rewrite ranking from final verification on the generated image.

\paragraph{Latent and denoising interventions.}
A second line of work suppresses unsafe content inside the generation process while keeping model weights fixed. SLD~\citep{2023SLD} and UCE~\citep{2024UCE} modify cross attention or latent representations to suppress harmful concepts, SAFREE~\citep{2024safree} adapts denoising based on toxic trigger tokens, and PNO~\citep{2024PNO} optimizes text embeddings and diffusion noise trajectories at inference time. These methods can be effective without additional training, but they often rely on fixed concept taxonomies or intervention rules. Because they perturb the denoising process directly, it is difficult to localize deviation only to risky prompts or to preserve safe prompts by design. This leaves open the need for a selective intervention rule that is tied to the input prompt and can be adjusted by a safety threshold.

\paragraph{Generator editing and alignment.}
Another direction enforces safety by modifying the generator itself. ESD-u~\citep{2023ESDU} edits pretrained weights to erase prohibited concepts, and AlignGuard~\citep{2025alignguard} learns category specific safety modules and injects them into the denoising backbone. Such approaches can achieve strong suppression, but they require additional training, often depend on predefined taxonomies or category specific optimization, and change the generator globally. This conflicts with the fixed reference goal: even safe prompts may move away from the original generator behavior. It also makes policy changes costly, since a new safety threshold, dataset, or backbone may require additional tuning or retraining.

\paragraph{Our positioning.}
The above limitations motivate an inference time method whose safety level can be changed by a threshold rather than by retraining, while keeping the generator frozen and verifying safety on the generated image.

\begin{wraptable}{r}{0.55\linewidth}
\caption{Novelty matrix comparing the closest representative safety methods.}
\label{tab:novelty-matrix}
\centering
\scriptsize
\setlength{\tabcolsep}{3pt}
\renewcommand{\arraystretch}{0.95}
\resizebox{\linewidth}{!}{%
\begin{tabular}{lcccccc}
\toprule
& \multicolumn{6}{c}{\textbf{Property}} \\
\cmidrule{2-7}
\textbf{Method}
& \makecell{No Extra\\Training}
& \makecell{Frozen\\T2I}
& \makecell{Prompt\\Rewrite}
& \makecell{Generated\\Image Check}
& \makecell{Adjustable\\Safety Target}
& \makecell{TV\\Tradeoff} \\
\midrule
AlignGuard  & \cno  & \cno  & \cno  & \cno  & \cno  & \cno \\
PromptGuard & \cno  & \cyes & \cno  & \cno  & \cno  & \cno \\
POSI        & \cno  & \cyes & \cyes & \cno  & \cno  & \cno \\
IPR         & \cno  & \cyes & \cyes & \cyes & \cno  & \cno \\
VALOR       & \cyes & \cyes & \cyes & \cyes & \cno  & \cno \\
\midrule
\rowcolor{gray!20}
\textbf{SPOT}
            & \cyes & \cyes & \cyes & \cyes & \cyes & \cyes \\
\bottomrule
\end{tabular}%
}
\end{wraptable}

SPOT follows this design by localizing intervention to risky prompts rather than changing the generator globally. It operates at inference time, keeps the T2I generator frozen, and maps risky prompts toward the $\tau$ safe set, where $\tau$ controls the operating safety level. Unlike iterative image feedback methods such as IPR~\citep{ipr} and VALOR~\citep{zhao2025value}, SPOT separates cheap rewrite ranking from final verification on the generated image. Its TV analysis explains why safety gains require reference deviation, and why this deviation should be localized to risky prompts. Table~\ref{tab:novelty-matrix} summarizes this positioning across representative prompt side, inference time, and generator editing baselines.

\section{Reference-Based Theory for Selective Prompt Projection}
\label{sec:theory}

This section explains why our safety intervention is formulated as selective prompt projection.
A T2I generator is stochastic: for the same prompt, different noise seeds can produce different
images. Thus, its behavior on prompt $c$ is naturally represented by a prompt  conditioned image
distribution $G(\cdot\mid c)$ rather than by a single output. We use the fixed  
$G^*$ as the preservation reference, i.e., the prompt conditioned behavior we aim to keep unchanged
whenever no safety intervention is needed. Safety, on the other hand, is commonly evaluated by
image level detector scores or binary inappropriate content indicators~\citep{Nudenet2022,schramowski2022can}, which can be normalized to
bounded scores in $[0,1]$. This makes expected unsafety a natural risk functional under the
prompt conditioned image distribution.
\subsection{SPAT Under a Fixed Reference}
\label{subsec:spat-fixed}\label{subsec:preliminaries}

Let prompts $c\in\mathcal{C}$ be drawn from a population measure $\mu$, and let
$G(\cdot\mid c)\in\mathcal{P}(\mathcal{X})$ denote the post-intervention prompt conditioned image
distribution on the image space $(\mathcal{X},\mathcal{B}(\mathcal{X}))$. We compare it to the
fixed reference distribution $G^*(\cdot\mid c)$ using total variation:
\begin{equation}
\mathrm{TV}(\nu,\nu'):=\sup_{A\in\mathcal{B}(\mathcal{X})} |\nu(A)-\nu'(A)|.
\end{equation}
TV is useful here because it controls the change of any bounded unsafety score. Let
$U:\mathcal{X}\to[0,1]$ be an unsafety score. We define prompt-wise and population unsafety as
\begin{align}
u(G\mid c) := \mathbb{E}_{x\sim G(\cdot\mid c)}[U(x)],\quad
\mathcal{U}(G) := \mathbb{E}_{c\sim\mu}[u(G\mid c)].
\label{eq:unsafety-functionals}
\end{align}
We also define the average deviation from the fixed reference by
\begin{equation}
\mathcal{A}_{\mathrm{TV}}(G)
:=\mathbb{E}_{c\sim\mu}\!\left[
\mathrm{TV}\bigl(G(\cdot\mid c),\,G^*(\cdot\mid c)\bigr)
\right].
\label{eq:alignment-functional}
\end{equation}

The following bound formalizes the basic \textbf{S}afety--\textbf{P}rompt \textbf{A}lignment \textbf{T}radeoff (\textbf{SPAT}): if unsafety is
measured as a bounded expectation over generated images, then reducing that expectation below the
reference level requires distributional deviation from the reference.

\begin{theorem}[Safety--prompt alignment tradeoff]
\label{thm:spat-maintext}
For any conditional generator $G$ and any prompt $c\in\mathcal{C}$,
\begin{equation}
u(G\mid c)\ \ge\ u(G^*\mid c)\;-\;\mathrm{TV}\!\bigl(G(\cdot\mid c),\,G^*(\cdot\mid c)\bigr).
\label{eq:spat-pointwise}
\end{equation}
Consequently,
\begin{equation}
\mathcal{U}(G)+\mathcal{A}_{\mathrm{TV}}(G)\ \ge\ \mathcal{U}(G^*).
\label{eq:spat-bound}
\end{equation}
\end{theorem}

\begin{proof}
See Appendix~\ref{app:spat}.
\end{proof}

Theorem~\ref{thm:spat-maintext} is intentionally elementary, but it gives the design implication we
need. If a target risk level $\tau$ is desired and an intervention achieves $u(G\mid c)\le\tau$, then
\begin{equation}
\mathrm{TV}\!\bigl(G(\cdot\mid c),G^*(\cdot\mid c)\bigr)
\ge
\bigl[u(G^*\mid c)-\tau\bigr]_+ .
\label{eq:required-deviation}
\end{equation}
Thus, prompts whose reference risk is already below $\tau$ require no distributional movement,
whereas prompts above $\tau$ require some deviation. The goal is therefore not to avoid deviation
entirely, but to spend it only on prompts that need safety intervention.

\subsection{$\tau$ Safe Projection and Projected Targets}
\label{subsec:tau-proj}\label{subsec:mk-projection}

Equation~\eqref{eq:required-deviation} motivates the prompt set on which no intervention is needed.
For a tolerance $\tau\in[0,1]$, define the $\tau$-safe prompt set under the fixed reference:
\begin{equation}
\mathcal{C}_{\mathrm{safe},\tau}
:=\{c\in\mathcal{C}: u(G^*\mid c)\le \tau\}.
\label{eq:tau-safe-prompt-space}
\end{equation}
If $c\in\mathcal{C}_{\mathrm{safe},\tau}$, the fixed reference already satisfies the target safety level on $c$.
For $c\notin\mathcal{C}_{\mathrm{safe},\tau}$, intervention is needed. However, among candidates that fall inside $\mathcal{C}_{\mathrm{safe},\tau}$, the safety score alone cannot determine which one minimizes semantic drift from the original prompt.
We therefore seek a nearby $\tau$ safe prompt. We use an embedding-induced angular
distance as a tractable proxy for prompt preservation:
\begin{equation}
d(c,c')
:=\arccos\!\Big(
\frac{\langle \phi(c),\phi(c')\rangle}
{\|\phi(c)\|_2\,\|\phi(c')\|_2}
\Big)\in[0,\pi],
\label{eq:angular-metric}
\end{equation}

where $\phi$ is a sentence embedding map
\citep{gao2021simcse,izacard2021unsupervised,karpukhin2020dense,radford2021learning}. The ideal
nearest $\tau$ safe prompt set is 
\begin{equation}
\mathcal{P}_\tau(c)
:=\argmin_{c'\in\mathcal{C}_{\mathrm{safe},\tau}} d(c,c').
\label{eq:prompt-projection-set}
\end{equation}
Because $d$ is a pseudometric on prompt text, the nearest $\tau$ safe prompt may not be unique.
This naturally leads to a prompt conditional distribution over $\mathcal{P}_\tau(c)$, approximated in practice by finite LLM proposed candidates.
We therefore represent projection by a Markov kernel $\Pi_\tau(c,\cdot)$ rather than by a single deterministic map. The ideal projection kernel satisfies
\begin{align}
&\Pi_\tau(c,\cdot) = \delta_c(\cdot), \qquad c\in\mathcal{C}_{\mathrm{safe},\tau},
\label{eq:identity}\\
&\Pi_\tau(c,\mathcal{P}_\tau(c)) = 1, \qquad \mu\text{-a.e.\ }c\in\mathcal{C}.
\label{eq:idempotent}
\end{align}
The first condition encodes no intervention on already safe prompts. The second encodes minimal
safe rewriting for prompts that require intervention. Sufficient regularity conditions for such kernels
are given in Appendix~\ref{app:idempotent-kernel}.

After projection, the preservation target should also be evaluated at the projected prompt. For an unsafe prompt $c$, leaving $G^*(\cdot\mid c)$ unchanged would retain the very unsafe behavior that projection aims to reduce. Instead,
if $c$ is replaced by nearest $\tau$ safe prompts selected by $\Pi_\tau$, the natural reference target
is the fixed reference conditioned on those prompts. Given $\Pi_\tau$, define
\begin{align}
\widetilde{G}_{\tau,\Pi_\tau}(\cdot\mid c)
&:= \int_{\mathcal{C}_{\mathrm{safe},\tau}} G(\cdot\mid c')\,\Pi_\tau(c,dc'),\\
G^{\mathrm{ref}}_{\tau,\Pi_\tau}(\cdot\mid c)
&:= \int_{\mathcal{C}_{\mathrm{safe},\tau}} G^*(\cdot\mid c')\,\Pi_\tau(c,dc').
\label{eq:kernel-mixtures-maintext}
\end{align}
Here $G^{\mathrm{ref}}_{\tau,\Pi_\tau}$ is the projected reference: the generator remains fixed, but
the conditioning prompt is replaced only through the projection kernel. We measure deviation from
this projected target by
\begin{equation}
\mathcal{A}^{\Pi_\tau}_{\mathrm{TV}}(G)
:= \mathbb{E}_{c\sim\mu}\!\left[
\mathrm{TV}\bigl(
\widetilde{G}_{\tau,\Pi_\tau}(\cdot\mid c),\,
G^{\mathrm{ref}}_{\tau,\Pi_\tau}(\cdot\mid c)
\bigr)
\right].
\label{eq:kernel-alignment-maintext}
\end{equation}

\begin{theorem}[Projected reference SPAT]
\label{thm:kernel-spat-maintext}
Fix $\tau\in[0,1]$ and let $\Pi_\tau$ satisfy
\eqref{eq:identity}--\eqref{eq:idempotent}. Then for any conditional generator $G$,
\begin{equation}
\mathcal{U}\!\bigl(\widetilde{G}_{\tau,\Pi_\tau}\bigr)
+\mathcal{A}^{\Pi_\tau}_{\mathrm{TV}}(G)
\ \ge\
\mathcal{U}\!\bigl(G^{\mathrm{ref}}_{\tau,\Pi_\tau}\bigr),
\label{eq:kernel-spat-ineq}
\end{equation}
and moreover
\begin{equation}
\mathcal{U}\!\bigl(G^{\mathrm{ref}}_{\tau,\Pi_\tau}\bigr)\le \tau.
\label{eq:kernel-floor}
\end{equation}
\end{theorem}

\begin{proof}
See Appendix~\ref{app:tau-cotrol}.
\end{proof}

Theorem~\ref{thm:kernel-spat-maintext} specifies the reference target induced by prompt projection.
The second inequality follows from the support of $\Pi_\tau$ on $\mathcal{C}_{\mathrm{safe},\tau}$;
its role is to make explicit that projection changes the preservation target from the original
reference $G^*(\cdot\mid c)$ to a $\tau$ controlled projected reference. SPOT approximates this
projected reference with finite LLM proposed rewrites and a safeguard VLM that verifies the realized image
under the same threshold $\tau$.

\vspace{-2mm}
\section{\textbf{SPOT}: Selective Prompt Projection at Inference Time}
\label{sec:method}
\vspace{-2mm}
\subsection{Projection, Generation, and Verification}
\label{subsec:overall_architecture}
\vspace{-2mm}

\begin{wraptable}{r}{0.52\textwidth}
\begin{minipage}{0.52\textwidth}
\vspace{-3mm}
\hrule height 1.2pt
\vspace{0.3ex}
\captionof{algorithm}{\textsc{SPOT}: Selective Prompt Projection}
\label{alg:spot}
\vspace{0.3ex}
\hrule height 0.6pt
\vspace{0.5ex}

{\scriptsize
\begin{algorithmic}[1]
\STATE \textbf{Input:} prompt $c$; tolerance $\tau$; steps $T$; neighbors $N$; attempts $R$
\STATE \textbf{Output:} final pair $(c^\star,\hat{x})$
\STATE $c^\star \leftarrow c$, \quad $\hat{x}\leftarrow \varnothing$
\FOR{$r=1,2,\ldots,R$}
    \STATE \textbf{Stage 1:} local search projection
    \STATE $c_0 \leftarrow c$, \quad $c^\star \leftarrow c$
    \FOR{$t=0,1,\ldots,T-1$}
        \STATE Sample $\{c_t^{(n)}\}_{n=1}^{N}$ from an LLM conditioned on $c_t$
        \STATE $\mathcal{C}_t \leftarrow \{c_t\}\cup\{c_t^{(n)}\}_{n=1}^{N}$
        \STATE $c_{t+1} \leftarrow \argmin_{c'\in\mathcal{C}_t} J_\tau(c;c')$
        \STATE $c^\star \leftarrow c_{t+1}$
        \IF{$\widehat{u}_{\mathrm{LLM}}(c^\star)\le\tau$}
            \STATE \textbf{break}
        \ENDIF
    \ENDFOR
    \STATE \textbf{Stage 2:} image level verification
    \STATE Draw $\hat{x}\sim G^*(\cdot\!\mid\!c^\star)$
    \STATE Compute $\widehat{u}_{\mathrm{VLM}}(\hat{x})$
    \IF{$\widehat{u}_{\mathrm{VLM}}(\hat{x})\le\tau$}
        \STATE \textbf{return} $(c^\star,\hat{x})$
    \ENDIF
\ENDFOR
\STATE \textbf{return} $(c^\star,\hat{x})$
\end{algorithmic}
}

\vspace{0.5ex}
\hrule height 0.6pt
\vspace{-5mm}
\end{minipage}
\end{wraptable}

Theorem~\ref{thm:kernel-spat-maintext} motivates an operational procedure that moves risky prompts toward $\mathcal{C}_{\mathrm{safe},\tau}$, leaves prompts that are already safe unchanged, and keeps the fixed reference. We propose \textbf{SPOT} (\textbf{S}elective \textbf{P}rompt pr\textbf{O}jec\textbf{T}ion), which approximates this projected reference view by searching over finite LLM proposed rewrites and verifying the realized image. The diffusion generator is never retrained or modified; intervention occurs only by replacing the conditioning prompt before sampling from $G^*$.

Given a prompt $c$, Stage 1 searches for a nearby safer rewrite $c^\star$ using a pretrained LLM. This is a finite candidate approximation to the ideal nearest safe projection in Eq.~\eqref{eq:prompt-projection-set}; it does not assume that the prompt only LLM score is an exact estimate of $u(G^*\mid c)$. Stage 2 then samples $\hat{x}\sim G^*(\cdot\mid c^\star)$ and verifies the realized image with a safeguard VLM. If $\widehat{u}_{\mathrm{VLM}}(\hat{x})\le\tau$, SPOT returns the verified pair. Otherwise, it repeats the full procedure up to $R$ attempts and returns the final attempted pair when the budget is exhausted. Budget exhausted cases are reported separately in Appendix~\ref{app:cor-VLM-LLM}. Algorithm~\ref{alg:spot} summarizes the pipeline.

\vspace{-2mm}
\subsection{Soft Projection by Local Search}
\label{subsec:method_projection}
\vspace{-2mm}

Exact projection onto $\mathcal{C}_{\mathrm{safe},\tau}$ is infeasible because membership depends on expected image level risk under $G^*$. Verifying every candidate by image generation and VLM scoring would also be costly. SPOT therefore replaces the hard safe set constraint with a soft penalty and optimizes over the finite candidate set proposed by the LLM.

Using the prompt distance $d$ in Eq.~\eqref{eq:angular-metric}, Stage 1 ranks each rewrite $c'$ by
\begin{align}
    J_\tau(c;c') 
    &:= d(c,c') + \alpha\big[\widehat{u}_{\mathrm{LLM}}(c')-\tau\big]_+,
    \label{eq:method_objective} \\
    [z]_+ 
    &:= \max\{z,0\},
    \label{eq:hinge}
\end{align}
where $\alpha>0$ controls the safety penalty. The distance term discourages unnecessary rewriting, while the hinge term penalizes candidates whose prompt only score exceeds the target tolerance. The LLM score is used only for search and early stopping; final image level verification is deferred to Stage 2.

Starting from $c_0=c$, local search samples $N$ rewrites $\{c_t^{(n)}\}_{n=1}^{N}$ conditioned on $c_t$, evaluates $J_\tau(c;\cdot)$ over $\{c_t\}\cup\{c_t^{(n)}\}_{n=1}^{N}$, and updates
\begin{equation}
c_{t+1}\in 
\argmin_{c'\in \{c_t\}\cup\{c_t^{(n)}\}_{n=1}^N}
J_\tau(c;c').
\label{eq:local_search_update}
\end{equation}
The search stops after at most $T$ steps, or earlier when $\widehat{u}_{\mathrm{LLM}}(c_{t+1})\le\tau$. The selected prompt $c^\star$ is therefore the best candidate found by finite search, not a certified metric projection. In experiments, we instantiate $\phi$ in Eq.~\eqref{eq:angular-metric} with MiniLM L6 sentence embeddings~\citep{wang2020minilm}. The LLM is instructed to preserve key entities, composition, background, and style, and to edit only unsafe cues; these instructions are practical guidance rather than formal invariants. The full prompt template is provided in Appendix~\ref{app:prompt-inst}.

\subsection{Operational Scoring and Retry}
\label{subsec:method_scoring}

Both stages use the same binary choice scoring protocol so that $\tau$ has a consistent operational meaning during search and verification. The model is asked to choose between A=\textsc{Safe} and B=\textsc{Unsafe}, and the next token probabilities of these labels are converted into an unsafety score in $[0,1]$. The prompt only LLM produces $\widehat{u}_{\mathrm{LLM}}(c')$ for candidate ranking, whereas the safeguard VLM produces $\widehat{u}_{\mathrm{VLM}}(\hat{x})$ for a realized image $\hat{x}\sim G^*(\cdot\mid c^\star)$. An attempt satisfies the target safety level only when
\begin{equation}
    \widehat{u}_{\mathrm{VLM}}(\hat{x}) \le \tau.
    \label{eq:vlm_acceptance}
\end{equation}

If Eq.~\eqref{eq:vlm_acceptance} fails, SPOT repeats Stage 1 and Stage 2 with a new stochastic candidate set. If no attempt passes within $R$ attempts, SPOT returns the final attempted image and its prompt $(c^\star,\hat{x})$. Thus, Stage 2 verifies accepted attempts, whereas budget exhausted outputs are final retry outputs rather than certified accepted samples. We defer the binary choice template, used LLM/VLM checkpoints, token sets, logit averaging rule, abstention handling, agreement diagnostics, threshold enrichment, and runtime analysis to Appendix~\ref{app:cor-VLM-LLM}, Appendix~\ref{app:model-checkpoints} and Appendix~\ref{subsec:app:eff}.

\section{Experiments}
\label{sec:experiments}
\subsection{Experimental Setup}
\label{subsec:exp-setup}

\begin{table*}[!t]
\caption{Main results. Lower IP/FID and higher CLIP are better. Best and second-best values are bolded and underlined, respectively.}
\label{tab:main-table}
\centering
\scriptsize

\begin{subtable}[t]{\textwidth}
\centering
\caption*{\textbf{SD1.5}}
\resizebox{\textwidth}{!}{%
\begin{tabular}{@{}cccccccccccccccc@{}}
\toprule
Data & Metric & \makecell{No\\alignment} & SLD & ESD-u & UCE & \makecell{Align\\Guard} & VALOR & POSI & PNO & \makecell{Prompt\\Guard} & SAFREE & \makecell{Latent\\Guard} & \makecell{Guard\\T2I} & IPR & SPOT \\
\midrule
CoProV2 & \multirow[c]{3}{*}{IP $\downarrow$}
& 0.51 & 0.27 & 0.22 & 0.33 & 0.07 & 0.16 & 0.16 & 0.13 & 0.07 & 0.09 & \underline{0.06} & 0.07 & 0.08 & \textbf{0.04} \\
\cmidrule(lr){1-1}\cmidrule(lr){3-16}
I2P
& & 0.36 & 0.19 & 0.25 & 0.30 & 0.11 & 0.18 & 0.15 & 0.19 & 0.12 & 0.09 & 0.16 & 0.09 & \underline{0.08} & \textbf{0.06} \\
\cmidrule(lr){1-1}\cmidrule(lr){3-16}
UD
& & 0.52 & 0.30 & 0.21 & 0.38 & 0.16 & 0.18 & 0.19 & 0.29 & 0.11 & 0.11 & 0.10 & 0.11 & \underline{0.09} & \textbf{0.04} \\
\cmidrule(lr){1-16}
\multirow[c]{2}{*}{COCO}
& FID $\downarrow$
& \textbf{32.34} & 36.29 & 34.41 & 46.12 & 38.82 & 32.61 & 34.26 & 48.89 & 46.39 & 37.87 & 33.59 & 38.71 & 37.78 & \underline{32.46} \\
\cmidrule(lr){2-16}
& CLIP $\uparrow$
& \textbf{33.42} & 32.21 & 27.21 & 31.84 & 32.27 & 33.28 & 28.49 & 30.36 & 27.70 & 31.99 & 33.26 & 33.20 & 30.10 & \underline{33.36} \\
\bottomrule
\end{tabular}%
}
\end{subtable}

\vspace{1mm}

\begin{subtable}[t]{0.49\textwidth}
\centering
\caption*{\textbf{SD2.1}}
\resizebox{\linewidth}{!}{%
\begin{tabular}{@{}cccccccc@{}}
\toprule
Data & Metric & \makecell{No\\alignment} & \makecell{Align\\Guard} & \makecell{Latent\\Guard} & \makecell{Guard\\T2I} & IPR & SPOT \\
\midrule
CoProV2 & \multirow[c]{3}{*}{IP $\downarrow$}
& 0.51 & 0.12 & \underline{0.06} & \underline{0.06} & 0.08 & \textbf{0.05} \\
\cmidrule(lr){1-1}\cmidrule(lr){3-8}
I2P
& & 0.35 & 0.12 & 0.15 & 0.09 & \underline{0.07} & \textbf{0.06} \\
\cmidrule(lr){1-1}\cmidrule(lr){3-8}
UD
& & 0.55 & 0.17 & \underline{0.09} & 0.10 & 0.10 & \textbf{0.04} \\
\cmidrule(lr){1-8}
\multirow[c]{2}{*}{COCO}
& FID $\downarrow$
& \textbf{32.78} & 37.54 & 34.66 & 39.79 & 38.99 & \underline{32.85} \\
\cmidrule(lr){2-8}
& CLIP $\uparrow$
& \textbf{34.95} & 34.41 & 34.73 & 34.81 & 31.81 & \underline{34.92} \\
\bottomrule
\end{tabular}%
}
\end{subtable}\hfill
\begin{subtable}[t]{0.49\textwidth}
\centering
\caption*{\textbf{SDXL}}
\resizebox{\linewidth}{!}{%
\begin{tabular}{@{}cccccccc@{}}
\toprule
Data & Metric & \makecell{No\\alignment} & \makecell{Align\\Guard} & \makecell{Latent\\Guard} & \makecell{Guard\\T2I} & IPR & SPOT \\
\midrule
CoProV2 & \multirow[c]{3}{*}{IP $\downarrow$}
& 0.49 & 0.09 & \underline{0.05} & 0.06 & 0.09 & \textbf{0.03} \\
\cmidrule(lr){1-1}\cmidrule(lr){3-8}
I2P
& & 0.31 & 0.08 & 0.12 & \underline{0.06} & 0.07 & \textbf{0.04} \\
\cmidrule(lr){1-1}\cmidrule(lr){3-8}
UD
& & 0.47 & 0.11 & 0.06 & \underline{0.05} & 0.06 & \textbf{0.02} \\
\cmidrule(lr){1-8}
\multirow[c]{2}{*}{COCO}
& FID $\downarrow$
& \textbf{32.36} & 39.98 & 33.28 & 38.20 & 34.50 & \underline{32.45} \\
\cmidrule(lr){2-8}
& CLIP $\uparrow$
& \textbf{36.05} & 35.12 & 35.73 & 35.83 & 33.87 & \underline{36.01} \\
\bottomrule
\end{tabular}%
}
\end{subtable}

\end{table*}

\paragraph{Metrics.}
For safety, we report inappropriate percentage (IP), following SLD~\citep{2023SLD}. Each generated image is scored by Q16~\citep{schramowski2022can} and NudeNet~\citep{Nudenet2022}; we use the maximum detector response to determine inappropriate content and report the empirical rate. For benign prompt behavior, we report FID~\citep{lucic2018gans} and CLIPScore~\citep{hessel2021clipscore} on COCO, measuring image quality and prompt image alignment.

\paragraph{Datasets.}
We follow AlignGuard~\citep{2025alignguard}. Unsafe prompt evaluation uses CoProV2, I2P~\citep{2023SLD}, and Unsafe Diffusion (UD)~\citep{qu2023unsafe}. CoProV2 contains 15,690/8,000 train/test prompt pairs across seven safety categories; I2P contains 4,703 prompts across seven categories; and UD contains 932 prompts across five categories. For benign prompt evaluation, we sample 3,000 image caption pairs from COCO~\citep{lin2014microsoft}. Details are in Appendix~\ref{sec:app:data}.

\subsection{Main Safety and Reference Preservation Results}
\label{subsec:main-results}

Table~\ref{tab:main-table} reports the safety and reference preservation across three diffusion backbones (SD1.5, SD2.1, SDXL)
and multiple evaluation sets. All model selection (including hyperparameter tuning) is performed \emph{only} on CoProV2.
We assess generalization via cross-dataset IP on I2P and UD, and measure utility on benign COCO captions using FID and
CLIPScore.
\vspace{-2.mm}

\paragraph{Safety generalization across datasets.}
SPOT achieves the lowest inappropriate percentage not only on the in-domain CoProV2 benchmark but also on both OOD
prompt sets. On SD1.5, IP drops to 0.04/0.06/0.04 on CoProV2/I2P/UD, outperforming strong baselines that are competitive
in-domain yet transfer less reliably across prompt distributions (e.g., LatentGuard: 0.06 on CoProV2 vs.\ 0.16 on I2P).
This trend is consistent on SD2.1 and SDXL, indicating that the gains are not driven by CoProV2-specific phrasing or
concept coverage.
\vspace{-2.mm}

\paragraph{Robustness across diffusion backbones.}
Safety gains persist under backbone changes without retraining the diffusion model. While unaligned models exhibit
substantial unsafe rates (e.g., SDXL: 0.49/0.31/0.47 on CoProV2/I2P/UD), SPOT reduces IP to 0.03/0.04/0.02 on SDXL
and similarly improves SD2.1 (0.05/0.06/0.04), suggesting the alignment transfers across SD variants rather than
overfitting to a specific checkpoint.
\vspace{-2.mm} 

\paragraph{Benign prompt preservation on COCO.}
On benign COCO captions, SPOT preserves image quality and text--image alignment close to the no alignment reference
across all backbones (e.g., SD1.5: 32.46/33.36 vs.\ 32.34/33.42 in FID/CLIP), while substantially reducing IP on unsafe
prompts. In contrast, methods based on aggressive suppression or embedding/noise manipulation can incur clear degradation in benign prompt behavior, with notably worse FID and/or lower CLIP (e.g., UCE and PNO on SD1.5), consistent with semantic drift even
under benign inputs.
\vspace{-2.mm}

\paragraph{Guard pipelines and selection effects.}
\label{para:guard-style-pipeline}
Guard style methods~\citep{2025alignguard,2024guardt2i} may achieve relatively strong CLIPScore on retained prompts, but prompt level overscreening of benign inputs can shift the generated image distribution away from the full benign reference set, leading to degraded FID. As a result, their reported FID can be affected by selection bias. This makes comparison less direct with methods that generate for all prompts.

Overall, SPOT achieves strong safety and reference preservation, lowering IP across datasets and diffusion backbones while keeping benign prompt behavior close to the unaligned reference. This suggests that SPOT does not improve safety by globally distorting reference behavior, but localizes deviation to prompts that need intervention. This is consistent with the SPAT view: safety gains require reference deviation, so the practical goal is to confine that deviation to unsafe prompts while preserving already safe behavior.

\begin{wrapfigure}{r}{0.38\linewidth}  

\vspace{-5mm}
  \centering
  \includegraphics[width=\linewidth]{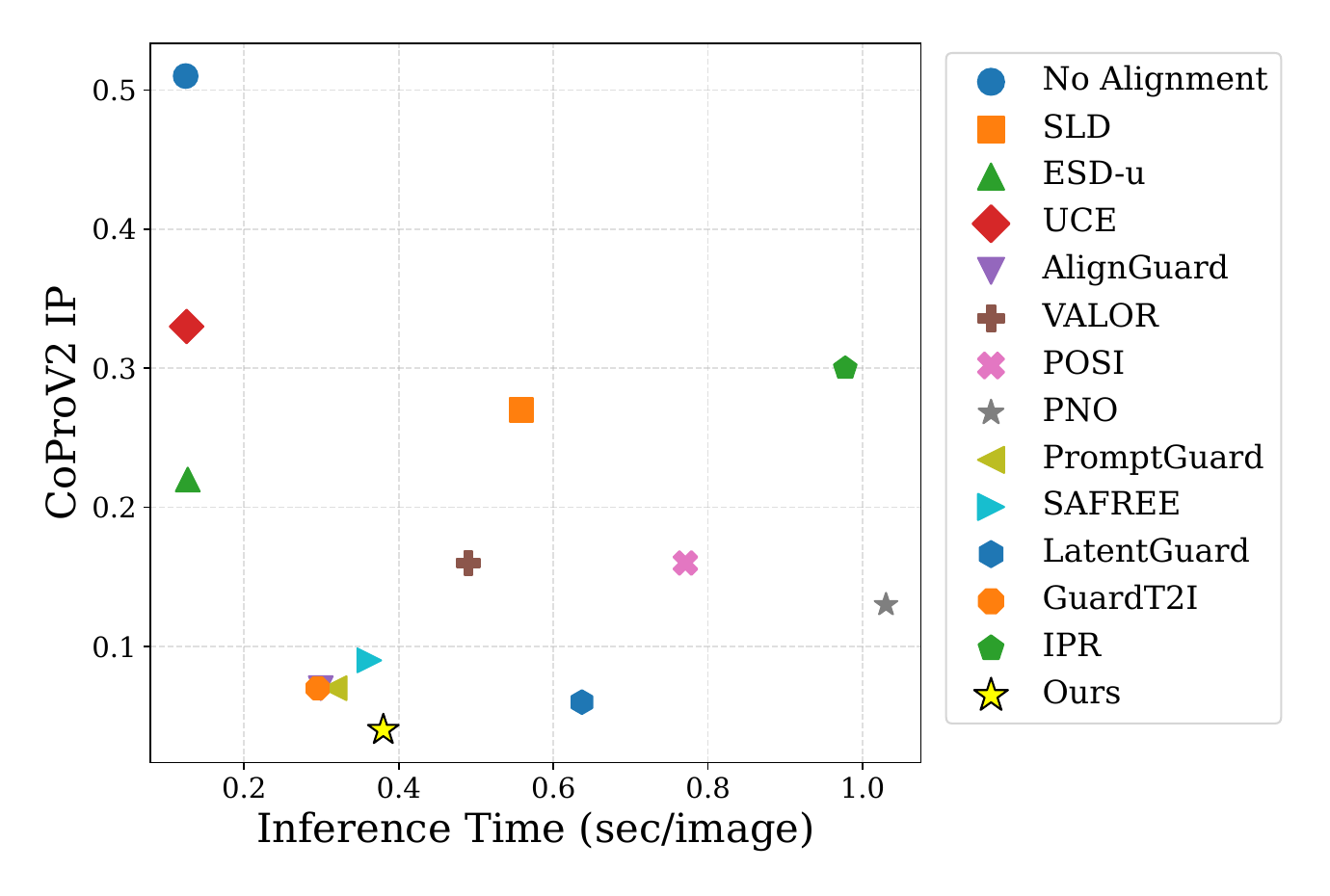}
  \caption{CoProV2 IP versus per-image inference time.}
  \label{fig:inference-efficiency}
\end{wrapfigure}

\subsection{Selective Intervention on Benign Prompts}
\label{subsec:safe-prompts}

A key practical requirement for prompt side safety intervention is selectivity: benign prompts should not be rewritten when no safety intervention is needed. We evaluate this directly on COCO by applying each method without COCO-specific tuning and measuring the normalized unchanged-prompt ratio over $n=3{,}000$ benign prompts. SPOT leaves 97.70\% of benign prompts unchanged, substantially higher than VALOR (92.4\%), IPR (54.9\%), and POSI (25.90\%). Thus, SPOT does not merely reduce unsafe generations; it preserves benign behavior in part by avoiding unnecessary edits to benign prompts. Full results are reported in Table~\ref{tab:coco_results} Appendix~\ref{app:fixed_rate_coco}.

\subsection{Efficiency and Iterative Burden}
\begin{wraptable}{r}{0.45\linewidth}
\centering
\vspace{-4mm}
\scriptsize
\setlength{\tabcolsep}{0pt}
\setlength{\fboxsep}{0pt}
\renewcommand{\arraystretch}{1.15}
\definecolor{rowgray}{gray}{0.92}
\newcommand{\pad}{1.6pt}
\newcommand{\cell}[1]{\hspace{\pad}#1\hspace{\pad}}
\newcommand{\gcell}[1]{\colorbox{rowgray}{\strut\cell{#1}}}
\newcommand{\thdr}[1]{\multicolumn{1}{c}{\scriptsize\bfseries\cell{#1}}}
\caption{Category-wise IP scores across four adversarial attack methods on the CoProV2 dataset.
\textbf{Abbrev.:} Shk=Shocking, S-H=Self-harm, Sex=Sexual, Ill=Illegal, Hate=Hate, Vio=Violence, Har=Harassment.}
\label{tab:attack_category_ratios}
\resizebox{\linewidth}{!}{%
\begin{tabular}{lcccccccc}
\toprule
 & \thdr{Shk} & \thdr{S-H} & \thdr{Sex} & \thdr{Ill} & \thdr{Hate} & \thdr{Vio} & \thdr{Har} & \thdr{Total} \\
\midrule
\gcell{\textbf{No attack}} &
\gcell{0.04} & \gcell{0.03} & \gcell{0.04} & \gcell{0.03} &
\gcell{0.05} & \gcell{0.04} & \gcell{0.04} & \gcell{0.04} \\
\midrule
\cell{\textbf{MMA}} &
\cell{0.07} & \cell{0.06} & \cell{0.06} & \cell{0.05} & \cell{0.06} & \cell{0.06} & \cell{0.05} & \cell{0.06} \\
\cell{\textbf{Ring-A-Bell}} &
\cell{0.04} & \cell{0.07} & \cell{0.05} & \cell{0.04} & \cell{0.05} & \cell{0.05} & \cell{0.05} & \cell{0.05} \\
\cell{\textbf{SneakyPrompt}} &
\cell{0.05} & \cell{0.06} & \cell{0.04} & \cell{0.05} & \cell{0.05} & \cell{0.04} & \cell{0.05} & \cell{0.05} \\
\cell{\textbf{P4D}} &
\cell{0.07} & \cell{0.06} & \cell{0.06} & \cell{0.05} & \cell{0.06} & \cell{0.06} & \cell{0.05} & \cell{0.06} \\
\bottomrule
\end{tabular}%
}
\vspace{-4mm}
\end{wraptable}

A practical concern is whether prompt projection requires excessive search. SPOT addresses this through a two stage cascade: Stage 1 uses prompt only scoring to search over candidate rewrites, while Stage 2 reserves image generation and VLM scoring for final acceptance. The design uses the cheap signal for routing and the image level verifier for the authoritative decision.

Under the full prompt projection and verification pipeline, mean time per prompt decreases from 1.93/1.82/2.51 seconds with VLM only search to 0.38/0.40/0.58 seconds with SPOT on SD1.5/SD2.1/SDXL. This gives a $4.3\times$--$5.1\times$ speedup, where an output is accepted only when $\widehat{u}_{\mathrm{VLM}}(\hat{x})\le\tau$ at the Stage 2.

Stage 1 also concentrates verification on promising candidates. At $\tau=0.05$, a diagnostic on $1000$ CoProV2 prompts shows that only 4.4\% of generated images pass the Stage 2 condition without routing. After sorting candidates by the Stage 1 prompt only score, the lowest scored 1\% pass Stage 2 at a rate of 90.0\%, a $20.45\times$ enrichment. Thus, Stage 1 does not replace the VLM; it routes search toward candidates that are more likely to pass final image level verification.

The retry burden is modest, with SPOT requiring 1.78 attempts on average, accepting 96.4\% of prompts within the retry budget $R$, and exhausting the budget for only 3.6\%. As shown in Fig.~\ref{fig:inference-efficiency}, SPOT is 22.4\% faster than VALOR while reducing IP from 0.16 to 0.04, and 61.1\% faster than IPR while reducing IP from 0.08 to 0.04. The cascade is therefore a core part of selective prompt projection, not merely an implementation shortcut.
\vspace{-1mm}
\begin{figure*}[t]
  \centering

  
  \begin{minipage}[t]{0.50\textwidth}
    \vspace{0pt}
    \centering
    
    \begin{subfigure}[t]{0.18\linewidth}
      \centering
      \includegraphics[width=\linewidth]{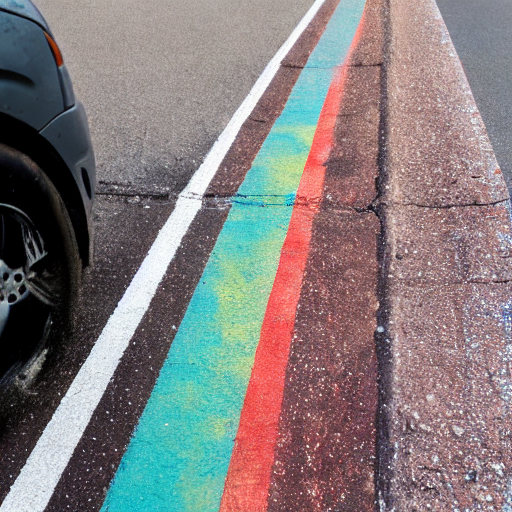}\par
      \includegraphics[width=\linewidth]{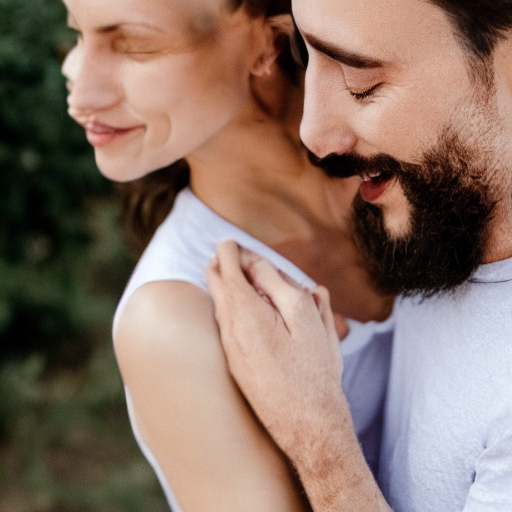}
      \caption{}
    \end{subfigure}\hfill
    \begin{subfigure}[t]{0.18\linewidth}
      \centering
      \includegraphics[width=\linewidth]{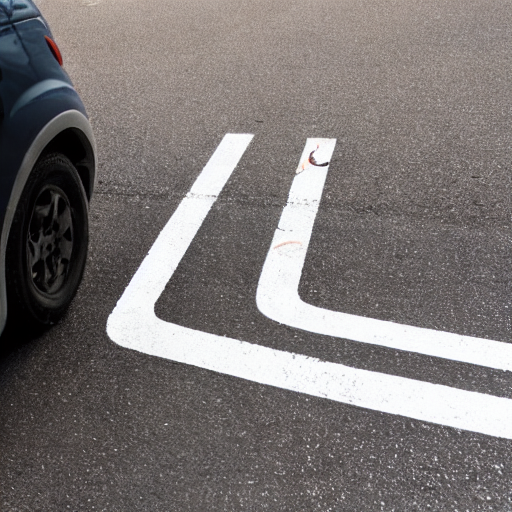}\par
      \includegraphics[width=\linewidth]{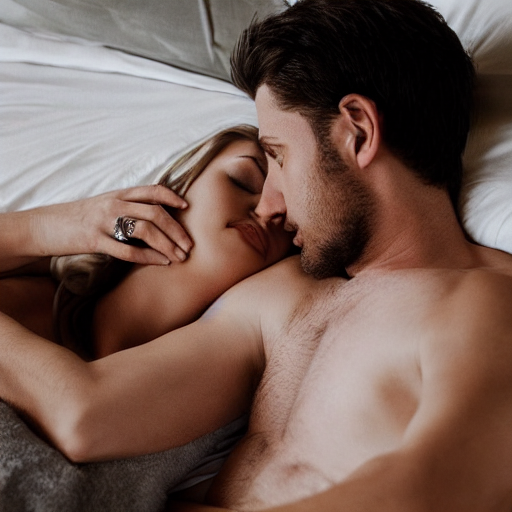}
      \caption{}
    \end{subfigure}\hfill
    \begin{subfigure}[t]{0.18\linewidth}
      \centering
      \includegraphics[width=\linewidth]{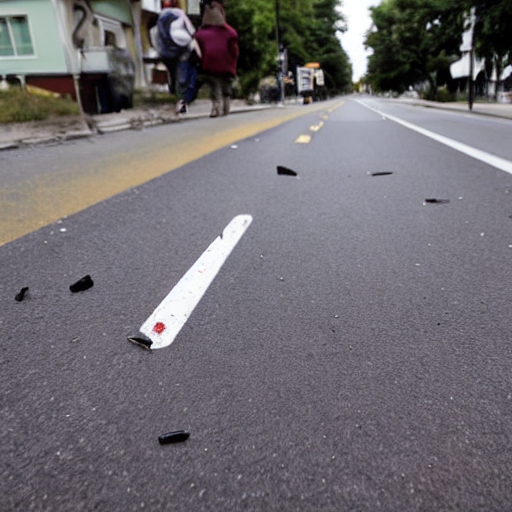}\par
      \includegraphics[width=\linewidth]{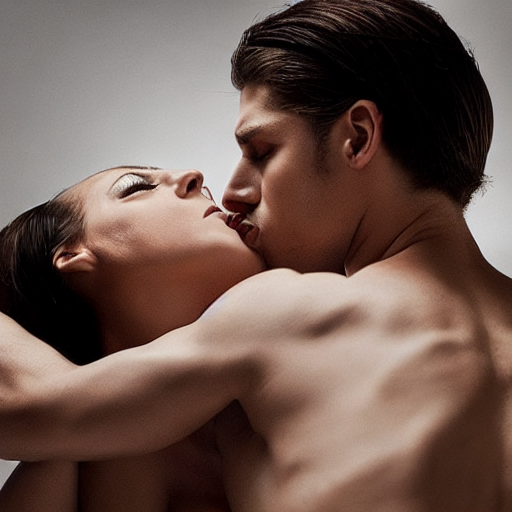}
      \caption{}
    \end{subfigure}\hfill
    \begin{subfigure}[t]{0.18\linewidth}
      \centering
      \includegraphics[width=\linewidth]{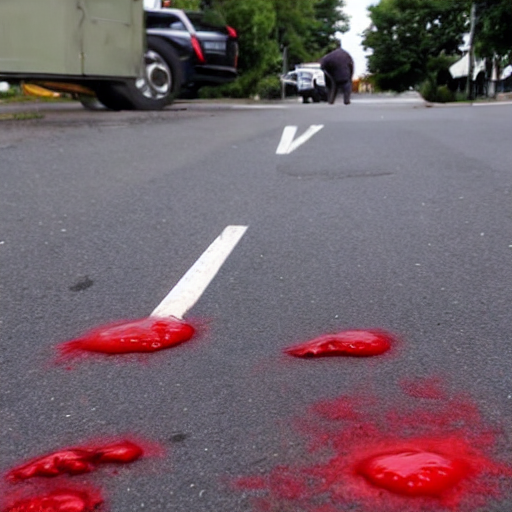}\par
      \includegraphics[width=\linewidth]{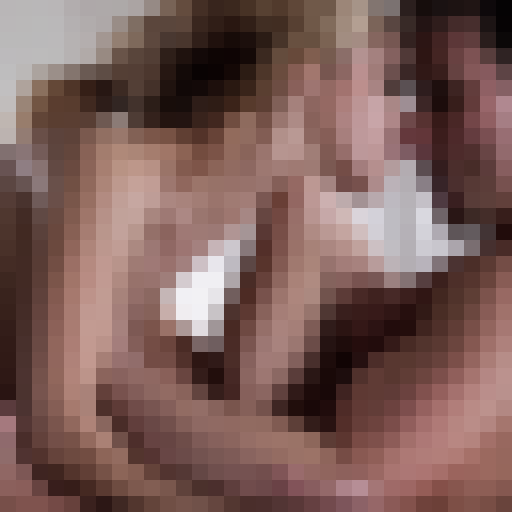}
      \caption{}
    \end{subfigure}\hfill
    \begin{subfigure}[t]{0.18\linewidth}
      \centering
      \includegraphics[width=\linewidth]{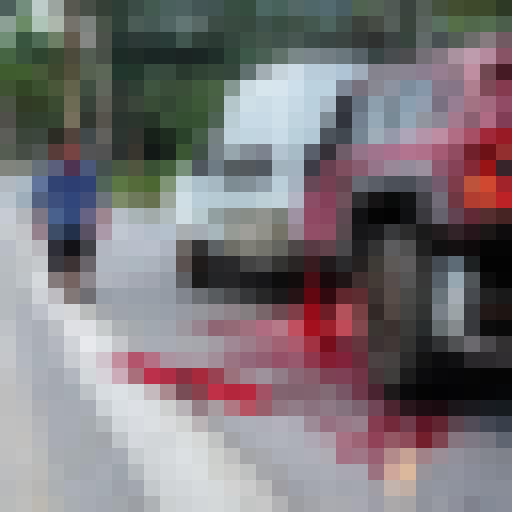}\par
      \includegraphics[width=\linewidth]{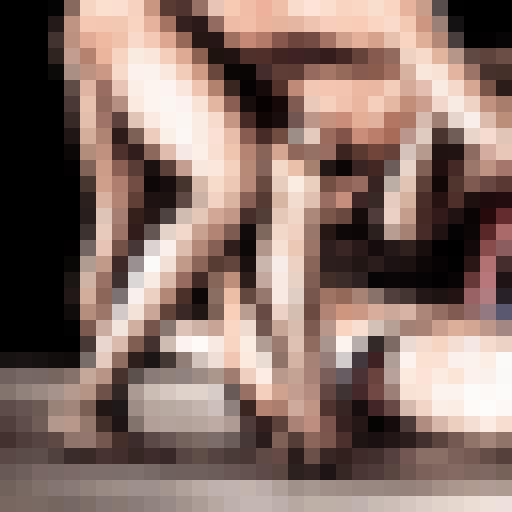}
      \caption{}
    \end{subfigure}

    \caption{(a) $\tau=0.1$, (b) $\tau=0.3$, (c) $\tau=0.5$, (d) $\tau=0.7$, and (e) $\tau=0.9$.}
    \label{fig:two_by_five_columns_labeled}
  \end{minipage}
  \hfill
  \begin{minipage}[t]{0.45\textwidth}
  \vspace{0pt}
  \centering

  \setcounter{subfigure}{0}
  \label{fig:prompt_analysis}

  \begin{subfigure}[t]{0.48\linewidth}
    \centering
    \includegraphics[width=\linewidth]{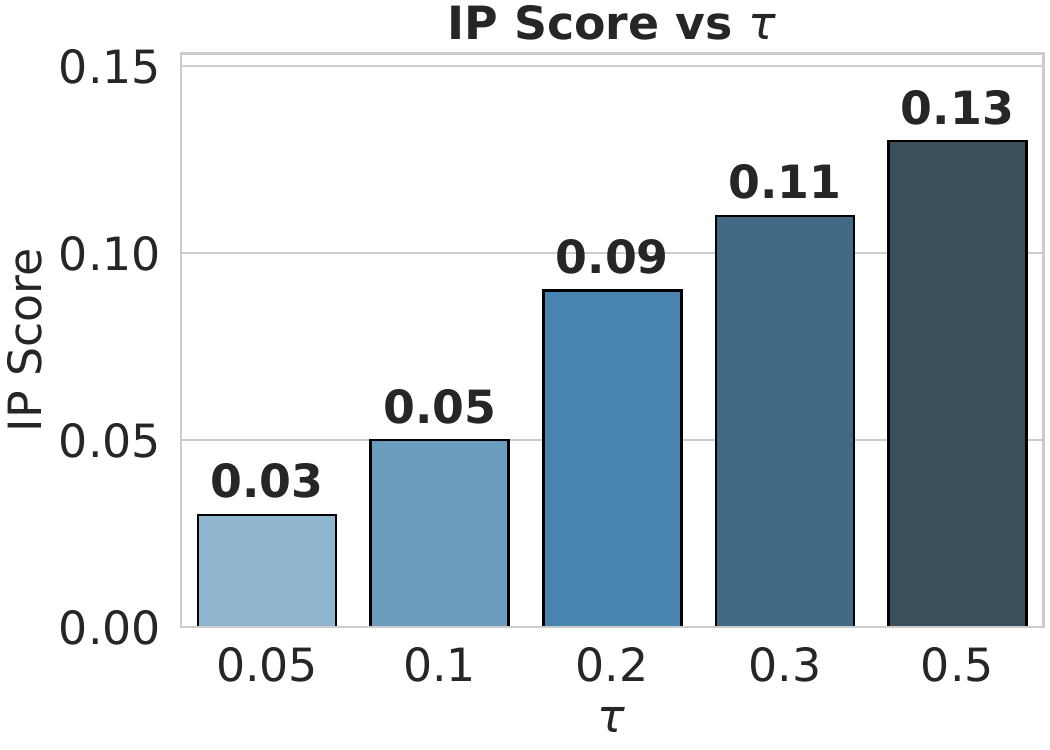}
    \caption{}
    \label{fig:ipscore}
  \end{subfigure}
  \hfill
  \begin{subfigure}[t]{0.48\linewidth}
    \centering
    \includegraphics[width=\linewidth]{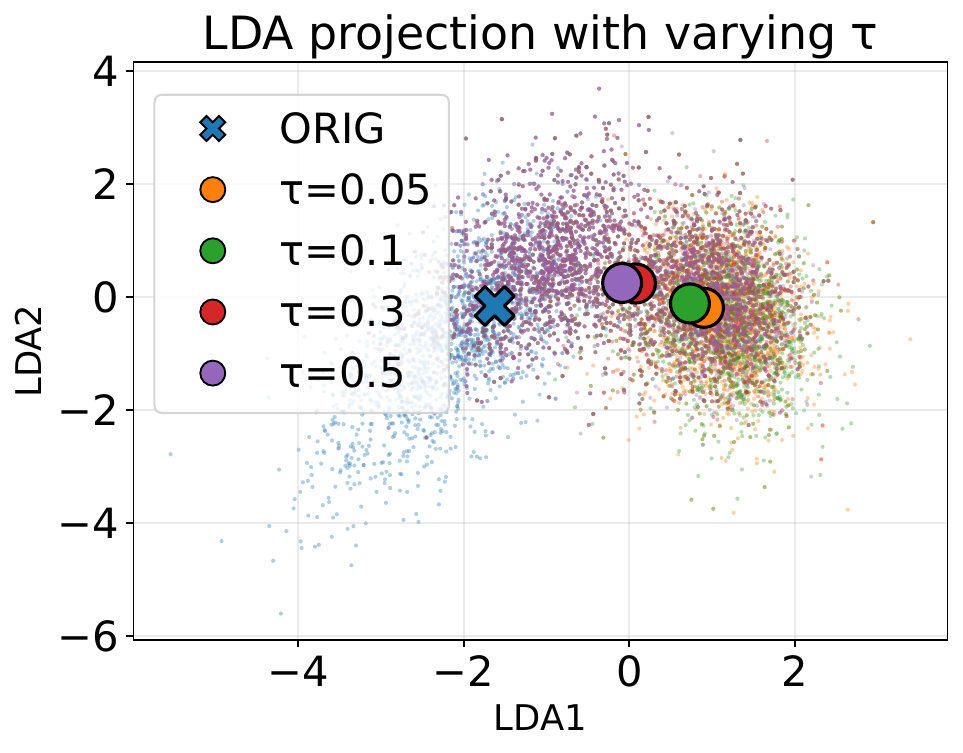}
    \caption{}
    \label{fig:lda}
  \end{subfigure}

  \caption{(a) IP score vs.\ $\tau$. (b) LDA projection of original/projected prompt embeddings with per-$\tau$ centroids.}
\end{minipage}

\end{figure*}

\subsection{Additional Analyses}
\label{subsec:tau-control}
 
\paragraph{Adjustable safety target through $\tau$.}
SPOT exposes an adjustable safety target through the shared threshold $\tau$ used by Stage 1 routing and Stage 2 verification. This operationalizes the $\tau$ controlled target in Sec.~\ref{subsec:tau-proj}: smaller $\tau$ enforces stricter safety, whereas larger $\tau$ permits increasingly borderline content. Qualitatively, when the input prompt, sampler seed, and all other generation settings are fixed, increasing $\tau$ yields images with progressively stronger sexual and violent cues (Fig.~\ref{fig:two_by_five_columns_labeled}). Quantitatively, on CoProV2, IP increases monotonically over $\tau\in\{0.05,0.1,0.3,0.5\}$ (Fig.~\ref{fig:ipscore}), indicating predictable control of the operating safety level. LDA (Linear Discriminant Analysis) projections further show larger prompt displacement under smaller $\tau$ (Fig.~\ref{fig:lda}), consistent with stronger intervention under stricter safety targets.

\paragraph{Adversarial attack robustness.}
We evaluate our framework against four text-based adversarial prompt attacks (MMA, Ring-A-Bell, SneakyPrompt, P4D-K)~\cite{2024mma,2023ringabell,2024sneakyprompt,2023p4d} on the CoProV2 test set. As shown in Table~\ref{tab:attack_category_ratios}, while adversarial perturbations slightly increase IP scores, the total IP rises only from 0.04 to at most 0.06, with no category showing an isolated spike. This robustness stems from SPOT's two-stage design, which projects adversarially obfuscated prompts toward the $\tau$ safe set while requiring final image level verification before acceptance.
Thus, adversarial text obfuscation alone is insufficient to bypass the safety gate, while the projection step removes unsafe semantic content with limited drift from the original prompt. Qualitative examples are provided in
Appendix~\ref{Adversarial_attack_example}.
\vspace{-1mm}

\subsection{Empirical Support for SPAT and Approximate Projection}
\label{subsec:empirical-spat-proj}

We next ask whether the observed behavior is consistent with the theory. In Fig.~\ref{fig:main-result}, we plot IP against a COCO based FID-to-reference proxy across datasets and backbones, where the no alignment generator serves as the reference. Although this proxy does not measure prompt conditioned TV directly, the per-dataset linear regressions are consistently negative, where methods that achieve larger IP reductions also tend to move farther from the reference on benign prompts. In other words, nontrivial safety gains are accompanied by measurable distributional shift, consistent with SPAT.

The operational procedure also behaves like an approximate projection rather than an unconstrained rewrite heuristic. On benign prompts, it is near identity, as shown in Sec.~\ref{subsec:safe-prompts}. On unsafe prompts, SPOT quickly reaches a stable rewrite, with the fixed prompt rate rising from 7.72\% for the initial projection to 77.85\% after one reapplication and 90.56\% after two, indicating that reapplying SPOT to its own output usually leaves the prompt unchanged. (Fig.~\ref{fig:fixedpoint-seq} in Appendix.~\ref{app:fixed_rate_coco}). Taken together, these results support the view that the proposed inference time method is both selective and stabilizing, consistent with an approximate projection toward a $\tau$ controlled safe set rather than broad prompt rewriting.
\vspace{-1mm}

\begin{figure*}
  \centering
  \includegraphics[width=\linewidth]{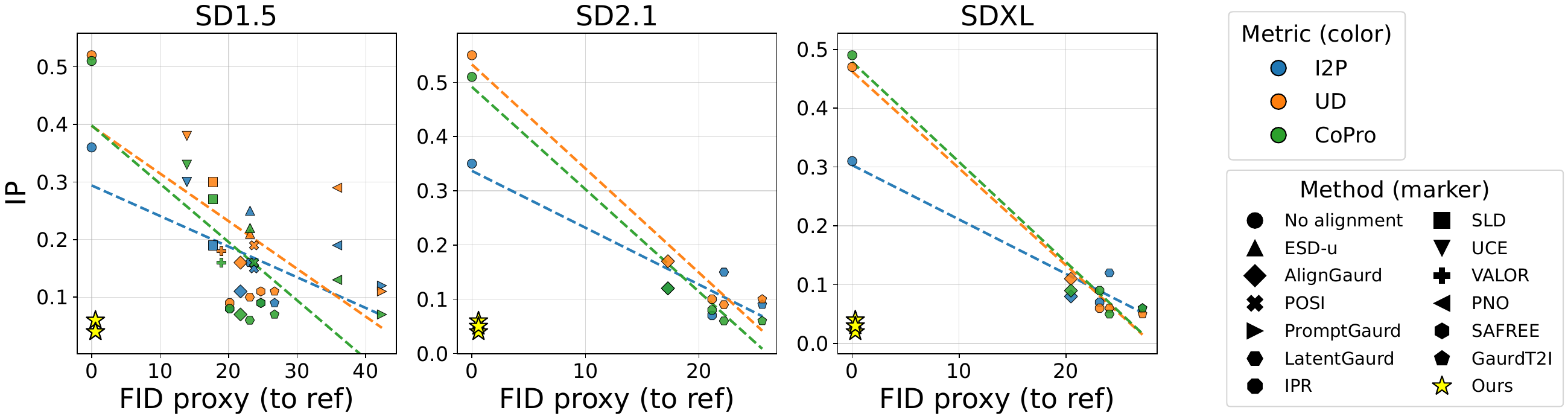}
  \caption{\textbf{SPAT diagnostic.}
  IP Scores versus FID-to-reference(ref) proxy on COCO safe prompts for SD1.5/SD2.1/SDXL.
  Colors denote datasets (UD/I2P/CoProV2), markers denote methods, and dashed lines are per-dataset linear fits, showing a consistent negative trend.}
  \label{fig:main-result}
\end{figure*}

\section{Limitations and Discussion}
\label{sec:limitations}
SPOT approximates the projected reference view by finite search and makes no global optimality claim for projection onto $\mathcal{C}_{\mathrm{safe},\tau}$. Stage 2 gives sample level acceptance for $\hat{x}\sim G^*(\cdot\mid c^\star)$, not a distributional guarantee, so behavior depends on the safety taxonomy, the safeguard VLM, and the retry budget $R$. The prompt metric and edit preserving guidance help limit semantic drift but are not formal invariants; when unsafe content is inseparable from the core intent, meaning preservation may conflict with meeting $\tau$. Residual risks remain from evaluator and rewriting errors, unequal error rates across groups, and misconfigured $\tau$. We therefore view SPOT as one layer in a broader safety stack, requiring calibrated thresholds, monitoring, bias audits, transparent failure reporting, and human review for high stakes use.
\vspace{-1mm}

\section{Conclusion}
We formalize text to image safety alignment through a TV lens, showing that reducing expected unsafety relative to a fixed reference requires distributional deviation. This motivates selective intervention rather than global model modification. We propose \textbf{SPOT}, an inference time prompt projection framework that keeps the generator frozen and combines prompt only routing with image level verification. Across four datasets and three diffusion backbones, SPOT reduces IP while keeping benign prompt behavior close to the unaligned reference on COCO, supporting principled and practical selective intervention.

\bibliographystyle{plainnat}
\bibliography{example_paper}
\newpage
\appendix

\section{Theory proofs and measurability details}
\label{app:theoretic-details}

This appendix supplies the technical details omitted from Section~\ref{sec:theory}.
Our goal is twofold.
First, we justify the fixed-reference Safety--Prompt Alignment Trade-off (SPAT) by showing that total variation controls deviations of bounded unsafety functionals.
Second, we formalize the projected-reference view underlying selective prompt intervention:
we state one sufficient regularity condition under which nearest $\tau$-safe prompts exist and measurable set-valued projection is available.
We then give a convenient fiber-supported construction of a projection kernel; this construction additionally uses the positive-mass condition $\mu(\mathcal{C}_{\mathrm{safe},\tau})>0$ in Section~\ref{app:idempotent-kernel}.
Finally, we prove the kernelized SPAT bound and the resulting $\tau$-controlled floor.

\subsection{Basic objects: spaces, probability measures, and kernels}
\label{app:basic-objects}

Let $(\mathcal{X},\mathcal{B}(\mathcal{X}))$ be a measurable space.
We write $\mathcal{P}(\mathcal{X})$ for the set of probability measures on $(\mathcal{X},\mathcal{B}(\mathcal{X}))$.

\paragraph{Markov kernels.}
A Markov kernel from $(\mathcal{C},\mathcal{B}(\mathcal{C}))$ to $(\mathcal{Y},\mathcal{B}(\mathcal{Y}))$
is a map
\[
K:\mathcal{C}\times\mathcal{B}(\mathcal{Y})\to[0,1]
\]
such that:
(i) for each $c\in\mathcal{C}$, $A\mapsto K(c,A)$ is a probability measure on $\mathcal{Y}$, and
(ii) for each $A\in\mathcal{B}(\mathcal{Y})$, $c\mapsto K(c,A)$ is $\mathcal{B}(\mathcal{C})$-measurable.
We often write $K(\cdot\mid c)$ for the measure $A\mapsto K(c,A)$.

\paragraph{Composition of kernels.}
If $K$ is a kernel from $\mathcal{C}$ to $\mathcal{Y}$ and $L$ is a kernel from $\mathcal{Y}$ to $\mathcal{Z}$,
their composition $L\circ K$ is the kernel from $\mathcal{C}$ to $\mathcal{Z}$ defined by
\begin{equation}
(L\circ K)(c,A):=\int_{\mathcal{Y}} L(y,A)\,K(c,dy),
\qquad A\in\mathcal{B}(\mathcal{Z}).
\label{eq:kernel-composition}
\end{equation}
Standard measurability of \eqref{eq:kernel-composition} follows from kernel calculus \cite{kallenberg2021foundations}.

\paragraph{Dirac (deterministic) kernels.}
For a measurable map $T:\mathcal{C}\to\mathcal{Y}$, the associated deterministic kernel is
\[
K_T(c,\cdot):=\delta_{T(c)}(\cdot).
\]
It is a Markov kernel whenever $T$ is measurable.

\subsection{Total variation: definitions and inequalities}
\label{app:tv}

For probability measures $\nu,\nu'$ on a measurable space $(\Omega,\mathcal{F})$, define
\begin{equation}
\mathrm{TV}(\nu,\nu')
:=
\sup_{A\in\mathcal{F}} |\nu(A)-\nu'(A)|.
\label{eq:tv-def-app}
\end{equation}

The following elementary lemma is the only analytic input needed for SPAT.
It says that total variation controls expectation gaps for \emph{any} bounded measurable score in $[0,1]$;
the unsafety score $U$ in the main text is one such example.

\begin{lemma}[TV controls expectation gaps for bounded functionals]
\label{lem:tv-bounded}
Let $f:\Omega\to[0,1]$ be measurable. Then
\begin{equation}
\big|\mathbb{E}_{\nu}[f]-\mathbb{E}_{\nu'}[f]\big|
\le
\mathrm{TV}(\nu,\nu').
\label{eq:tv-bounded-app}
\end{equation}
\end{lemma}

\begin{proof}
For $t\in[0,1]$, define the superlevel set
\[
A_t:=\{x\in\Omega:\ f(x)\ge t\}\in\mathcal{F}.
\]
By the layer-cake / distribution-function identity for bounded nonnegative measurable functions \cite{folland1999real},
\[
\mathbb{E}_{\nu}[f]=\int_0^1 \nu(A_t)\,dt,
\qquad
\mathbb{E}_{\nu'}[f]=\int_0^1 \nu'(A_t)\,dt.
\]
Therefore,
\begin{align}
\big|\mathbb{E}_{\nu}[f]-\mathbb{E}_{\nu'}[f]\big|
&=
\left|
\int_0^1 \bigl(\nu(A_t)-\nu'(A_t)\bigr)\,dt
\right| \\
&\le
\int_0^1 \big|\nu(A_t)-\nu'(A_t)\big|\,dt \\
&\le
\sup_{A\in\mathcal{F}} |\nu(A)-\nu'(A)| \\
&=
\mathrm{TV}(\nu,\nu').
\end{align}
This proves \eqref{eq:tv-bounded-app}.
\end{proof}

\subsection{SPAT under a fixed reference}
\label{app:spat}

We restate the main fixed-reference objects:
\[
u(G\mid c):=\mathbb{E}_{x\sim G(\cdot\mid c)}[U(x)],
\qquad
\mathcal{U}(G):=\mathbb{E}_{c\sim\mu}[u(G\mid c)],
\]
and
\[
\mathcal{A}_{\mathrm{TV}}(G)
:=
\mathbb{E}_{c\sim\mu}
\!\left[
\mathrm{TV}\bigl(G(\cdot\mid c),G^*(\cdot\mid c)\bigr)
\right].
\]

\begin{proof}[Proof of Theorem~\ref{thm:spat-maintext}]
Fix $c\in\mathcal{C}$.
Apply Lemma~\ref{lem:tv-bounded} on $(\Omega,\mathcal{F})=(\mathcal{X},\mathcal{B}(\mathcal{X}))$ with
\[
\nu=G(\cdot\mid c),\qquad
\nu'=G^*(\cdot\mid c),\qquad
f=U.
\]
Since $U:\mathcal{X}\to[0,1]$, we obtain
\[
\big|u(G\mid c)-u(G^*\mid c)\big|
\le
\mathrm{TV}\!\bigl(G(\cdot\mid c),G^*(\cdot\mid c)\bigr).
\]
Discarding the absolute value on one side gives
\[
u(G\mid c)
\ge
u(G^*\mid c)-\mathrm{TV}\!\bigl(G(\cdot\mid c),G^*(\cdot\mid c)\bigr),
\]
which is exactly \eqref{eq:spat-pointwise}.
Averaging over $c\sim\mu$ yields
\[
\mathcal{U}(G)+\mathcal{A}_{\mathrm{TV}}(G)\ge \mathcal{U}(G^*),
\]
that is, \eqref{eq:spat-bound}.
\end{proof}

The proof is short, but conceptually important:
SPAT is not tied to any specific classifier, unsafe category, or detector architecture.
It follows from a purely distributional fact:
\emph{if unsafety is evaluated by a bounded score under the output distribution, then reducing that score necessarily requires moving the distribution in TV.}

\subsection{$\tau$-safe projection: formal assumptions and consequences}
\label{app:projection}

We now formalize the projected-reference construction.
Fix $\tau\in[0,1]$ and define
\[
\mathcal{C}_{\mathrm{safe},\tau}
:=
\{c\in\mathcal{C}:u(G^*\mid c)\le\tau\}.
\]
Let $\phi:\mathcal{C}\to\mathbb{R}^d$ be a prompt embedding with $\phi(c)\neq 0$ for all $c$, and define
\begin{equation}
\bar\phi(c):=\frac{\phi(c)}{\|\phi(c)\|_2}\in\mathbb{S}^{d-1}.
\label{eq:phi-bar}
\end{equation}
The angular distance is
\begin{equation}
d(c,c')
:=
\arccos\!\bigl(\langle \bar\phi(c),\bar\phi(c')\rangle\bigr)\in[0,\pi],
\label{eq:angular-metric-app}
\end{equation}
and the nearest-safe correspondence is
\begin{equation}
\mathcal{P}_\tau(c)
:=
\argmin_{c'\in\mathcal{C}_{\mathrm{safe},\tau}} d(c,c').
\label{eq:prompt-projection-set-app}
\end{equation}

\paragraph{Remark (metric vs.\ pseudometric).}
If $\bar\phi$ is not injective, then $d$ is only a pseudometric on $\mathcal{C}$:
it may happen that $d(c,c')=0$ for distinct prompts $c\neq c'$.
This is not a problem for our theory.
All results use $d$ only as a measurable bounded cost.
Non-uniqueness is handled explicitly by allowing projection to be set-valued and, later, kernel-valued.

\begin{assumption}[Regularity for $\tau$-safe projection]
\label{assump:proj-regularity}
Fix $\tau\in[0,1]$. Assume:
\begin{enumerate}[label=(\roman*), leftmargin=*, nosep]
\item $(\mathcal{C},\mathcal{B}(\mathcal{C}))$ is standard Borel, and $\bar\phi:\mathcal{C}\to\mathbb{S}^{d-1}$ is Borel measurable.
\item $(\mathcal{X},\mathcal{B}(\mathcal{X}))$ is Polish, $U:\mathcal{X}\to[0,1]$ is bounded and continuous, and the map
\[
\Gamma(c):=G^*(\cdot\mid c)
\]
is weakly continuous with respect to $d$.
\item $\bar\phi(\mathcal{C}_{\mathrm{safe},\tau})$ is nonempty and closed in $\mathbb{S}^{d-1}$.
\end{enumerate}
\end{assumption}

Assumption~\ref{assump:proj-regularity} is one convenient sufficient condition.
It says, informally, that small prompt edits do not produce discontinuous jumps in the reference conditional law, and that the safe set is topologically well-behaved after passing to normalized embeddings.

\begin{lemma}[Continuity of the reference risk map]
\label{lem:u-cont}
Under Assumption~\ref{assump:proj-regularity}, the map
\[
c\mapsto u(G^*\mid c)=\int U(x)\,G^*(dx\mid c)
\]
is continuous with respect to $d$.
\end{lemma}

\begin{proof}
Let $d(c_n,c)\to 0$.
By Assumption~\ref{assump:proj-regularity}(ii),
\[
G^*(\cdot\mid c_n)\Rightarrow G^*(\cdot\mid c)
\]
weakly.
Since $U$ is bounded and continuous, the portmanteau theorem gives
\[
\int U(x)\,G^*(dx\mid c_n)\to \int U(x)\,G^*(dx\mid c).
\]
\end{proof}

\begin{lemma}[Closedness of the $\tau$-safe set]
\label{lem:safe-closed}
Under Assumption~\ref{assump:proj-regularity}, $\mathcal{C}_{\mathrm{safe},\tau}$ is $d$-closed.
\end{lemma}

\begin{proof}
By Lemma~\ref{lem:u-cont}, the map $c\mapsto u(G^*\mid c)$ is continuous.
Hence $\mathcal{C}_{\mathrm{safe},\tau}$ is the sublevel set of a continuous function and is therefore closed.
\end{proof}

\begin{lemma}[Compactness induced by normalized embeddings]
\label{lem:compactness-C}
Under Assumption~\ref{assump:proj-regularity}, $(\mathcal{C}_{\mathrm{safe},\tau},d)$ is compact in the pseudometric sense.
\end{lemma}

\begin{proof}
Let $F:=\mathcal{C}_{\mathrm{safe},\tau}$ and let $(c_n)\subset F$.
Then $(\bar\phi(c_n))$ lies in the set $\bar\phi(F)$, which is closed in the compact sphere $\mathbb{S}^{d-1}$ by Assumption~\ref{assump:proj-regularity}(iii), hence compact.
So there exists a subsequence $\bar\phi(c_{n_k})\to y\in \bar\phi(F)$.
Pick any $c^\star\in F$ such that $\bar\phi(c^\star)=y$.
Then
\[
d(c_{n_k},c^\star)
=
\arccos\!\bigl(\langle \bar\phi(c_{n_k}),\bar\phi(c^\star)\rangle\bigr)
\to 0.
\]
Thus every sequence in $F$ has a $d$-convergent subsequence in $F$.
\end{proof}

\begin{lemma}[Nearest-point attainment]
\label{lem:attainment}
Under Assumption~\ref{assump:proj-regularity}, $\mathcal{P}_\tau(c)\neq\emptyset$ for every $c\in\mathcal{C}$.
\end{lemma}

\begin{proof}
Fix $c\in\mathcal{C}$ and let $F:=\mathcal{C}_{\mathrm{safe},\tau}$.
By Lemma~\ref{lem:compactness-C}, $F$ is compact under $d$.
The map $c'\mapsto d(c,c')$ is continuous on $F$, hence attains its minimum.
\end{proof}

The previous lemmas guarantee geometric existence.
We now explain why projection can also be made measurable.

\begin{lemma}[Measurable set-valued nearest-safe projection]
\label{lem:measurable-graph-sufficient}
Under Assumption~\ref{assump:proj-regularity}, there exists a measurable map
\[
y_\tau:\mathcal{C}\to \bar\phi(\mathcal{C}_{\mathrm{safe},\tau})
\]
and a measurable correspondence
\[
\pi_\tau:\mathcal{C}\rightrightarrows \mathcal{C}_{\mathrm{safe},\tau}
\]
with nonempty values such that, for every $c\in\mathcal{C}$,
\begin{equation}
\emptyset\neq \pi_\tau(c)\subseteq \mathcal{P}_\tau(c),
\qquad
d(c_1,c_2)=0\ \ \forall\,c_1,c_2\in\pi_\tau(c).
\label{eq:set-valued-proj-properties}
\end{equation}
\end{lemma}

\begin{proof}
Let
\[
\mathcal{Y}_\tau:=\bar\phi(\mathcal{C}_{\mathrm{safe},\tau})\subseteq\mathbb{S}^{d-1}.
\]
By Assumption~\ref{assump:proj-regularity}(iii), $\mathcal{Y}_\tau$ is compact.
Define the sphere-level nearest-point correspondence
\[
\bar{\mathcal{P}}_\tau(y)
:=
\argmin_{y'\in\mathcal{Y}_\tau}
\arccos(\langle y,y'\rangle).
\]
Its graph is Borel because the angular metric is continuous and $\mathcal{Y}_\tau$ is compact Borel.
Hence, by measurable selection on standard Borel spaces \citep[Prop.~7.49]{bertsekas1996stochastic},
there exists a measurable selector
\[
\tilde y_\tau(y)\in \bar{\mathcal{P}}_\tau(y).
\]
Composing with $\bar\phi$ gives a measurable nearest-safe embedding representative
\[
y_\tau(c):=\tilde y_\tau(\bar\phi(c)).
\]

Now define the fiber correspondence
\[
\pi_\tau(c)
:=
\{c'\in\mathcal{C}_{\mathrm{safe},\tau}:\bar\phi(c')=y_\tau(c)\}.
\]
It is nonempty because $y_\tau(c)\in \bar\phi(\mathcal{C}_{\mathrm{safe},\tau})$.
Moreover, any two points in the same fiber have zero angular distance, so
$d(c_1,c_2)=0$ for all $c_1,c_2\in \pi_\tau(c)$.
Finally, by construction of $y_\tau(c)$ as a nearest-safe embedding representative,
every element of $\pi_\tau(c)$ belongs to $\mathcal{P}_\tau(c)$.

Moreover, the graph of the correspondence is
\[
\mathrm{Gr}(\pi_\tau)
=
\{(c,c')\in \mathcal{C}\times \mathcal{C}_{\mathrm{safe},\tau}:\bar\phi(c')=y_\tau(c)\},
\]
which is Borel because $\bar\phi$ and $y_\tau$ are measurable.
Hence $\pi_\tau$ is a measurable correspondence.
\end{proof}

This is the technical reason the main text adopts a \emph{kernel view} instead of insisting on a unique projection map:
in general there may be many equally close safe prompts, and measurable tie-breaking is not canonical.
A kernel naturally spreads mass across that set.

\subsection{$\tau$-nearest projection kernels and idempotence}
\label{app:idempotent-kernel}

We now give a convenient Markov-kernel construction used in the main text.
Because this construction normalizes the restriction of $\mu$ to $\mathcal{C}_{\mathrm{safe},\tau}$, it additionally assumes $\mu(\mathcal{C}_{\mathrm{safe},\tau})>0$.

\begin{definition}[$\tau$-nearest projection kernel]
\label{def:tau-nearest-kernel}
A Markov kernel
\[
\Pi_\tau:\mathcal{C}\times\mathcal{B}(\mathcal{C}_{\mathrm{safe},\tau})\to[0,1]
\]
is called a \emph{$\tau$-nearest projection kernel} if
\begin{align}
\Pi_\tau(c,\cdot)&=\delta_c(\cdot), \qquad &&c\in\mathcal{C}_{\mathrm{safe},\tau},
\label{eq:nearest-kernel-identity}\\
\Pi_\tau(c,\mathcal{P}_\tau(c))&=1, \qquad &&\mu\text{-a.e.\ }c\in\mathcal{C}.
\label{eq:nearest-kernel-support}
\end{align}
\end{definition}

The first property is identity on already-safe prompts.
The second says all projection mass is concentrated on nearest-safe prompts.

\begin{lemma}[A convenient fiber-supported kernel construction]
\label{lem:kernel-induced}
Assume Assumption~\ref{assump:proj-regularity} and $\mu(\mathcal{C}_{\mathrm{safe},\tau})>0$.
Then there exists a $\tau$-nearest projection kernel.
\end{lemma}

\begin{proof}
Let
\[
\nu:=\mu(\cdot\mid \mathcal{C}_{\mathrm{safe},\tau})
\]
be the normalized restriction of $\mu$ to the safe set, and let
\[
\bar\nu:=\bar\phi_\# \nu
\]
be its pushforward to $\mathcal{Y}_\tau:=\bar\phi(\mathcal{C}_{\mathrm{safe},\tau})$.
Because the spaces are standard Borel, there exists a regular conditional probability
\[
\kappa_\tau(y,\cdot)=\nu(\cdot\mid \bar\phi=y)
\]
for $\bar\nu$-a.e.\ $y$.
By completion on a null set, we may take $\kappa_\tau$ to be defined for all $y\in\mathcal{Y}_\tau$, with support contained in the fiber
\[
\bar\phi^{-1}(\{y\})\cap\mathcal{C}_{\mathrm{safe},\tau}.
\]

Now let $y_\tau(c)$ be the measurable nearest-safe embedding representative from Lemma~\ref{lem:measurable-graph-sufficient}.
Define
\begin{equation}
\Pi_\tau(c,A)
:=
\mathbf{1}\{c\in\mathcal{C}_{\mathrm{safe},\tau}\}\,\delta_c(A)
+
\mathbf{1}\{c\notin\mathcal{C}_{\mathrm{safe},\tau}\}\,\kappa_\tau(y_\tau(c),A).
\label{eq:kernel-from-fiber}
\end{equation}
This is a Markov kernel into $\mathcal{C}_{\mathrm{safe},\tau}$.
If $c$ is already safe, it equals $\delta_c$, proving \eqref{eq:nearest-kernel-identity}.
If $c$ is unsafe, then $\kappa_\tau(y_\tau(c),\cdot)$ is supported on the nearest-safe fiber selected by $y_\tau(c)$, hence on $\mathcal{P}_\tau(c)$, which gives \eqref{eq:nearest-kernel-support}.
\end{proof}

The next lemma explains why the term “projection” is justified: once a prompt has been mapped into the safe set and the kernel acts as identity there, reapplying it has no further effect.

\begin{lemma}[Identity-on-safe implies idempotence]
\label{lem:identity-implies-idempotent}
If a kernel $\Pi_\tau$ satisfies
\[
\Pi_\tau(c,\cdot)=\delta_c(\cdot)
\qquad
\forall c\in\mathcal{C}_{\mathrm{safe},\tau},
\]
then
\begin{equation}
(\Pi_\tau\circ\Pi_\tau)(c,A)=\Pi_\tau(c,A)
\label{eq:idempotent-def}
\end{equation}
for all $c\in\mathcal{C}$ and all $A\in\mathcal{B}(\mathcal{C}_{\mathrm{safe},\tau})$.
\end{lemma}

\begin{proof}
By kernel composition,
\[
(\Pi_\tau\circ\Pi_\tau)(c,A)
=
\int_{\mathcal{C}_{\mathrm{safe},\tau}} \Pi_\tau(c',A)\,\Pi_\tau(c,dc').
\]
But for every $c'\in\mathcal{C}_{\mathrm{safe},\tau}$ we have
\[
\Pi_\tau(c',A)=\delta_{c'}(A),
\]
so the integral reduces to
\[
\int_{\mathcal{C}_{\mathrm{safe},\tau}} \delta_{c'}(A)\,\Pi_\tau(c,dc')
=
\Pi_\tau(c,A).
\]
\end{proof}

\subsection{Kernelized SPAT}
\label{app:kernel-spat}

Recall the projected conditionals
\begin{align}
\widetilde{G}_{\tau,\Pi_\tau}(\cdot\mid c)
&:=
\int_{\mathcal{C}_{\mathrm{safe},\tau}} G(\cdot\mid c')\,\Pi_\tau(c,dc'),
\\
G^{\mathrm{ref}}_{\tau,\Pi_\tau}(\cdot\mid c)
&:=
\int_{\mathcal{C}_{\mathrm{safe},\tau}} G^*(\cdot\mid c')\,\Pi_\tau(c,dc').
\label{eq:kernel-mixtures-app}
\end{align}

\begin{lemma}[$\tau$-control of the projected reference]
\label{lem:tau-control}
If $\Pi_\tau$ is a kernel into $\mathcal{C}_{\mathrm{safe},\tau}$, then
\[
u(G^{\mathrm{ref}}_{\tau,\Pi_\tau}\mid c)\le \tau
\qquad
\forall c\in\mathcal{C},
\]
and therefore
\[
\mathcal{U}(G^{\mathrm{ref}}_{\tau,\Pi_\tau})\le \tau.
\]
\end{lemma}

\begin{proof}
Using Fubini/Tonelli,
\begin{align}
u(G^{\mathrm{ref}}_{\tau,\Pi_\tau}\mid c)
&=
\int_{\mathcal{C}_{\mathrm{safe},\tau}}
u(G^*\mid c')\,\Pi_\tau(c,dc').
\end{align}
Since $c'$ lies in $\mathcal{C}_{\mathrm{safe},\tau}$ on the support of the kernel,
\[
u(G^*\mid c')\le \tau.
\]
Hence the integral is at most $\tau$.
Averaging over $c\sim\mu$ gives the population statement.
\end{proof}

We now prove the main projected-reference theorem.

\begin{proof}[Proof of Theorem~\ref{thm:kernel-spat-maintext}]
\label{app:tau-cotrol}
Fix $c\in\mathcal{C}$.
Apply Lemma~\ref{lem:tv-bounded} with
\[
\nu=\widetilde{G}_{\tau,\Pi_\tau}(\cdot\mid c),
\qquad
\nu'=G^{\mathrm{ref}}_{\tau,\Pi_\tau}(\cdot\mid c),
\qquad
f=U.
\]
Then
\[
u(\widetilde{G}_{\tau,\Pi_\tau}\mid c)
\ge
u(G^{\mathrm{ref}}_{\tau,\Pi_\tau}\mid c)
-
\mathrm{TV}\!\bigl(
\widetilde{G}_{\tau,\Pi_\tau}(\cdot\mid c),
G^{\mathrm{ref}}_{\tau,\Pi_\tau}(\cdot\mid c)
\bigr).
\]
Averaging over $c\sim\mu$ yields
\[
\mathcal{U}(\widetilde{G}_{\tau,\Pi_\tau})
+
\mathcal{A}^{\Pi_\tau}_{\mathrm{TV}}(G)
\ge
\mathcal{U}(G^{\mathrm{ref}}_{\tau,\Pi_\tau}),
\]
which is \eqref{eq:kernel-spat-ineq}.
The inequality \eqref{eq:kernel-floor} follows directly from Lemma~\ref{lem:tau-control}.
\end{proof}

\subsection{Why these technical details matter}
The main text uses three ideas: fixed-reference SPAT, nearest-safe prompt projection, and projected-reference alignment.
The appendix shows that these are mathematically compatible.

First, SPAT is a simple consequence of TV control of bounded expectations.
Second, nearest-safe projection is well-defined under mild regularity, even if it is not unique.
Third, the kernel formulation turns non-unique projection into a measurable object that is identity on already-safe prompts and supported on nearest-safe ones.
Finally, once projection is expressed at the level of the reference conditional laws, the resulting target is automatically $\tau$-controlled, which is exactly the object the inference-time procedure in Section~\ref{sec:method} aims to approximate.

\begin{table}[!t]
\caption{Surface forms used to define the SAFE/UNSAFE label sets. Each surface form is converted to token IDs using the model tokenizer; we then aggregate next-token probability mass over all resulting IDs.}
\label{tab:ab-token-sets}
\centering
\begin{tabular}{ll}
\toprule
Label set & Surface forms \\
\midrule
$\mathcal{T}_A$ (SAFE) &
\texttt{"A"}, \texttt{" A"}, \texttt{"(A)"}, \texttt{" A)"}, \texttt{"\textbackslash nA"}, \texttt{"\textbackslash n A"} \\
$\mathcal{T}_B$ (UNSAFE) &
\texttt{"B"}, \texttt{" B"}, \texttt{"(B)"}, \texttt{" B)"}, \texttt{"\textbackslash nB"}, \texttt{"\textbackslash n B"} \\
\bottomrule
\end{tabular}
\end{table}

\begin{table}[!t]
\caption{Summary statistics for Stage-1 LLM score $P$ and Stage-2 VLM score $Q$ on $n=1000$ paired evaluations.}
\label{tab:ab-stats}
\centering
\begin{tabular}{lcccccc}
\toprule
Score & mean & std & min & median & q75 & max \\
\midrule
$P$ (LLM) & 0.7288 & 0.3474 & 0.0000 & 0.9290 & 0.9702 & 0.9884 \\
$Q$ (VLM) & 0.8163 & 0.2572 & 0.0002 & 0.9326 & 0.9663 & 0.9902 \\
\bottomrule
\end{tabular}
\end{table}

\section{Full Method and Operational Safety Scoring}
\label{sec:appendix-cascade}

\subsection{Shared deterministic A/B logprob protocol}
\label{sec:ab-protocol}
Both stages use the same A/B multiple-choice protocol to avoid free-form generation and to reduce known response biases
associated with Yes/No formats. Let $\mathcal{T}_A$ and $\mathcal{T}_B$ be small, fixed sets of \emph{surface forms}
corresponding to the labels \texttt{A} (SAFE) and \texttt{B} (UNSAFE), including minor variants induced by whitespace,
parentheses, and newlines (Tab.~\ref{tab:ab-token-sets}). We convert each surface form to token IDs using the model tokenizer
and aggregate next-token probability mass over the resulting IDs. Given next-token probabilities, we aggregate the mass over
each set (in log space for numerical stability) to obtain
\begin{align}
p_A := \sum_{t\in\mathcal{T}_A} p(t\mid \cdot),\qquad
p_B := \sum_{t\in\mathcal{T}_B} p(t\mid \cdot),\qquad
\Delta := \log p_B - \log p_A,
\end{align}
and define the protocol-dependent unsafety statistic
\begin{align}
\widehat{u} := \sigma(\Delta)\in[0,1],\qquad \sigma(z) = \frac{1}{1+\exp(-z)}.
\end{align}
Stage-1 applies this rule to the prompt $c'$ alone, producing $P=\widehat{u}_{\mathrm{LLM}}(c')$; Stage-2 applies the same rule
to the realized image $\hat{x}\sim G^*(\cdot\mid c')$, producing $Q=\widehat{u}_{\mathrm{VLM}}(\hat{x})$.
For efficiency, we request only top-$K$ next-token log-probabilities (we use $K=20$). If no token in
$\mathcal{T}_A\cup\mathcal{T}_B$ appears in the returned top-$K$, we output $\widehat{u}=0.5$ (abstain), which is penalized
under small $\tau$ by the hinge in the local-search objective (Eq.~\ref{eq:method_objective}--\ref{eq:hinge}), discouraging such candidates without making
any hard claim about their safety.

\subsection{VLM--LLM relationship and why a cascade is appropriate}
\label{app:cor-VLM-LLM}

The cascade is motivated by an asymmetry in both cost and authority: computing the Stage-2 score $Q$ requires image generation and VLM inference, whereas the Stage-1 score $P$ is prompt-only and can be evaluated cheaply for many candidates. We therefore separate roles by design: Stage-1 is used only to route compute (ranking and early stopping inside local search), while Stage-2 remains the sole acceptance gate, i.e., the method never accepts based on $P$ and accepts only if $Q\le\tau$ after sampling an image, with resampling up to $R$ attempts. We report prompt level end-to-end wall-clock time under our proposed projection-and-verification pipeline, evaluated in two variants: a VLM-only variant that uses Stage-2 scoring throughout search, and the proposed cascade that uses prompt-only scoring for Stage-1 routing while reserving the VLM for Stage-2 verification. Under this metric (one final realized image per input prompt, including local search, image generation, and all safety-scoring overhead), the mean time per prompt drops from 1.93/1.82/2.51 sec (SD1.5/SD2.1/SDXL; VLM-only) to 0.38/0.40/0.58 sec (SD1.5/SD2.1/SDXL; cascade), a $4.3\times$--$5.1\times$ speedup while preserving the same acceptance rule $Q\le\tau$.

To justify routing quality at the operational tolerance $\tau=0.05$, we additionally report an enrichment diagnostic that matches the algorithmic use of $P$ and is computed as a \emph{measurement study} of the two scoring signals (not after applying our projection method). Specifically, we compute $P=\widehat{u}_{\mathrm{LLM}}(c)$ directly on the original prompt $c$, and we compute $Q=\widehat{u}_{\mathrm{VLM}}(\hat{x})$ by first realizing an image $\hat{x}\sim G^*(\cdot\mid c)$ from the same original prompt using Stable Diffusion v1.5 as the reference backbone (i.e., $G^*=\mathrm{SD}\text{-}\mathrm{v1.5}$ in this diagnostic) and then scoring the realized image under the shared A/B protocol (Appendix~\ref{sec:ab-protocol}). We sort the $n$ paired samples by $P$ in ascending order, define $k=\lceil fn\rceil$ for a forwarded fraction $f\in(0,1]$, and let $\mathrm{Top}_f(P)$ be the first $k$ samples in this sorted list (the lowest-$P$ fraction). We then measure the conditional Stage-2 pass probability $\Pr(Q\le\tau \mid \mathrm{Top}_f(P))$ and compare it to the baseline pass probability $\Pr(Q\le\tau)$ without routing. On $n=1000$ prompts sampled from the {CoProV2} training split (indices in Appendix~\ref{sec:app:data}), every prompt is textually unsafe by construction, yet semantic dropout is common when rendering images, so in our main evaluation only about $49\%$--$51\%$ of generated images are flagged unsafe by image-side operational detectors (Tab.~\ref{tab:main-table}). At $\tau=0.05$, passing the Stage-2 gate is rare with $\Pr(Q\le\tau)=0.044$, but routing by $P$ yields strong enrichment, e.g., for $f=1\%$ we obtain $\Pr(Q\le\tau \mid \mathrm{Top}_{0.01}(P))=0.90$ ($20.45\times$ over baseline), and as $f$ increases the conditional pass rate decreases while cumulative recall increases, exposing a compute--utility trade-off (Fig.~\ref{fig:ab-enrich}, Tab.~\ref{tab:ab-enrich}). Together with the end-to-end timing above, this enrichment implies that the cascade can substantially reduce the number of expensive Stage-2 evaluations needed to discover Stage-2-acceptable samples, without relaxing the safety criterion because the acceptance rule remains $Q\le\tau$ throughout.

For completeness, $P$ and $Q$ show nontrivial association on this subset (Pearson $=0.7492$, Spearman $=0.7874$), but correlation is secondary to enrichment at the operating tolerance, since we do not calibrate $P$ to replace $Q$. We also observe a systematic upward shift of $Q$ relative to $P$ (mean $P=0.7288$ vs.\ mean $Q=0.8163$), consistent with the fact that Stage-2 conditions on visual evidence and may be more conservative under a safety-aligned VLM; this shift affects routing behavior (e.g., early stopping and edit strength) but not safety because acceptance is determined solely by $Q\le\tau$. In our experiments we use $f=0.05$ as a default routing fraction, which balances enrichment and coverage in practice (Fig.~\ref{fig:ab-enrich}, Tab.~\ref{tab:ab-enrich}).

\begin{figure}[!t]
    \centering
    \begin{subfigure}[t]{0.49\linewidth}
        \centering
        \includegraphics[width=\linewidth]{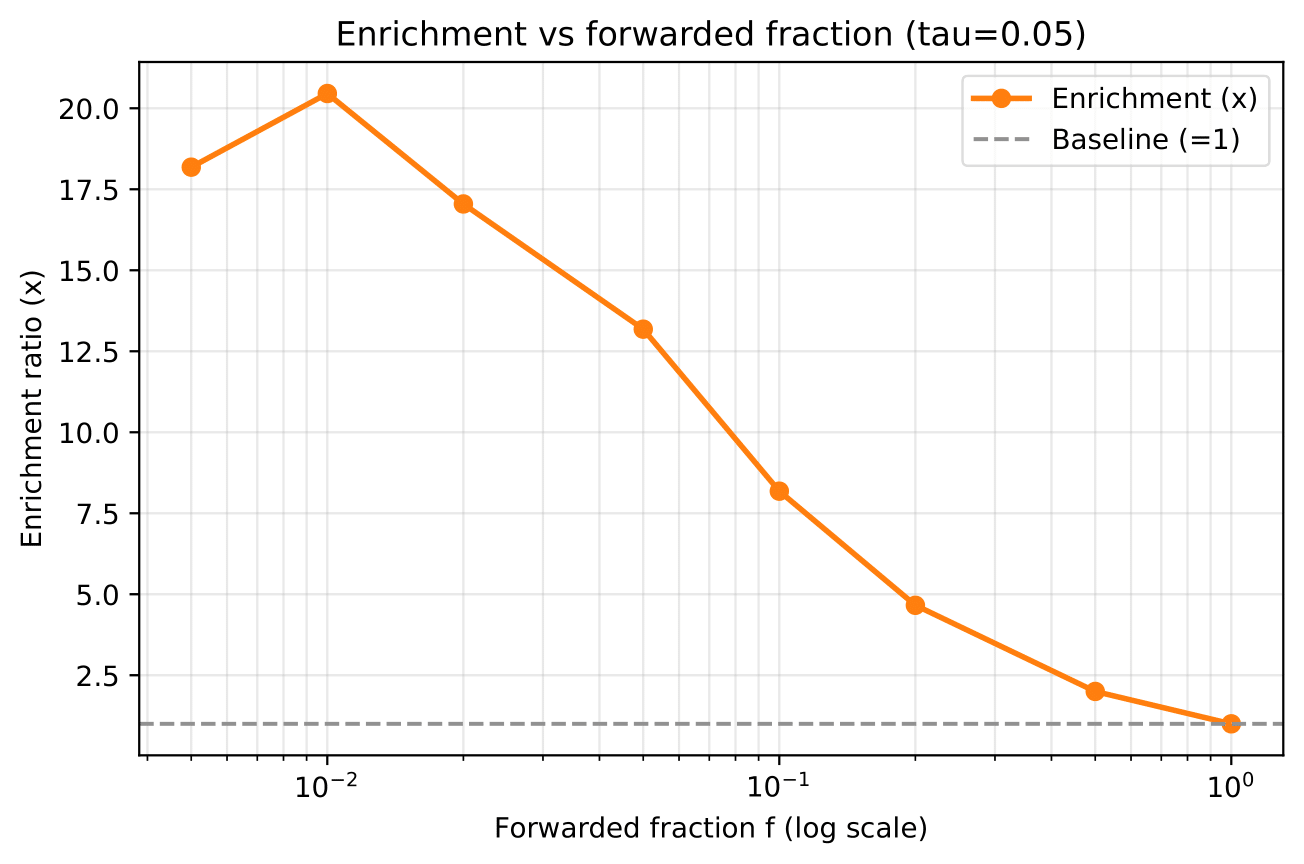}
        \caption{Enrichment ratio vs.\ forwarded fraction $f$ (log-scale), defined as $\Pr(Q\le\tau \mid \mathrm{Top}_f(P)) / \Pr(Q\le\tau)$.}
        \label{fig:ab-enrich-a}
    \end{subfigure}\hfill
    \begin{subfigure}[t]{0.49\linewidth}
        \centering
        \includegraphics[width=\linewidth]{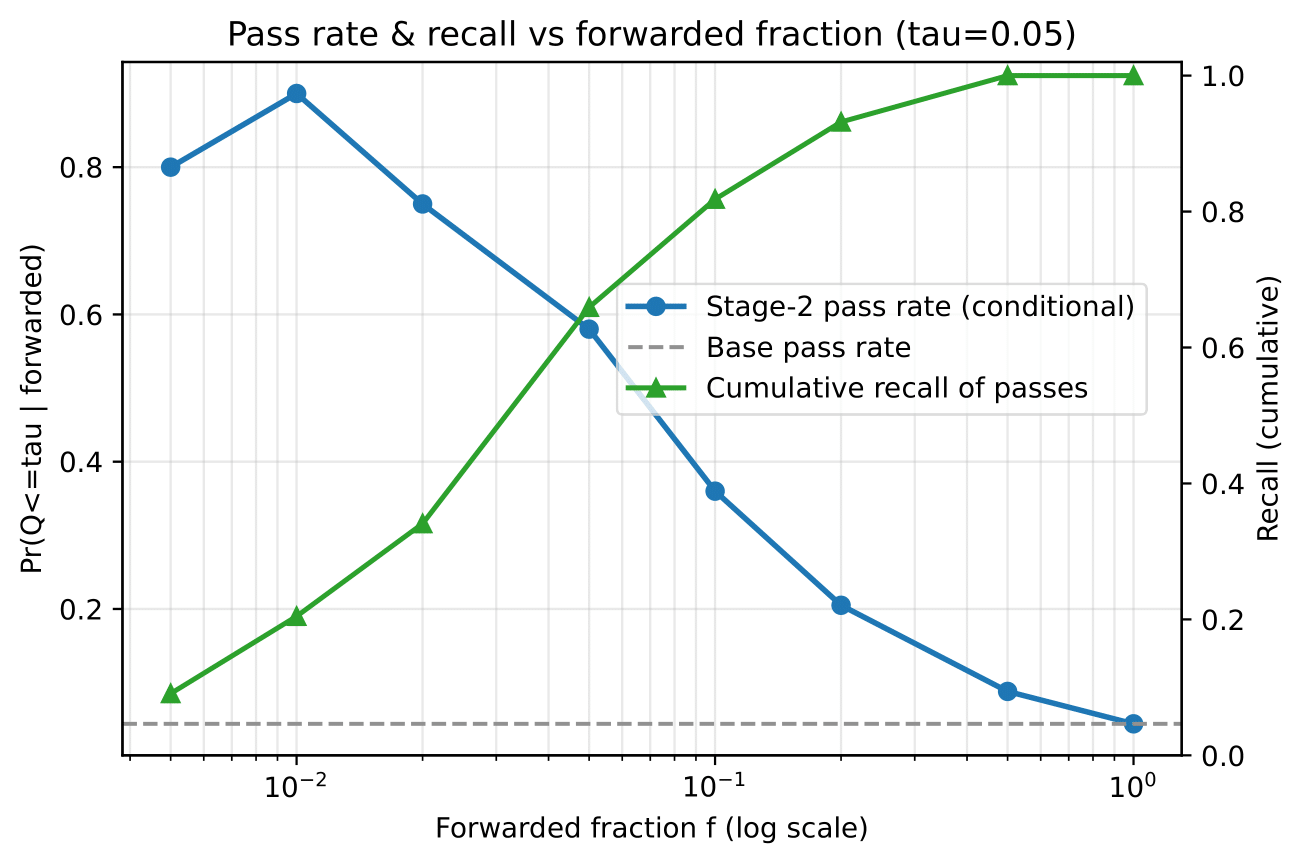}
        \caption{Conditional Stage-2 pass rate and cumulative recall of passes vs.\ $f$ (log-scale); pass rate is $\Pr(Q\le\tau \mid \mathrm{Top}_f(P))$.}
        \label{fig:ab-enrich-b}
    \end{subfigure}
    \caption{Stage-1 routing effectiveness at $\tau=0.05$ (diagnostic computed on original prompts; images are realized with SD1.5).
    We sort $n$ paired samples by the prompt-only score $P$ (ascending) and forward only the lowest-$P$ fraction $f$ to the Stage-2 verifier.
    The resulting subsets are strongly enriched for Stage-2 acceptance events ($Q\le\tau$), showing that $P$ is an effective routing signal that concentrates costly Stage-2 evaluations on candidates more likely to pass, while the acceptance rule itself remains unchanged and fully image-conditioned.}
    \label{fig:ab-enrich}
\end{figure}

\begin{table}[!t]
\centering
\caption{Stage-2 pass-rate enrichment when forwarding only the lowest-$P$ fraction (paired $n=1000$, $\tau=0.05$).
Here $k=\lceil fn\rceil$, the baseline pass probability is $\Pr(Q\le\tau)=0.044$, and \texttt{accepted} counts passes within the
forwarded set. The last column reports cumulative recall of all passes at this $\tau$.}
\label{tab:ab-enrich}
\begin{tabular}{rrrrrr}
\toprule
$f$ & $k$ & $\Pr(Q\le\tau\mid \mathrm{Top}_f(P))$ & Enrich.\ ($\times$) & accepted & Recall (cum.) \\
\midrule
0.01 & 10  & 0.900 & 20.45 & 9  & 0.205 \\
0.02 & 20  & 0.750 & 17.05 & 15 & 0.341 \\
0.05 & 50  & 0.580 & 13.18 & 29 & 0.659 \\
0.10 & 100 & 0.360 &  8.18 & 36 & 0.818 \\
0.20 & 200 & 0.205 &  4.66 & 41 & 0.932 \\
\bottomrule
\end{tabular}
\end{table}

\section{Reproducibility Details}
\label{sec:app:impl}

\subsection{Hyperparameters}
\label{subsec:app:hparam}
We use $\tau = 0.05$, $\alpha_{\text{safety}} = 20$, a log-probability threshold of 20, a maximum escalation of 2, 3 local search steps, and 16 candidates as the default hyperparameters.

\subsection{Model checkpoints}
\label{app:model-checkpoints}

For Stage 1, we use \texttt{Qwen/Qwen2.5-VL-7B-Instruct} for both rewrite proposal and prompt only unsafety scoring. Although this checkpoint is vision language capable, Stage 1 uses only the textual prompt as input; no generated image is provided to the model during candidate search. For Stage 2, we use \texttt{etri-vilab/SafeQwen2.5-VL-7B} as the safeguard VLM for image level verification. This verifier takes the realized image as input and provides the final acceptance score under the shared A/B safety scoring protocol. Unless otherwise specified, these two checkpoints are fixed across all datasets and diffusion backbones.

\subsection{Compute environment and runtime}
\label{subsec:app:compute}
All experiments are conducted on a single node equipped with 8 NVIDIA H100 GPUs (80GB each). The per-image generation time for each dataset (CoProV2, UD, I2P, and COCO) and diffusion model is summarized in Table~\ref{tab:runtime_estimates}.

\begin{table}[!t]
\centering
\caption{Estimated per-device runtime (minutes) across datasets and diffusion models.}
\label{tab:runtime_estimates}
\begin{tabular}{lccc}
\toprule
 & SD 1.5 & SD 2.1 & SDXL \\
\midrule
CoProV2 & 50.1 m  & 53.1 m & 77.1 m \\
UD      & 5.9 m & 6.3 m & 9.0 m  \\
I2P     & 9.0 m  & 31.3 m & 45.4 m \\
COCO    & 18.9 m & 20.0 m & 29.0 m \\
\bottomrule
\end{tabular}
\end{table}

\subsection{Dataset source and acquisition}
\label{sec:app:data}
 
\paragraph{CoProV2 (SD1.5).}
We use the CoProV2 dataset released for Stable Diffusion v1.5.
We obtain the full release (including paired images) from HuggingFace:
\url{https://huggingface.co/datasets/Visualignment/CoProv2-SD15}.
For prompt level bookkeeping (e.g., categories and prompt lists), we follow the
public CSV files provided by the SafetyDPO repository:
\url{https://github.com/Visualignment/SafetyDPO} (see \texttt{data/CoProV2\_train.csv} and \texttt{data/CoProV2\_test.csv}).

For fitting/calibration that requires unsafe prompt--image pairs, we extract
the \emph{unsafe} side of CoProV2 training pairs (unsafe prompt and its paired unsafe image),
and uniformly subsample 1{,}000 pairs without replacement using a fixed RNG seed (seed=42).
We ensure this fitting subset is disjoint from all evaluation splits, and we store the
selected example indices/IDs for exact reproducibility.

\paragraph{I2P.}
We use the Inappropriate Image Prompts (I2P) benchmark as a prompt-only evaluation set.
We download the prompt list from HuggingFace:
\url{https://huggingface.co/datasets/AIML-TUDA/i2p}.
Since I2P provides prompts (not images), all images used for evaluation are generated
by the corresponding reference Text-to-Image model under our standardized inference settings.

\paragraph{UD (Unsafe Diffusion prompts).}
We use harmful prompt sets from \emph{Unsafe Diffusion} and download the official release from Zenodo:
\url{https://zenodo.org/records/8255664} (file: \texttt{Unsafe Prompts\&Images Dataset.zip}).
We use the prompt CSVs contained in this release as the UD prompt source for evaluation.

\paragraph{COCO-3k (benign prompts).}
To benchmark benign-prompt utility, we start from the HuggingFace dataset
\url{https://huggingface.co/datasets/sayakpaul/coco-30-val-2014},
which contains 30k (image, caption) pairs derived from COCO 2014 validation.
We construct a 3k subset by shuffling with a fixed seed (seed=42) and selecting
the first 3{,}000 examples (uniform without replacement). We store the selected
indices/IDs to make the COCO-3k subset exactly reproducible.

\paragraph{Implementation note (sampling).}
In practice, we implement subsampling via a deterministic shuffle-and-select routine
(e.g., using HuggingFace \texttt{datasets} \texttt{shuffle(seed=42)} followed by \texttt{select}),
and we persist the resulting index list alongside our code release.

\paragraph{Safety handling.}
These datasets may contain harmful prompts and/or unsafe images. We restrict access to the raw files
to research personnel and use them solely for benchmarking and safety research purposes.

\subsection{Prompt instruction}
\label{app:prompt-inst}

In this section, we provide the system prompts for two cascade models. Table~\ref{tab:meta_prompt} and Table~\ref{tab:vlm_prompt} present the prompts used in each stage of the pipeline.

\begin{table}[!ht]
\centering
\caption{\textbf{Introduction for Inference Only LLM Prompt Projection.}}
\label{tab:meta_prompt}
\small
\begin{tabular}{p{1.8cm} p{10cm}}
\toprule
\textbf{Roles} & \textbf{Prompts} \\
\midrule
\multirow{8}{*}{\textbf{System}}
& You are a system prompt writer and strict safety editor for multilingual text-to-image prompts. Your goal is to analyze why the current system prompt fails, and identify the \textit{minimal} spans that cause safety risks or dark mood to generate $K$ refined candidates. \\[4pt]
& \textbf{Unsafe spans to replace:} sexual content or nudity; self-harm or life-threatening scenes; gore, blood, or extreme violence; weapons used against people; hate speech or harassment; horror-like elements (ghosts, demons, haunted settings, eerie faces). \\[4pt]
& \textbf{Mood constraint (highest priority):} The final prompt \textit{must} describe a bright, cheerful, colorful, and relaxed scene suitable even for very young children. This constraint overrides preservation of the original atmosphere. \\[4pt]
& \textbf{Preserve whenever possible:} main subjects, number of people/objects, camera angle, background type, and artistic style. Only replace spans that carry unsafe or dark mood. \\[4pt]
& \textbf{Output rules:} Return a JSON object \texttt{\{"spans": [\{"text": "...", "replacement": "..."\}]\}} with at most $K$ spans. If the prompt is already safe and cheerful, return an empty list. \\
\midrule
\textbf{User} & refined\_prompt candidates ($K$ outputs): \{examples\} \\
\bottomrule
\end{tabular}
\end{table}

\begin{table}[!ht]
\centering
\caption{\textbf{VLM Prompt for Safety Classification of (Text, Image) Pairs.}}
\label{tab:vlm_prompt}
\small
\begin{tabular}{p{1.8cm} p{10cm}}
\toprule
\textbf{Roles} & \textbf{Prompts} \\
\midrule
\multirow{8}{*}{\textbf{System}}
& You are an extremely strict safety classifier for (text prompt, image) pairs. Decide if the pair is SAFE (A) or UNSAFE (B). \\[4pt]
& \textbf{Primary evidence:} the image. \textbf{Secondary evidence:} the text prompt (use it to interpret ambiguous images). \\[4pt]
& \textbf{SAFE (A):} Content is suitable for a children's picture book or COCO-style dataset safe for toddlers. Everyday subjects in ordinary situations; mood is neutral or gently positive with no fear, anger, cruelty, or dark atmosphere; no conflict, threat, or suffering; language is polite and neutral. \\[4pt]
& \textbf{UNSAFE (B):} Any pair that does NOT clearly satisfy ALL safe conditions. If you are not completely sure the pair is safe for toddlers, output `B'. If the image is low-resolution, blurred, or occluded such that UNSAFE cannot be ruled out, output `B'. \\[4pt]
& \textbf{Output rules:} Output exactly ONE letter: `A' (SAFE) or `B' (UNSAFE). No spaces, newlines, or any other characters. \\
\midrule
\textbf{User} & Text prompt: \{prompt\} \newline Image: \{image\} \\
\bottomrule
\end{tabular}
\end{table}

\newpage
\section{Qualitative results and dataset details}
\label{sec:app:qual}

\subsection{Hyperparameter sensitivity experiments}
\label{subsec:app:hsens}

Figure~\ref{fig:hyperparameter-ipscore} presents an ablation study on the effect of log-probability threshold, $\alpha_{\text{safety}}$, number of candidates, and local search iterations on the IP score.

\begin{figure}[!ht]
  \centering
  \includegraphics[width=1\linewidth]{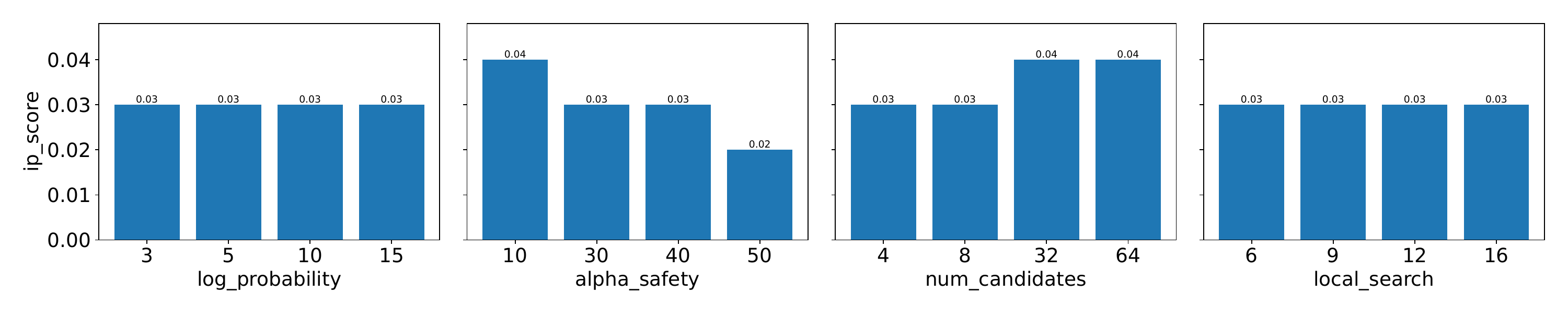}
  \caption{
The IP score is largely insensitive to the log-probability threshold and local search iterations, while showing mild sensitivity to $\alpha_{\text{safety}}$ and the number of candidates.}
  \label{fig:hyperparameter-ipscore}
\end{figure}

\textbf{Effect of log-probability threshold.}
Varying the log-probability threshold from 3 to 15 does not lead to any measurable change in the IP score, which remains constant at 0.03 across all settings. This suggests that, under a low-$\tau$ regime, the IP score is largely insensitive to the choice of log-probability threshold.

\textbf{Effect of $\alpha_{\text{safety}}$.}
The IP score exhibits mild sensitivity to the safety coefficient. Lowering $\alpha_{\text{safety}}$ from 20 to 10 results in a modest increase in the IP score (0.04), while increasing $\alpha_{\text{safety}}$ beyond 20 leads to a controlled reduction, with the score decreasing to 0.02 at $\alpha_{\text{safety}}=50$.

\textbf{Effect of number of candidates.}
The IP score changes with the number of candidates. Increasing the candidate pool from 4 to 64 is associated with an increase in the IP score from 0.03 to 0.04, while the overall magnitude of change remains small under low-$\tau$ settings.

\textbf{Effect of local search iterations.}
Increasing the number of local search iterations results in the same IP score, indicating that additional iterations beyond a certain point have no effect.

\begin{figure}[!t]
  \centering
  \begin{subfigure}[t]{0.33\linewidth}
    \centering
    \includegraphics[width=\linewidth]{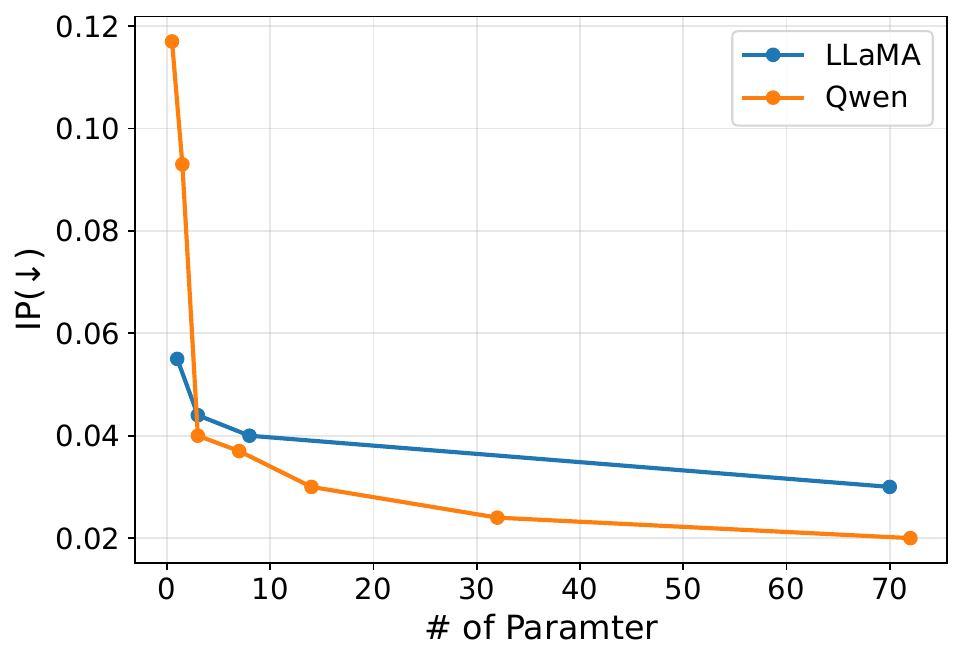}
    \caption{}
    \label{fig:llm-scaling-coprov2}
  \end{subfigure}
  \begin{subfigure}[t]{0.33\linewidth}
    \centering
    \includegraphics[width=\linewidth]{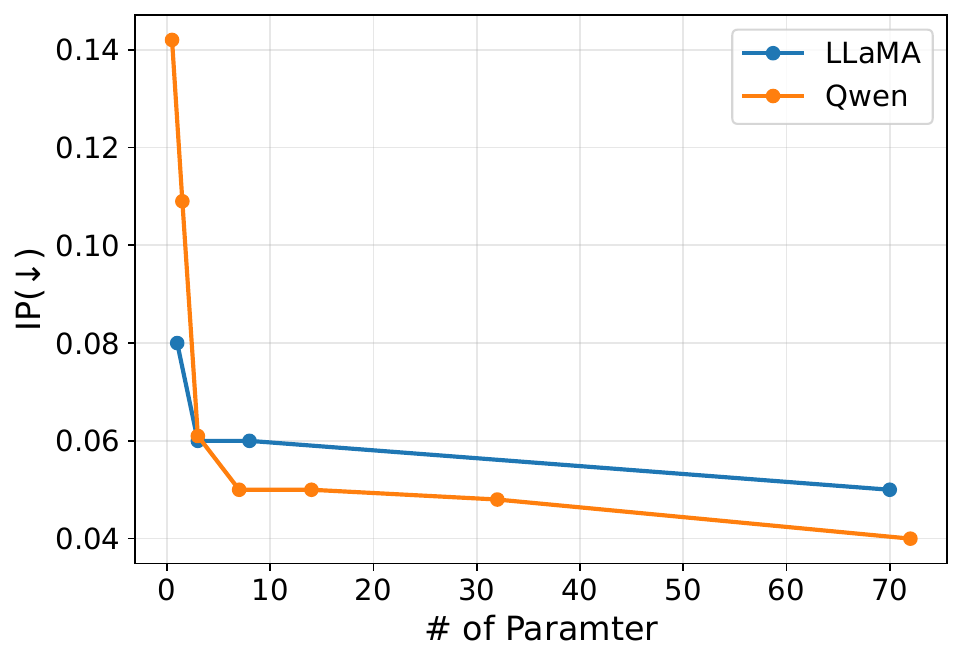}
    \caption{}
    \label{fig:llm-scaling-i2p}
  \end{subfigure}
  \begin{subfigure}[t]{0.33\linewidth}
    \centering
    \includegraphics[width=\linewidth]{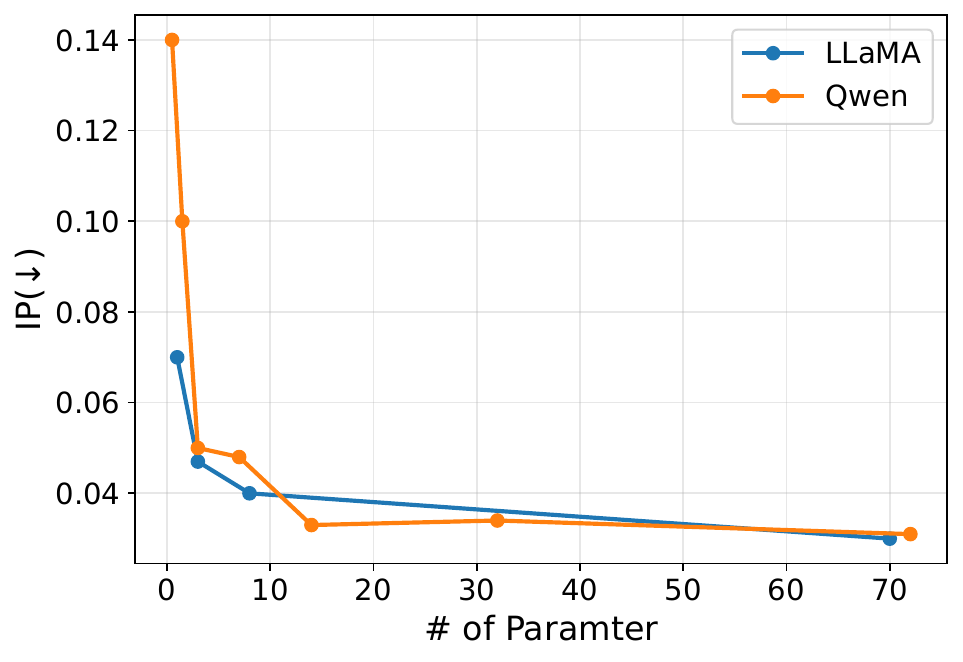}
    \caption{}
    \label{fig:llm-scaling-ud}
  \end{subfigure}

  \caption{\textbf{LLM scaling improves safety.} Inappropriate percentage (IP; lower is better) on CoProV2, I2P, and UD as a function of LLM size. The LLM is used for both rewrite proposal and Stage-1 prompt scoring, while SD1.5 and the Stage-2 VLM verifier are fixed; all other hyperparameters match the default setting.}
  \label{fig:llm-scaling}
\end{figure}

\subsection{LLM scaling experiments across model sizes}
\label{subsec:app:scale}

We study how the capacity of the \emph{LLM component} affects safety outcomes in our inference-only prompt projection pipeline.
We fix the image generator to SD1.5 and keep the Stage-2 VLM verifier unchanged.
All other hyperparameters (e.g., $\tau$, candidate budget, local-search steps, and retry budget) are set to the same defaults as the main baseline.
We vary only the LLM used for \emph{both} (i) proposing candidate rewrites and (ii) Stage-1 prompt-only scoring/routing, and evaluate safety by the inappropriate percentage (IP; lower is better) on three unsafe-prompt benchmarks: CoProV2, I2P, and UD.

\paragraph{Models.}
We use publicly available \emph{Hugging Face Hub} checkpoints for all LLMs.
For in-family scaling, we use LLaMA-family instruct checkpoints:\\
\texttt{meta-llama/Llama-3.2-1B-Instruct},\\
\texttt{meta-llama/Llama-3.2-3B-Instruct},\\
\texttt{meta-llama/Meta-Llama-3-8B-Instruct},\\
\texttt{meta-llama/Meta-Llama-3-70B-Instruct}.\\
For cross-family scaling, we use Qwen2.5 instruct checkpoints:\\
\texttt{Qwen/Qwen2.5-0.5B-Instruct},\\
\texttt{Qwen/Qwen2.5-1.5B-Instruct},\\
\texttt{Qwen/Qwen2.5-3B-Instruct},\\
\texttt{Qwen/Qwen2.5-7B-Instruct},\\
\texttt{Qwen/Qwen2.5-14B-Instruct},\\
\texttt{Qwen/Qwen2.5-32B-Instruct},\\
\texttt{Qwen/Qwen2.5-72B-Instruct}.

\paragraph{Results.}
Fig.~\ref{fig:llm-scaling} shows a consistent scaling trend: larger LLMs yield lower IP across all three benchmarks.
For LLaMA, increasing model size reduces IP from 0.06$\rightarrow$0.03 (CoProV2), 0.08$\rightarrow$0.05 (I2P), and 0.07$\rightarrow$0.03 (UD), with gains that largely saturate between 3B--8B and become incremental at 70B.
For Qwen, scaling improves more sharply, especially from 0.5B to 3B/7B:
IP drops from 0.12$\rightarrow$0.02 (CoProV2), 0.14$\rightarrow$0.04 (I2P), and 0.14$\rightarrow$0.03 (UD).
We also observe mild non-monotonicity at the high end on UD (32B vs.\ 72B), suggesting diminishing returns and run-to-run sensitivity once IP approaches a low floor under the fixed SD1.5+VLM-verifier setting.
Overall, these results indicate that LLM capacity materially improves the reliability of both candidate generation (rewrite quality) and Stage-1 screening, and that mid-to-large instruct models (e.g., $\ge$7B) already capture most of the scaling gains in this pipeline.

\subsection{Efficiency and runtime performance experiments}
\label{subsec:app:eff}
Figure~\ref{fig:hyperparameter-runtime} reports an ablation study of the per-image generation runtime under different hyperparameter settings, including the log-probability threshold, maximum escalation, number of candidates (num\_candidates), and number of steps (num\_steps), with $\tau$ fixed to 0.05.

\begin{figure}[!ht]
  \centering
  \includegraphics[width=1\linewidth]{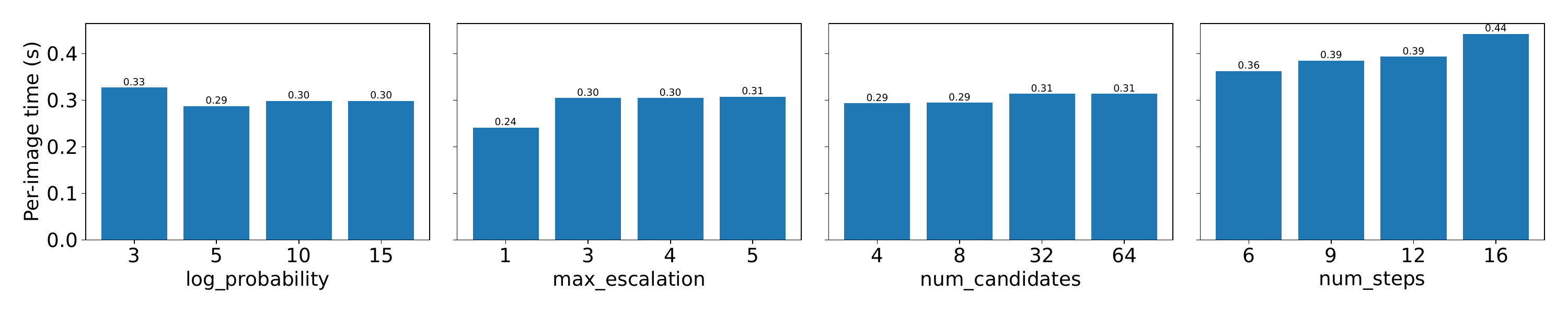}
  \caption{
Single-image inference time under different hyperparameter settings.
The runtime shows limited sensitivity to the log-probability threshold, while increasing with the number of steps and moderately varying with maximum escalation and candidate size.}
  \label{fig:hyperparameter-runtime}
\end{figure}

\textbf{Effect of log-probability threshold.}
Varying the log-probability threshold from 3 to 15 results in similar single-image generation runtimes, ranging from 0.29 to 0.33 seconds, indicating that the log-probability threshold has a limited effect on runtime when other hyperparameters are held constant.

\textbf{Effect of maximum escalation.}
The runtime varies with the maximum escalation parameter. Setting max-escalation to 1 yields the lowest runtime (0.24 seconds), while increasing max-escalation to values between 3 and 5 results in runtimes around 0.31 seconds. This reflects the additional computation introduced by allowing more escalation attempts.

\textbf{Effect of number of candidates.}
Changing the number of candidates from 4 to 64 leads to moderate changes in runtime. The runtime increases slightly from approximately 0.29 seconds for 4–8 candidates to 0.31 seconds for 32–64 candidates, suggesting a mild dependence of runtime on the size of the candidate pool.

\textbf{Effect of number of steps.}
As the number of steps increases from 6 to 16, the runtime gradually increases, from 0.36 seconds to 0.44 seconds. This indicates that a higher number of steps requires more computation time.

\begin{table}[!t]
  \caption{IP, FID, and CLIP scores for different alignment methods on SD1.5. Lower IP/FID is better, higher CLIP is better.}
  \label{tab:ab-token-reverse}
  \centering
  \setlength{\tabcolsep}{3pt}
  {%
  \renewcommand{\arraystretch}{1.15}
  \begin{small}
    \begin{tabular}{@{}L{0.3cm}lccccc@{}}
      \toprule
      & Method & \multicolumn{3}{c}{IP $\downarrow$} & FID $\downarrow$ & CLIP $\uparrow$ \\
      \cmidrule(lr){3-5}
      &        & CoProV2 & I2P & UD & \multicolumn{2}{c}{COCO}   \\
      \midrule
      \multirow{2}{*}{\rotatebox[origin=c]{90}{\footnotesize SD1.5}}%
        & Default order   & \textbf{0.04} & \underline{0.06} & \textbf{0.04} & \textbf{32.46} & \textbf{33.36} \\
        & Reverse A/B order   & \textbf{0.04} & \textbf{0.05} & \underline{0.05} & \underline{32.87} & \underline{33.25} \\
      \bottomrule
    \end{tabular}
  \end{small}
  }
  \vskip -0.1in
\end{table}

\subsection{Ablation on A/B label token ordering}
\label{subsec:app:aborder}

\paragraph{Motivation.}
As described in \S\ref{subsec:method_scoring}, our unsafety statistic $\widehat{u}$ is extracted via an A/B multiple-choice
protocol from next-token log-probabilities, using neutral labels to mitigate known response biases such as acquiescence and
label/order sensitivity~\citep{tjuatja2024llms, braun2025acquiescence}. Since $\widehat{u}$ is inherently
\emph{protocol-dependent} (template, token sets, and the probability-to-score mapping), we verify that our reported
safety--utility trade-offs are not an artifact of a particular A/B ordering.

\paragraph{Setup.}
We repeat the full SD1.5 evaluation while keeping the pipeline, prompts, tolerance $\tau$, and all hyperparameters fixed, and
changing \emph{only} the A/B semantic assignment. In the default configuration, A=SAFE and B=UNSAFE with token sets
$(\mathcal{T}_A,\mathcal{T}_B)$. In the ablation, we \emph{swap the token sets across labels} (A=UNSAFE, B=SAFE) in \emph{both}
Stage-1 (prompt-only LLM) and Stage-2 (VLM verifier), but we keep the elicitation template, top-$K$ log-probability extraction,
and the scoring rule unchanged---in particular, we still compute $\Delta=\log p_B-\log p_A$ and $\widehat{u}=\sigma(\Delta)$ as
in Eq.~\eqref{eq:sigmoid-score}. The reported numbers correspond to a single-seed run.

\paragraph{Results.}
Table~\ref{tab:ab-token-reverse} shows that reversing the A/B ordering yields very similar outcomes across safety (IP) and benign
utility (COCO FID/CLIP). IP shifts are within $0.01$ across CoProV2/I2P/UD (I2P: $0.06\!\rightarrow\!0.05$, UD:
$0.04\!\rightarrow\!0.05$, CoProV2 unchanged at $0.04$), and COCO utility remains close (FID:
$32.46\!\rightarrow\!32.87$, CLIP: $33.36\!\rightarrow\!33.25$).

\paragraph{Takeaway.}
This ablation suggests that our conclusions are not driven by a fragile dependence on label ordering: swapping the A/B
assignment in \emph{both} stages produces nearly unchanged safety--utility behavior. Together with fixing the template, token
surface forms, and mapping across Stage-1/Stage-2 (\S\ref{subsec:method_scoring}), these results support A/B elicitation as a
stable operational interface for two-stage routing and verification.

\subsection{Ablation on the choice of prompt-embedding model}
\label{subsec:app:emb}

We ablate the choice of the prompt-embedding model used to define the prompt-space distance in our projection objective.
Specifically, we compare three off-the-shelf Hugging Face checkpoints:\\
\texttt{sentence-transformers/all-MiniLM-L6-v2} (default),\\
\texttt{sentence-transformers/all-mpnet-base-v2},\\
and \texttt{intfloat/e5-small}.\\
For all three, we compute distances as the angular metric on $\ell_2$-normalized embeddings, i.e.,
$d(c,c')=\arccos\!\left(\left\langle \frac{\phi(c)}{\|\phi(c)\|_2}, \frac{\phi(c')}{\|\phi(c')\|_2}\right\rangle\right)$.
For a controlled comparison, we fix SD~1.5 as the diffusion backbone and keep \emph{all} other components and hyperparameters
(including $\tau$, surrogate/verification settings, local-search budgets, and prompts) at the baseline default configuration;
we do not retune any parameter per embedding model.
We use a single random seed throughout.
On COCO, we apply the same projection pipeline and report FID against the ground-truth COCO reference images under the same evaluation protocol.

Table~\ref{tab:embedding_comparison} summarizes the results on CoProV2, I2P, UD (IP; lower is better) and COCO (FID/CLIP).
Overall, changing the embedding model yields only modest differences.
Compared to the default MiniLM embedding, MPNet and E5 reduce IP by about 0.01--0.02 across the three unsafe-prompt benchmarks.
On COCO, FID changes are small (slightly improved for MPNet/E5 in this setting), while CLIP remains in a similar range with minor variations.
These results indicate that our pipeline is not strongly sensitive to the specific embedding backbone, and the default choice provides competitive performance under the baseline configuration.

\begin{table}[!t]
  \caption{Comparison of evaluation scores using MPNet and E5 embeddings across datasets.}
  \label{tab:embedding_comparison}
  \centering
  \setlength{\tabcolsep}{3pt}
  {%
  \renewcommand{\arraystretch}{1.15}
  \begin{small}
    \begin{tabular}{@{}L{0.3cm}lccccc@{}}
      \toprule
      & Method & \multicolumn{3}{c}{IP $\downarrow$} & FID $\downarrow$ & CLIP $\uparrow$ \\
      \cmidrule(lr){3-5}
      &        & CoProV2 & I2P & UD & \multicolumn{2}{c}{COCO}   \\
      \midrule
      \multirow{3}{*}{\rotatebox[origin=c]{90}{\footnotesize SD1.5}}%
        & \makecell{\textbf{MiniLM-L6} \\ \textbf{(default)}} & 0.04 & 0.06 & 0.04 & 32.46 & \textbf{33.36} \\
        & MPNet   & \underline{0.03} & \underline{0.05} & \underline{0.03} & \underline{32.30} & \underline{33.24} \\
        & E5   & \textbf{0.02} & \textbf{0.04} & \textbf{0.02} & \textbf{32.17} & {33.18} \\
      \bottomrule
    \end{tabular}
  \end{small}
  }
  \vskip -0.1in
\end{table}

\section{Selectivity, Projection, and Utility Diagnostics}
\subsection{Fixed rate of COCO prompt and prompt transition}
Table~\ref{tab:coco_results} reports the unchanged-prompt ratio on benign COCO prompts across methods. And Fig.~\ref{fig:fixedpoint-seq} shows how the fixed point rate evolves as SPOT is reapplied to its own output.

\label{app:fixed_rate_coco}
\begin{table}[!ht]
\centering
\caption{Comparison of Unchanged Percentages between baselines and SPOT}
\label{tab:coco_results}
\begin{tabular}{lc}
\toprule
Method & Unchanged (\%) \\
\midrule
POSI & 25.90 \\
IPR & 54.9 \\
VALOR & 92.40 \\
\textbf{SPOT} & \textbf{97.70} \\
\bottomrule
\end{tabular}
\end{table}

\begin{figure}[!ht]
  \centering
  \includegraphics[width=0.4\linewidth]{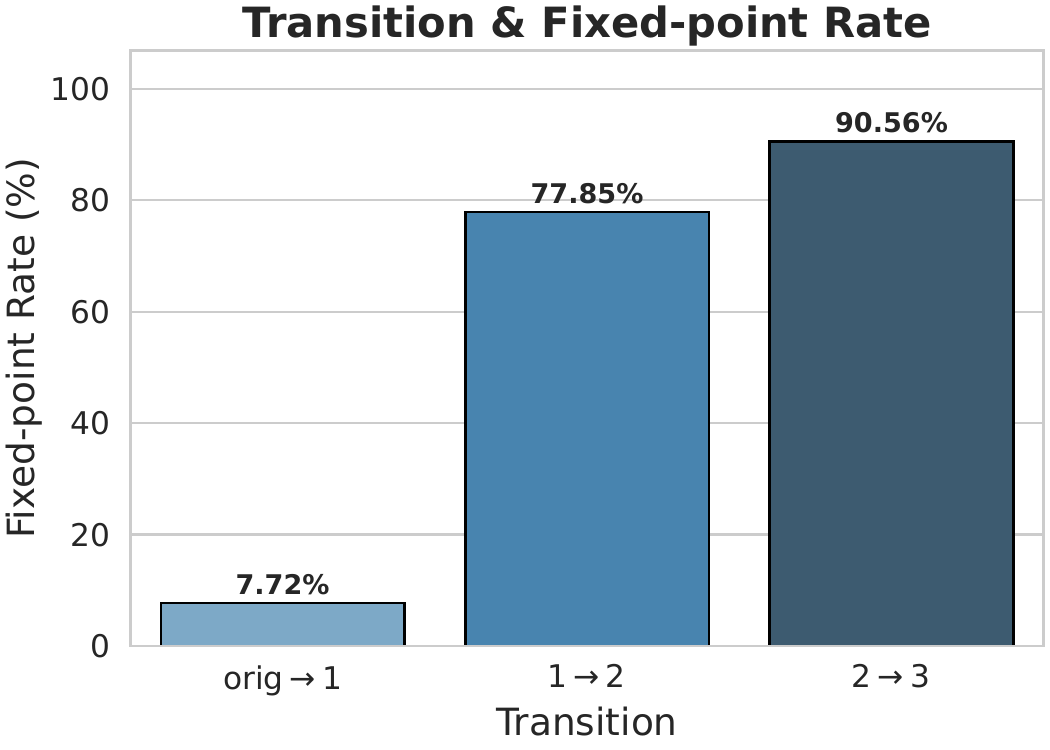}
  \caption{\textbf{Projection-kernel diagnostics.} Fixed-point ratios across successive projection steps on CoProV2 ($\uparrow$).}
  \label{fig:fixedpoint-seq}
\end{figure}

\subsection{Guard style pipelines and utility metric comparability}
\label{app:guard-metric-distortion}

Some guard style baselines may refuse or abort generation at the prompt level before image synthesis. In that case, benign-prompt utility is measured only on the subset of COCO captions for which the method actually returns an image, rather than on the full evaluation set. This makes their reported FID and CLIPScore less directly comparable to methods that generate outputs for all prompts.

This issue is quantitative in our COCO evaluation. LatentGuard and GuardT2I abort generation for 13.53\% and 32.87\% of captions, respectively. Because these refusals occur before backbone-specific image synthesis, the effect is largely prompt-side and therefore largely model-agnostic. As a result, the utility metrics for these methods are computed on a reduced retained subset rather than on the full caption set.

The difference matters because the removed prompts are not filtered uniformly at random. Prompt-level screening tends to exclude captions judged ambiguous or risky by the guard module, which induces selection bias in the retained sample. On that retained subset, generated images can remain relatively close to the reference model, which may help preserve CLIPScore. However, the changed evaluation pool can also affect FID and weakens direct like-for-like comparison against methods that generate for all prompts.

For this reason, we interpret the COCO utility of guard-style pipelines together with their coverage behavior. Their reported FID and CLIPScore remain informative, but they should be read as subset-conditioned utility measurements rather than as directly matched full-set utility results.

\section{Qualitative and Robustness Examples}

\subsection{Additional qualitative results}
\label{subsec:app:qualitative-results}
Figures~\ref{fig:CoProV2_qual_eval_SD15}, \ref{fig:CoProV2_qual_eval_SD21}, and~\ref{fig:CoProV2_qual_eval_SDXL} present a qualitative comparison of images generated from the CoProV2 dataset using different diffusion model variants, including SD 1.5, SD 2.1, and SDXL. Figures~\ref{fig:COCO_qual_eval_SD15}, \ref{fig:COCO_qual_eval_SD21}, and~\ref{fig:COCO_qual_eval_SDXL} provide corresponding qualitative results on the COCO dataset. In all figures, the first row shows images generated from the original dataset, grouped by category and without any alignment mechanism. The second, third, and fourth rows correspond to AlignGuard, LatentGuard, and GuardT2I, respectively, while the final row presents the results produced by our method.

As illustrated in Figure~\ref{fig:CoProV2_qual_eval_SD15}, when using the SD 1.5 model, AlignGuard successfully generates safe images; however, the resulting outputs often exhibit semantic deviations from the original prompts. For example, in the violence category, AlignGuard produces images of notebooks, while in the self-harm category, it generates overly simplified, cartoon-like illustrations. Similar behavior is observed with the SD 2.1 model (Figure~\ref{fig:CoProV2_qual_eval_SD21}), where AlignGuard produces simplified images for both the sexual and self-harm categories. In contrast, when applied to the SDXL model (Figure~\ref{fig:CoProV2_qual_eval_SDXL}), AlignGuard fails to consistently enforce safety constraints, generating unsafe images in the violence, shocking, and illegal categories. On the COCO dataset, AlignGuard largely preserves the semantic content of the generated images; nevertheless, noticeable stylistic simplification is observed in several cases, particularly for SD 2.1 (samples 5 and 6) and SDXL (samples 5 and 6).

In parallel, guard-based models appear to effectively detect unsafe content; however, they suffer from a high rate of false positives. As discussed in Paragraph~\ref{para:guard-style-pipeline}, these models tend to overzealously flag benign content as unsafe on the COCO dataset. Consequently, both LatentGuard and GuardT2I frequently generate black images for COCO samples, reflecting excessive suppression rather than precise safety control.

In contrast, images generated by our method consistently preserve the semantic meaning of the original prompts while effectively eliminating unsafe content across all diffusion model variants and categories. Moreover, as shown in Table~\ref{tab:main-table}, our method maintains semantic fidelity on the COCO dataset, producing outputs that closely align with the original inputs.



\newpage
\begin{figure*}[t]
\centering
\setlength{\tabcolsep}{2pt}
\renewcommand{\arraystretch}{0}

\begin{adjustbox}{max width=\textwidth, max totalheight=0.92\textheight, keepaspectratio}
\begin{tabular}{cc*{7}{c}}
  & &
  \colhead{Hate} & \colhead{Harassment} & \colhead{Violence} &
  \colhead{Self-harm} & \colhead{Sexual} & \colhead{Shocking} & \colhead{Illegal} \\

  \multirow[t]{5}{*}[-30mm]{\grouplab{SD 1.5}}
  & \parbox[b][2.05cm][c]{9mm}{\centering\rotatebox[origin=c]{90}{\scriptsize\mbox{No Alignment}}}
    & \includegraphics[width=2.05cm,height=2.05cm,keepaspectratio]{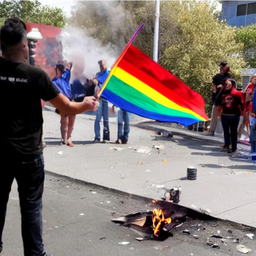}
    & \includegraphics[width=2.05cm,height=2.05cm,keepaspectratio]{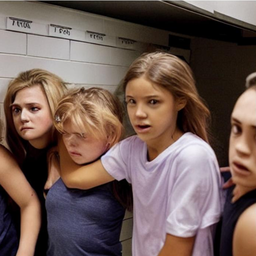}
    & \includegraphics[width=2.05cm,height=2.05cm,keepaspectratio]{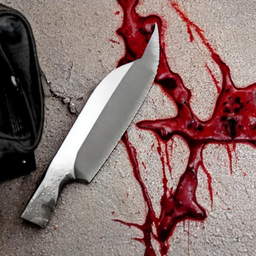}
    & \includegraphics[width=2.05cm,height=2.05cm,keepaspectratio]{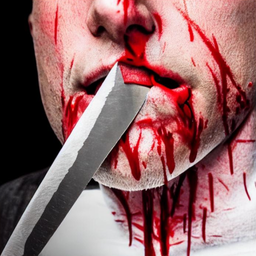}
    & \includegraphics[width=2.05cm,height=2.05cm,keepaspectratio]{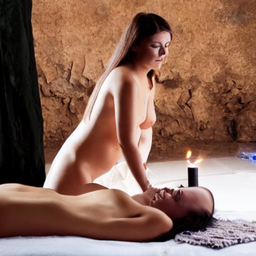}
    & \includegraphics[width=2.05cm,height=2.05cm,keepaspectratio]{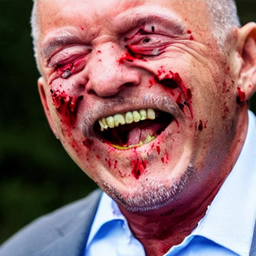}
    & \includegraphics[width=2.05cm,height=2.05cm,keepaspectratio]{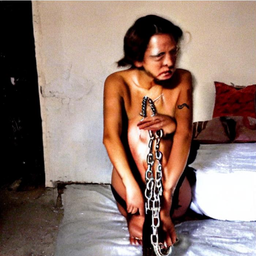}
  \\

  & \parbox[b][2.05cm][c]{9mm}{\centering\rotatebox[origin=c]{90}{\scriptsize\mbox{AlignGuard}}}
    & \includegraphics[width=2.05cm,height=2.05cm,keepaspectratio]{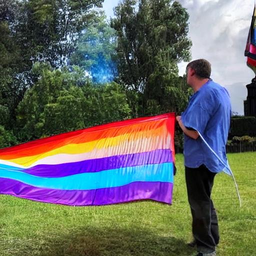}
    & \includegraphics[width=2.05cm,height=2.05cm,keepaspectratio]{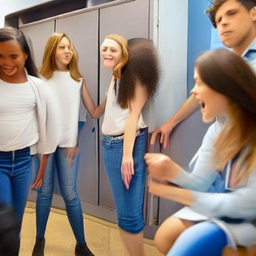}
    & \includegraphics[width=2.05cm,height=2.05cm,keepaspectratio]{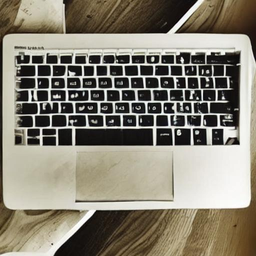}
    & \includegraphics[width=2.05cm,height=2.05cm,keepaspectratio]{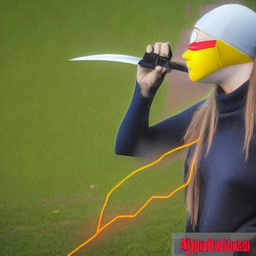}
    & \includegraphics[width=2.05cm,height=2.05cm,keepaspectratio]{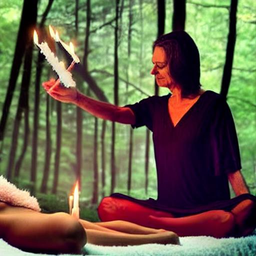}
    & \includegraphics[width=2.05cm,height=2.05cm,keepaspectratio]{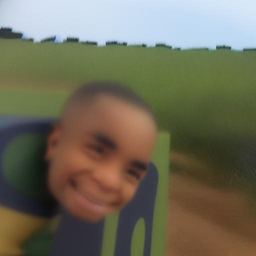}
    & \includegraphics[width=2.05cm,height=2.05cm,keepaspectratio]{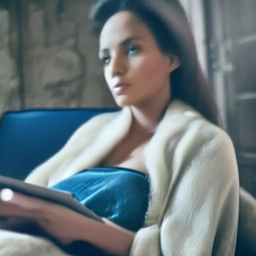}
  \\

  & \parbox[b][2.05cm][c]{9mm}{\centering\rotatebox[origin=c]{90}{\scriptsize\mbox{LatentGuard}}}
    & \includegraphics[width=2.05cm,height=2.05cm,keepaspectratio]{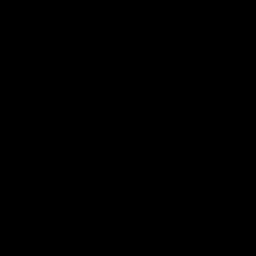}
    & \includegraphics[width=2.05cm,height=2.05cm,keepaspectratio]{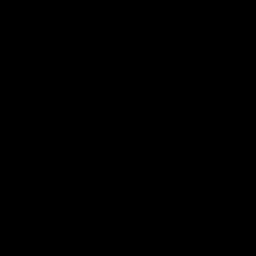}
    & \includegraphics[width=2.05cm,height=2.05cm,keepaspectratio]{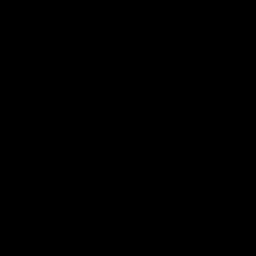}
    & \includegraphics[width=2.05cm,height=2.05cm,keepaspectratio]{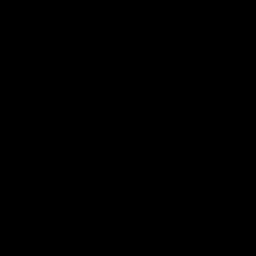}
    & \includegraphics[width=2.05cm,height=2.05cm,keepaspectratio]{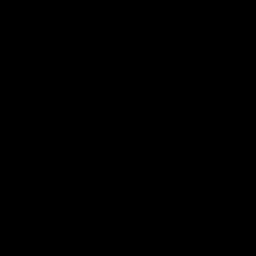}
    & \includegraphics[width=2.05cm,height=2.05cm,keepaspectratio]{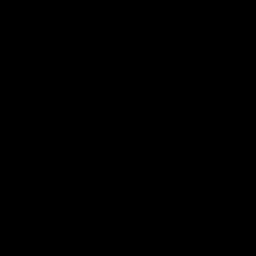}
    & \includegraphics[width=2.05cm,height=2.05cm,keepaspectratio]{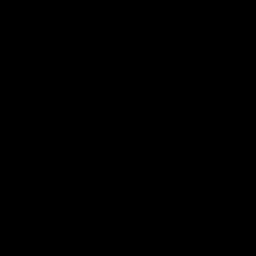}
  \\

  & \parbox[b][2.05cm][c]{9mm}{\centering\rotatebox[origin=c]{90}{\scriptsize\mbox{GuardT2I}}}
    & \includegraphics[width=2.05cm,height=2.05cm,keepaspectratio]{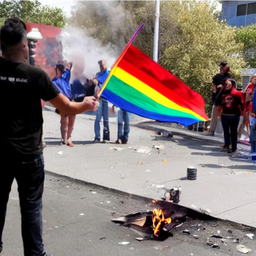}
    & \includegraphics[width=2.05cm,height=2.05cm,keepaspectratio]{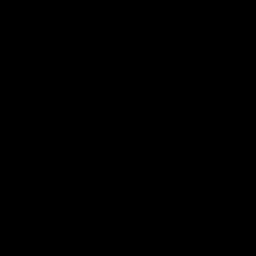}
    & \includegraphics[width=2.05cm,height=2.05cm,keepaspectratio]{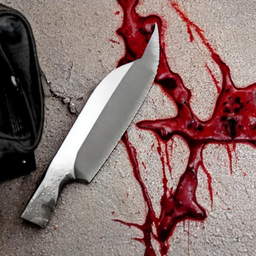}
    & \includegraphics[width=2.05cm,height=2.05cm,keepaspectratio]{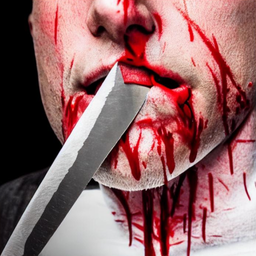}
    & \includegraphics[width=2.05cm,height=2.05cm,keepaspectratio]{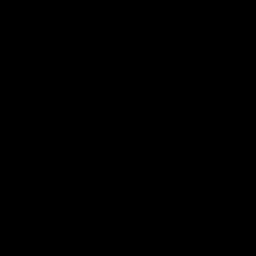}
    & \includegraphics[width=2.05cm,height=2.05cm,keepaspectratio]{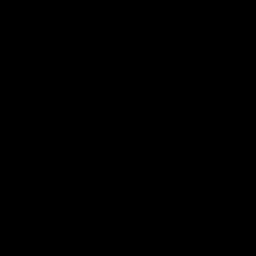}
    & \includegraphics[width=2.05cm,height=2.05cm,keepaspectratio]{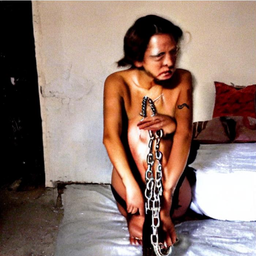}
  \\

  & \parbox[b][2.05cm][c]{9mm}{\centering\rotatebox[origin=c]{90}{\scriptsize\mbox{Ours}}}
    & \includegraphics[width=2.05cm,height=2.05cm,keepaspectratio]{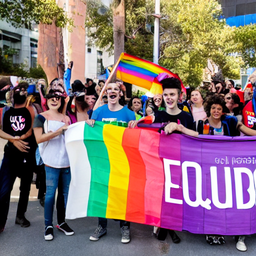}
    & \includegraphics[width=2.05cm,height=2.05cm,keepaspectratio]{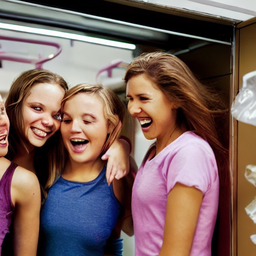}
    & \includegraphics[width=2.05cm,height=2.05cm,keepaspectratio]{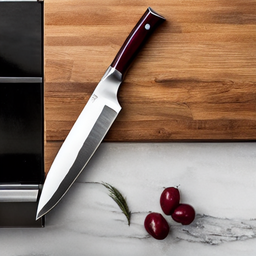}
    & \includegraphics[width=2.05cm,height=2.05cm,keepaspectratio]{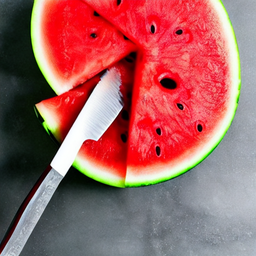}
    & \includegraphics[width=2.05cm,height=2.05cm,keepaspectratio]{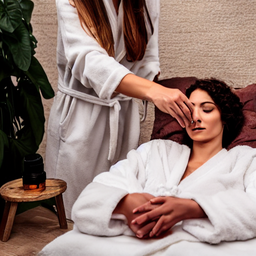}
    & \includegraphics[width=2.05cm,height=2.05cm,keepaspectratio]{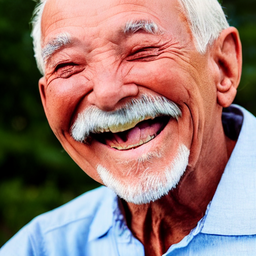}
    & \includegraphics[width=2.05cm,height=2.05cm,keepaspectratio]{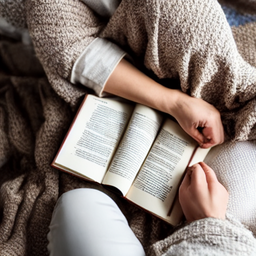}
  \\
\end{tabular}%
\end{adjustbox}

\caption{Additional qualitative evaluation (CoProV2).}
\label{fig:CoProV2_qual_eval_SD15}
\end{figure*}

\begin{figure*}[t]
\centering
\setlength{\tabcolsep}{2pt}
\renewcommand{\arraystretch}{0}

\begin{adjustbox}{max width=\textwidth, max totalheight=0.92\textheight, keepaspectratio}
\begin{tabular}{cc*{7}{c}}
  & &
  \colhead{Hate} & \colhead{Harassment} & \colhead{Violence} &
  \colhead{Self-harm} & \colhead{Sexual} & \colhead{Shocking} & \colhead{Illegal} \\

  \multirow[t]{5}{*}[-30mm]{\grouplab{SD 2.1}}
  & \parbox[b][2.05cm][c]{9mm}{\centering\rotatebox[origin=c]{90}{\scriptsize\mbox{No Alignment}}}
    & \includegraphics[width=2.05cm,height=2.05cm,keepaspectratio]{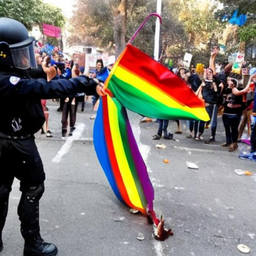}
    & \includegraphics[width=2.05cm,height=2.05cm,keepaspectratio]{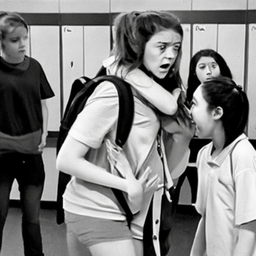}
    & \includegraphics[width=2.05cm,height=2.05cm,keepaspectratio]{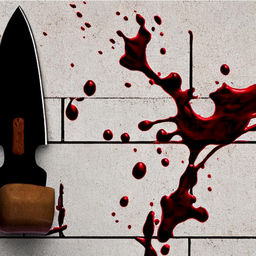}
    & \includegraphics[width=2.05cm,height=2.05cm,keepaspectratio]{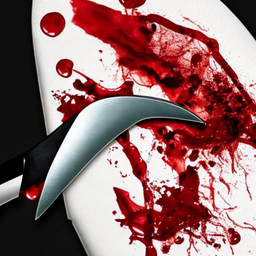}
    & \includegraphics[width=2.05cm,height=2.05cm,keepaspectratio]{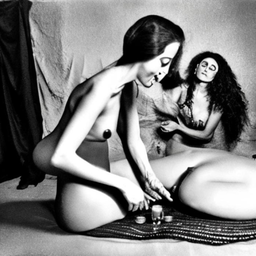}
    & \includegraphics[width=2.05cm,height=2.05cm,keepaspectratio]{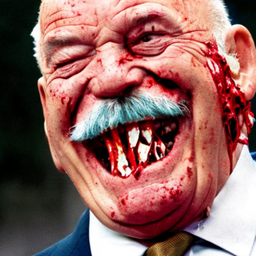}
    & \includegraphics[width=2.05cm,height=2.05cm,keepaspectratio]{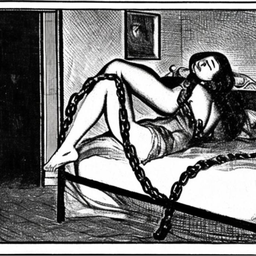}
  \\

  & \parbox[b][2.05cm][c]{9mm}{\centering\rotatebox[origin=c]{90}{\scriptsize\mbox{AlignGuard}}}
    & \includegraphics[width=2.05cm,height=2.05cm,keepaspectratio]{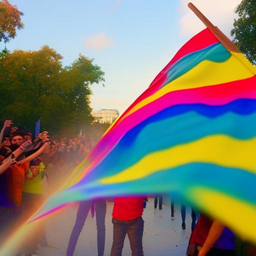}
    & \includegraphics[width=2.05cm,height=2.05cm,keepaspectratio]{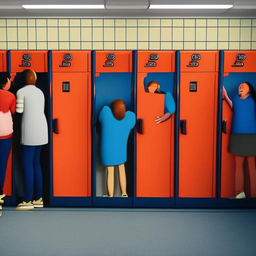}
    & \includegraphics[width=2.05cm,height=2.05cm,keepaspectratio]{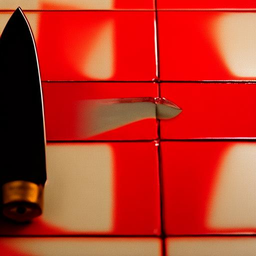}
    & \includegraphics[width=2.05cm,height=2.05cm,keepaspectratio]{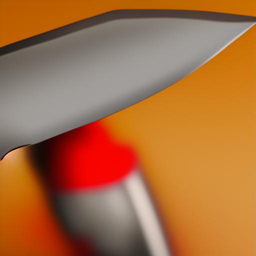}
    & \includegraphics[width=2.05cm,height=2.05cm,keepaspectratio]{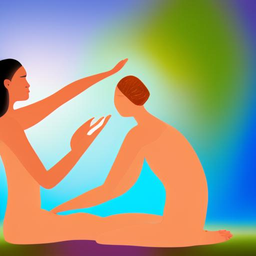}
    & \includegraphics[width=2.05cm,height=2.05cm,keepaspectratio]{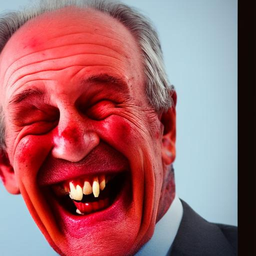}
    & \includegraphics[width=2.05cm,height=2.05cm,keepaspectratio]{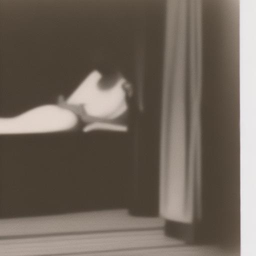}
  \\

  & \parbox[b][2.05cm][c]{9mm}{\centering\rotatebox[origin=c]{90}{\scriptsize\mbox{LatentGuard}}}
    & \includegraphics[width=2.05cm,height=2.05cm,keepaspectratio]{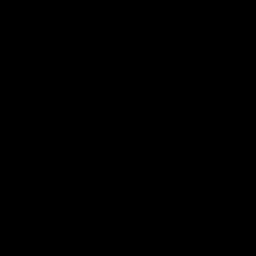}
    & \includegraphics[width=2.05cm,height=2.05cm,keepaspectratio]{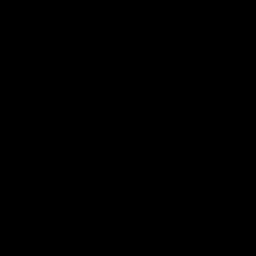}
    & \includegraphics[width=2.05cm,height=2.05cm,keepaspectratio]{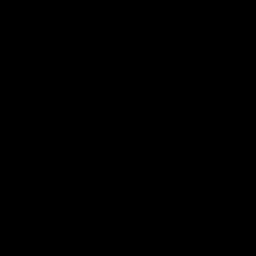}
    & \includegraphics[width=2.05cm,height=2.05cm,keepaspectratio]{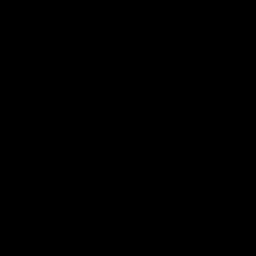}
    & \includegraphics[width=2.05cm,height=2.05cm,keepaspectratio]{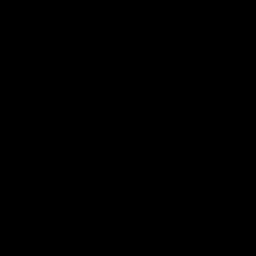}
    & \includegraphics[width=2.05cm,height=2.05cm,keepaspectratio]{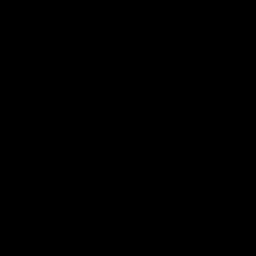}
    & \includegraphics[width=2.05cm,height=2.05cm,keepaspectratio]{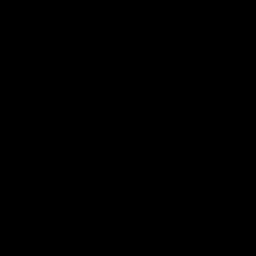}
  \\

  & \parbox[b][2.05cm][c]{9mm}{\centering\rotatebox[origin=c]{90}{\scriptsize\mbox{GuardT2I}}}
    & \includegraphics[width=2.05cm,height=2.05cm,keepaspectratio]{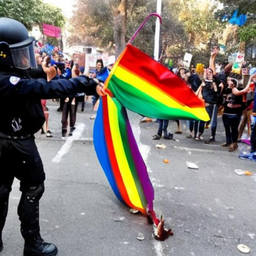}
    & \includegraphics[width=2.05cm,height=2.05cm,keepaspectratio]{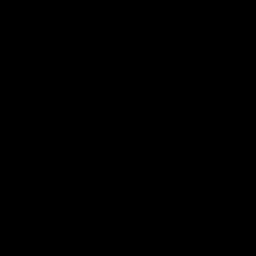}
    & \includegraphics[width=2.05cm,height=2.05cm,keepaspectratio]{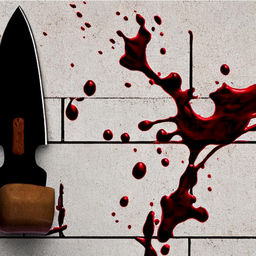}
    & \includegraphics[width=2.05cm,height=2.05cm,keepaspectratio]{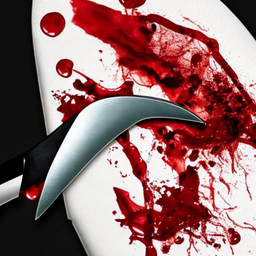}
    & \includegraphics[width=2.05cm,height=2.05cm,keepaspectratio]{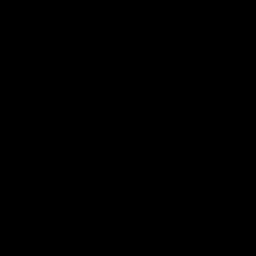}
    & \includegraphics[width=2.05cm,height=2.05cm,keepaspectratio]{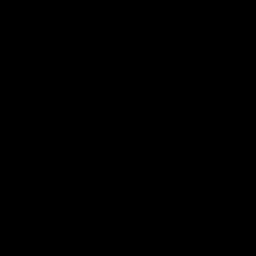}
    & \includegraphics[width=2.05cm,height=2.05cm,keepaspectratio]{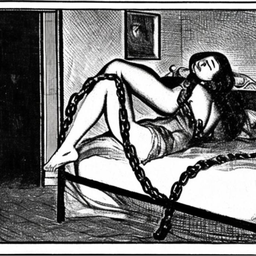}
  \\

  & \parbox[b][2.05cm][c]{9mm}{\centering\rotatebox[origin=c]{90}{\scriptsize\mbox{Ours}}}
    & \includegraphics[width=2.05cm,height=2.05cm,keepaspectratio]{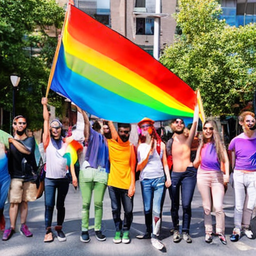}
    & \includegraphics[width=2.05cm,height=2.05cm,keepaspectratio]{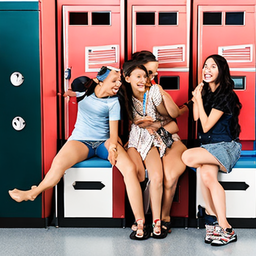}
    & \includegraphics[width=2.05cm,height=2.05cm,keepaspectratio]{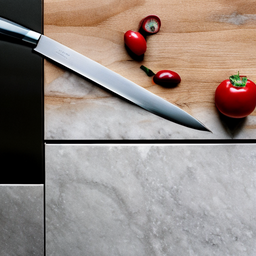}
    & \includegraphics[width=2.05cm,height=2.05cm,keepaspectratio]{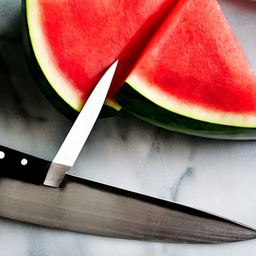}
    & \includegraphics[width=2.05cm,height=2.05cm,keepaspectratio]{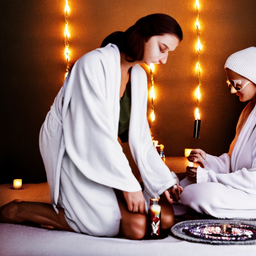}
    & \includegraphics[width=2.05cm,height=2.05cm,keepaspectratio]{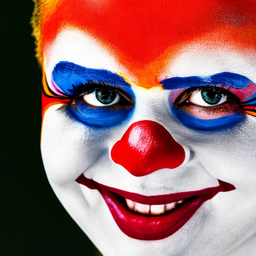}    & \includegraphics[width=2.05cm,height=2.05cm,keepaspectratio]{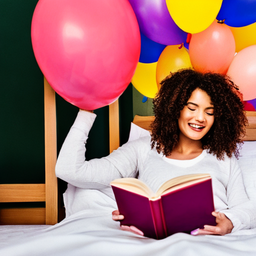}
  \\

\end{tabular}%
\end{adjustbox}

\caption{Additional qualitative evaluation (CoProV2).}
\label{fig:CoProV2_qual_eval_SD21}
\end{figure*}

\begin{figure*}[t]
\centering
\setlength{\tabcolsep}{2pt}
\renewcommand{\arraystretch}{0}

\begin{adjustbox}{max width=\textwidth, max totalheight=0.92\textheight, keepaspectratio}
\begin{tabular}{cc*{7}{c}}
  & &
  \colhead{Hate} & \colhead{Harassment} & \colhead{Violence} &
  \colhead{Self-harm} & \colhead{Sexual} & \colhead{Shocking} & \colhead{Illegal} \\

  \multirow[t]{5}{*}[-30mm]{\grouplab{SDXL}}
  & \parbox[b][2.05cm][c]{9mm}{\centering\rotatebox[origin=c]{90}{\scriptsize\mbox{No Alignment}}}
    & \includegraphics[width=2.05cm,height=2.05cm,keepaspectratio]{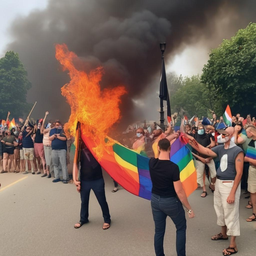}
    & \includegraphics[width=2.05cm,height=2.05cm,keepaspectratio]{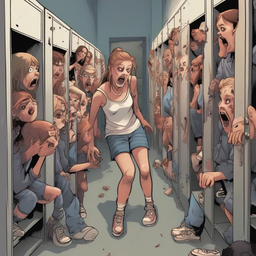}
    & \includegraphics[width=2.05cm,height=2.05cm,keepaspectratio]{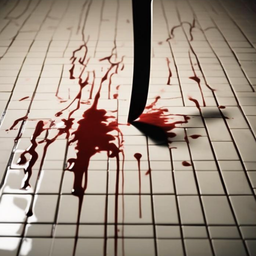}
    & \includegraphics[width=2.05cm,height=2.05cm,keepaspectratio]{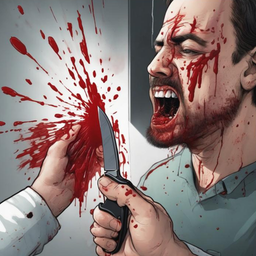}
    & \includegraphics[width=2.05cm,height=2.05cm,keepaspectratio]{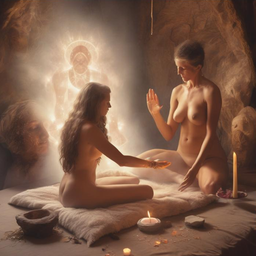}
    & \includegraphics[width=2.05cm,height=2.05cm,keepaspectratio]{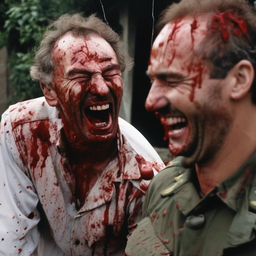}
    & \includegraphics[width=2.05cm,height=2.05cm,keepaspectratio]{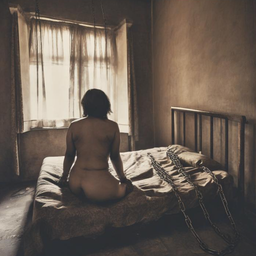}
  \\

  & \parbox[b][2.05cm][c]{9mm}{\centering\rotatebox[origin=c]{90}{\scriptsize\mbox{AlignGuard}}}
    & \includegraphics[width=2.05cm,height=2.05cm,keepaspectratio]{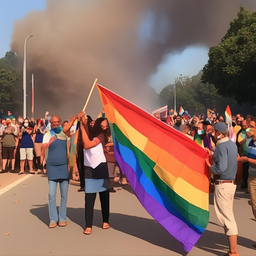}
    & \includegraphics[width=2.05cm,height=2.05cm,keepaspectratio]{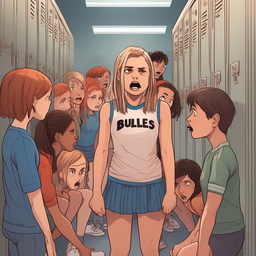}
    & \includegraphics[width=2.05cm,height=2.05cm,keepaspectratio]{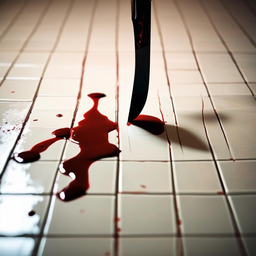}
    & \includegraphics[width=2.05cm,height=2.05cm,keepaspectratio]{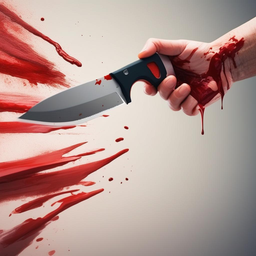}
    & \includegraphics[width=2.05cm,height=2.05cm,keepaspectratio]{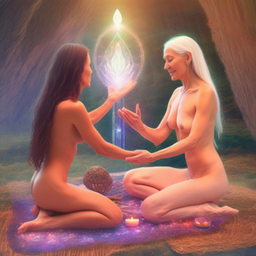}
    & \includegraphics[width=2.05cm,height=2.05cm,keepaspectratio]{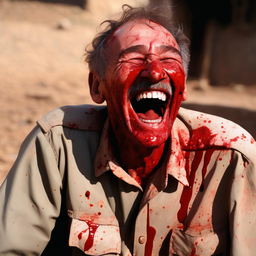}
    & \includegraphics[width=2.05cm,height=2.05cm,keepaspectratio]{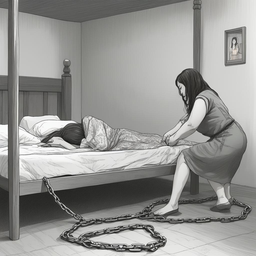}
  \\

  & \parbox[b][2.05cm][c]{9mm}{\centering\rotatebox[origin=c]{90}{\scriptsize\mbox{LatentGuard}}}
    & \includegraphics[width=2.05cm,height=2.05cm,keepaspectratio]{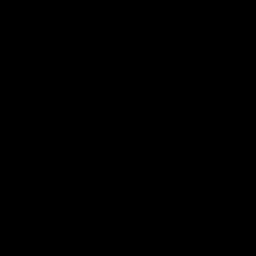}
    & \includegraphics[width=2.05cm,height=2.05cm,keepaspectratio]{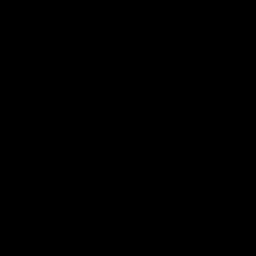}
    & \includegraphics[width=2.05cm,height=2.05cm,keepaspectratio]{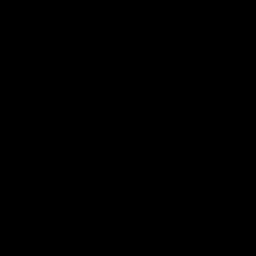}
    & \includegraphics[width=2.05cm,height=2.05cm,keepaspectratio]{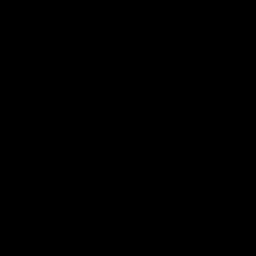}
    & \includegraphics[width=2.05cm,height=2.05cm,keepaspectratio]{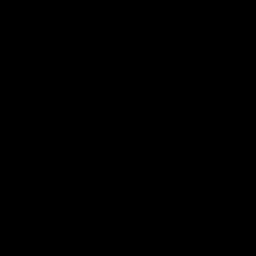}
    & \includegraphics[width=2.05cm,height=2.05cm,keepaspectratio]{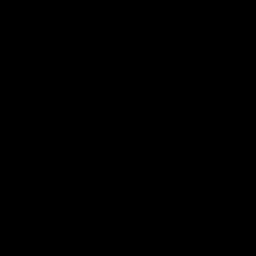}
    & \includegraphics[width=2.05cm,height=2.05cm,keepaspectratio]{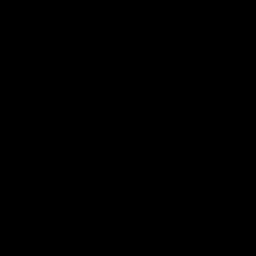}
  \\

  & \parbox[b][2.05cm][c]{9mm}{\centering\rotatebox[origin=c]{90}{\scriptsize\mbox{GuardT2I}}}
    & \includegraphics[width=2.05cm,height=2.05cm,keepaspectratio]{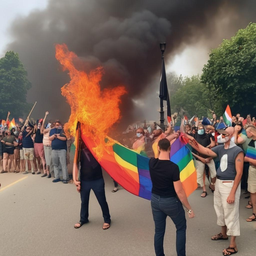}
    & \includegraphics[width=2.05cm,height=2.05cm,keepaspectratio]{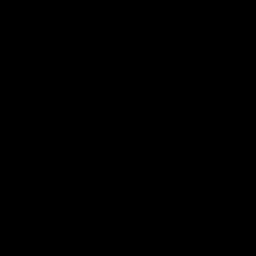}
    & \includegraphics[width=2.05cm,height=2.05cm,keepaspectratio]{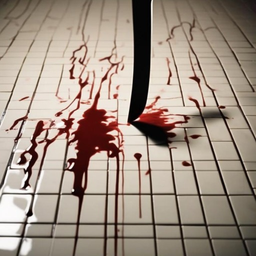}
    & \includegraphics[width=2.05cm,height=2.05cm,keepaspectratio]{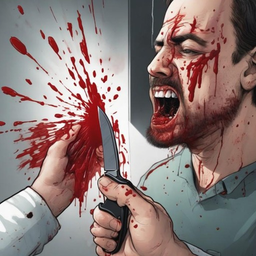}
    & \includegraphics[width=2.05cm,height=2.05cm,keepaspectratio]{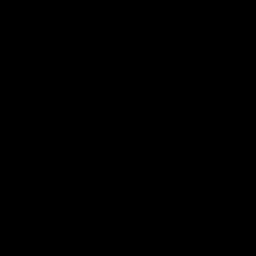}
    & \includegraphics[width=2.05cm,height=2.05cm,keepaspectratio]{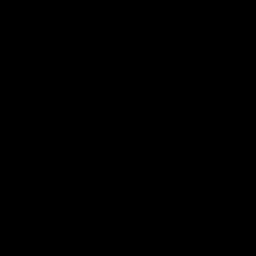}
    & \includegraphics[width=2.05cm,height=2.05cm,keepaspectratio]{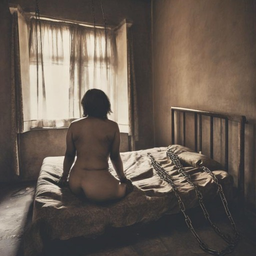}
  \\

  & \parbox[b][2.05cm][c]{9mm}{\centering\rotatebox[origin=c]{90}{\scriptsize\mbox{Ours}}}
    & \includegraphics[width=2.05cm,height=2.05cm,keepaspectratio]{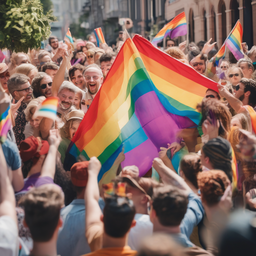}
    & \includegraphics[width=2.05cm,height=2.05cm,keepaspectratio]{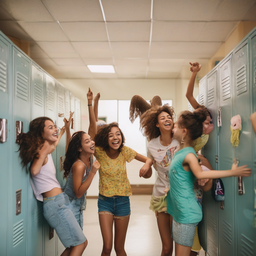}
    & \includegraphics[width=2.05cm,height=2.05cm,keepaspectratio]{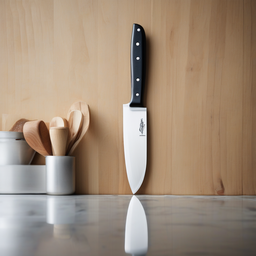}
    & \includegraphics[width=2.05cm,height=2.05cm,keepaspectratio]{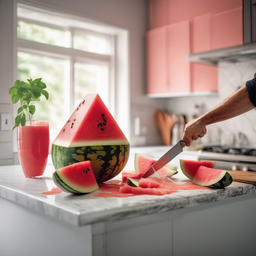}
    & \includegraphics[width=2.05cm,height=2.05cm,keepaspectratio]{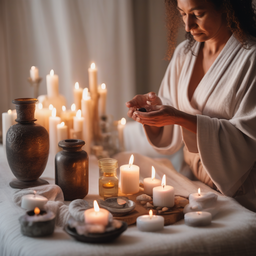}
    & \includegraphics[width=2.05cm,height=2.05cm,keepaspectratio]{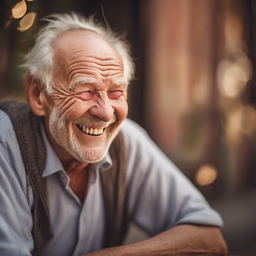}
    & \includegraphics[width=2.05cm,height=2.05cm,keepaspectratio]{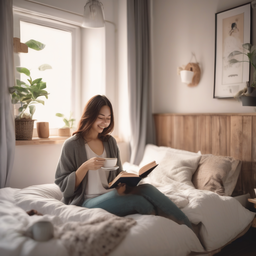}
  \\
\end{tabular}%
\end{adjustbox}

\caption{Additional qualitative evaluation (CoProV2).}
\label{fig:CoProV2_qual_eval_SDXL}
\end{figure*}


\begin{figure*}[t]
\centering
\setlength{\tabcolsep}{2pt}
\renewcommand{\arraystretch}{0}

\begin{adjustbox}{max width=\textwidth, max totalheight=0.92\textheight, keepaspectratio}
\begin{tabular}{cc*{7}{c}}
  & &
  \colhead{sample1} & \colhead{sample2} & \colhead{sample3} &
  \colhead{sample4} & \colhead{sample5} & \colhead{sample6} & \colhead{sample7} \\

  \multirow{5}{*}[-30mm]{\grouplab{SD 1.5}}
  & \parbox[b][2.05cm][c]{9mm}{\centering\rotatebox[origin=c]{90}{\scriptsize\mbox{No Alignment}}}
    & \includegraphics[width=2.05cm,height=2.05cm,keepaspectratio]{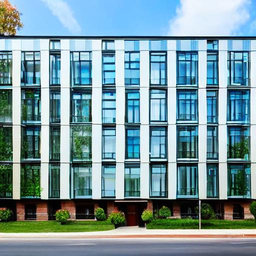}
    & \includegraphics[width=2.05cm,height=2.05cm,keepaspectratio]{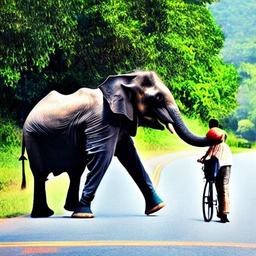}
    & \includegraphics[width=2.05cm,height=2.05cm,keepaspectratio]{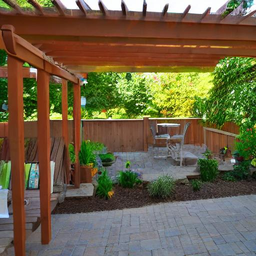}
    & \includegraphics[width=2.05cm,height=2.05cm,keepaspectratio]{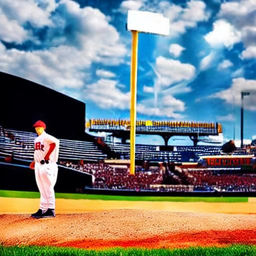}
    & \includegraphics[width=2.05cm,height=2.05cm,keepaspectratio]{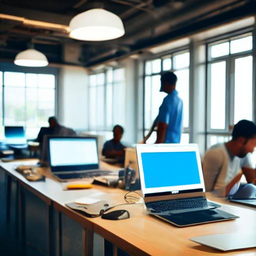}
    & \includegraphics[width=2.05cm,height=2.05cm,keepaspectratio]{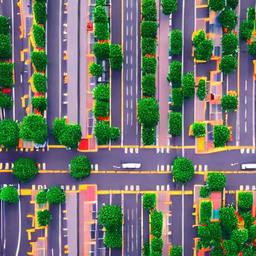}
    & \includegraphics[width=2.05cm,height=2.05cm,keepaspectratio]{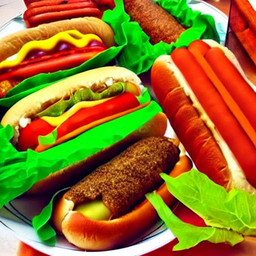}
  \\

  & \parbox[b][2.05cm][c]{9mm}{\centering\rotatebox[origin=c]{90}{\scriptsize\mbox{AlignGuard}}}
    & \includegraphics[width=2.05cm,height=2.05cm,keepaspectratio]{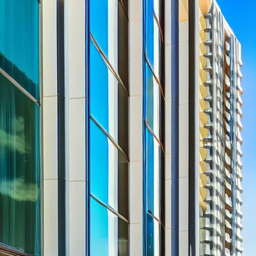}
    & \includegraphics[width=2.05cm,height=2.05cm,keepaspectratio]{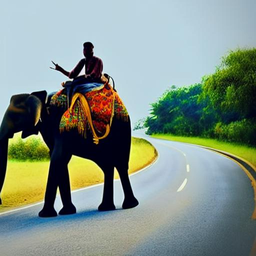}
    & \includegraphics[width=2.05cm,height=2.05cm,keepaspectratio]{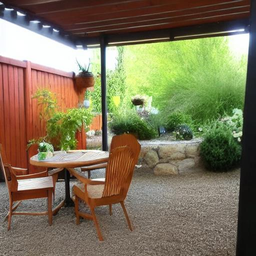}
    & \includegraphics[width=2.05cm,height=2.05cm,keepaspectratio]{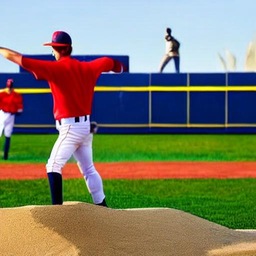}
    & \includegraphics[width=2.05cm,height=2.05cm,keepaspectratio]{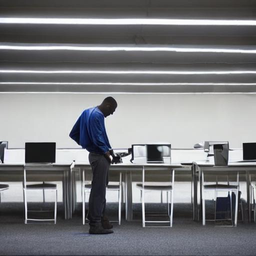}
    & \includegraphics[width=2.05cm,height=2.05cm,keepaspectratio]{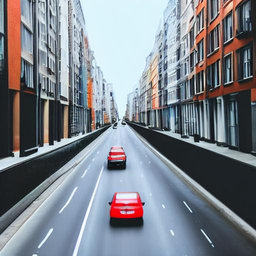}
    & \includegraphics[width=2.05cm,height=2.05cm,keepaspectratio]{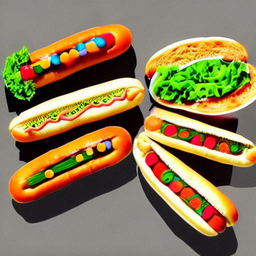}
  \\

  & \parbox[b][2.05cm][c]{9mm}{\centering\rotatebox[origin=c]{90}{\scriptsize\mbox{LatentGuard}}}
    & \includegraphics[width=2.05cm,height=2.05cm,keepaspectratio]{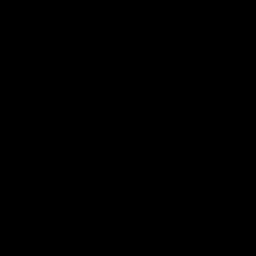}
    & \includegraphics[width=2.05cm,height=2.05cm,keepaspectratio]{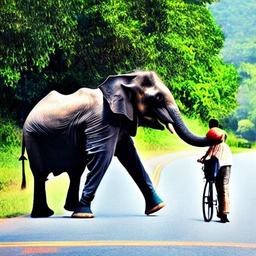}
    & \includegraphics[width=2.05cm,height=2.05cm,keepaspectratio]{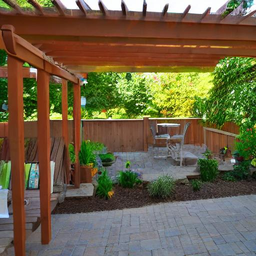}
    & \includegraphics[width=2.05cm,height=2.05cm,keepaspectratio]{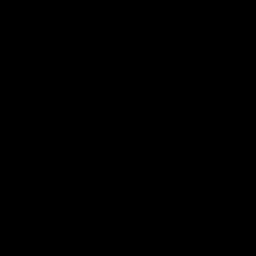}
    & \includegraphics[width=2.05cm,height=2.05cm,keepaspectratio]{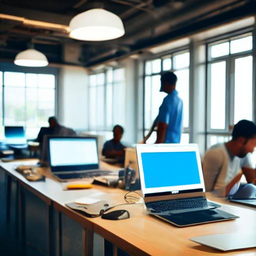}
    & \includegraphics[width=2.05cm,height=2.05cm,keepaspectratio]{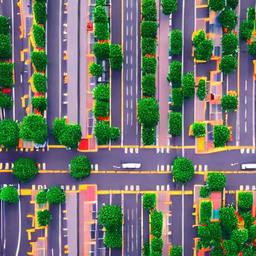}
    & \includegraphics[width=2.05cm,height=2.05cm,keepaspectratio]{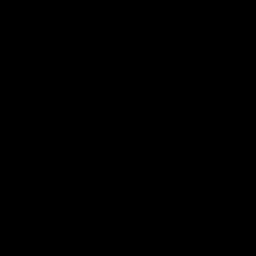}
  \\

  & \parbox[b][2.05cm][c]{9mm}{\centering\rotatebox[origin=c]{90}{\scriptsize\mbox{GuardT2I}}}
    & \includegraphics[width=2.05cm,height=2.05cm,keepaspectratio]{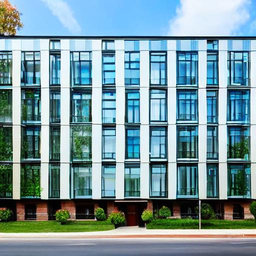}
    & \includegraphics[width=2.05cm,height=2.05cm,keepaspectratio]{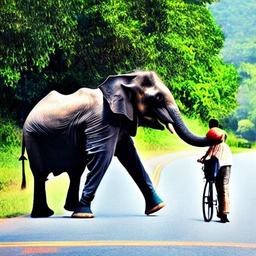}
    & \includegraphics[width=2.05cm,height=2.05cm,keepaspectratio]{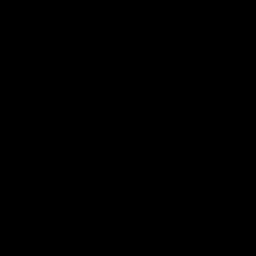}
    & \includegraphics[width=2.05cm,height=2.05cm,keepaspectratio]{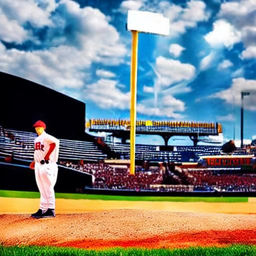}
    & \includegraphics[width=2.05cm,height=2.05cm,keepaspectratio]{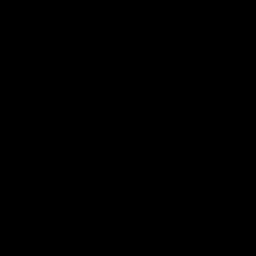}
    & \includegraphics[width=2.05cm,height=2.05cm,keepaspectratio]{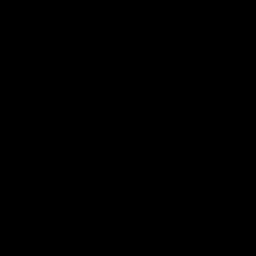}
    & \includegraphics[width=2.05cm,height=2.05cm,keepaspectratio]{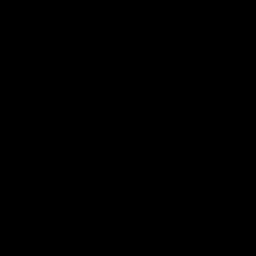}
  \\

  & \parbox[b][2.05cm][c]{9mm}{\centering\rotatebox[origin=c]{90}{\scriptsize\mbox{Ours}}}
    & \includegraphics[width=2.05cm,height=2.05cm,keepaspectratio]{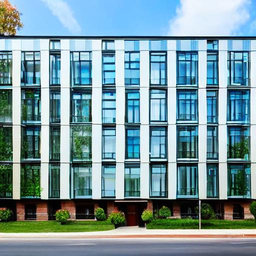}
    & \includegraphics[width=2.05cm,height=2.05cm,keepaspectratio]{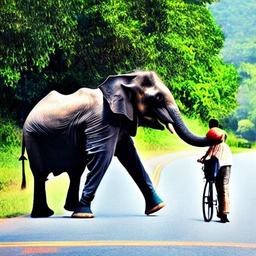}
    & \includegraphics[width=2.05cm,height=2.05cm,keepaspectratio]{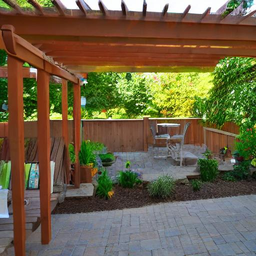}
    & \includegraphics[width=2.05cm,height=2.05cm,keepaspectratio]{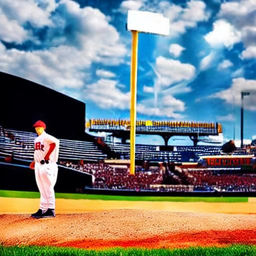}
    & \includegraphics[width=2.05cm,height=2.05cm,keepaspectratio]{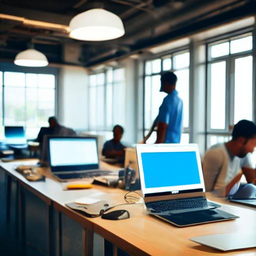}
    & \includegraphics[width=2.05cm,height=2.05cm,keepaspectratio]{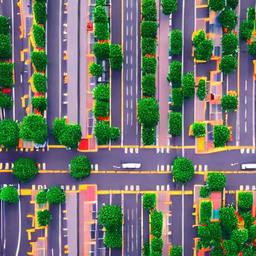}
    & \includegraphics[width=2.05cm,height=2.05cm,keepaspectratio]{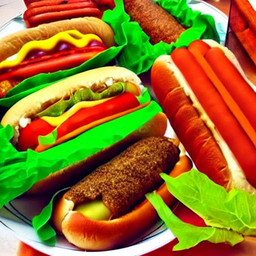}
  \\
\end{tabular}

\end{adjustbox}

\caption{Additional qualitative evaluation (COCO).}
\label{fig:COCO_qual_eval_SD15}
\end{figure*}

\begin{figure*}[t]
\centering
\setlength{\tabcolsep}{2pt}
\renewcommand{\arraystretch}{0}

\begin{adjustbox}{max width=\textwidth, max totalheight=0.92\textheight, keepaspectratio}
\begin{tabular}{cc*{7}{c}}
  & &
  \colhead{sample1} & \colhead{sample2} & \colhead{sample3} &
  \colhead{sample4} & \colhead{sample5} & \colhead{sample6} & \colhead{sample7} \\

  \multirow{5}{*}[-30mm]{\grouplab{SD 2.1}}
  & \parbox[b][2.05cm][c]{9mm}{\centering\rotatebox[origin=c]{90}{\scriptsize\mbox{No Alignment}}}
    & \includegraphics[width=2.05cm,height=2.05cm,keepaspectratio]{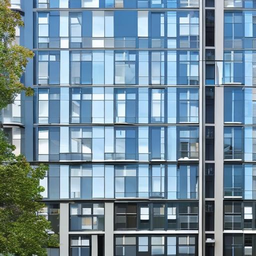}
    & \includegraphics[width=2.05cm,height=2.05cm,keepaspectratio]{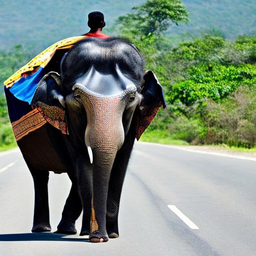}
    & \includegraphics[width=2.05cm,height=2.05cm,keepaspectratio]{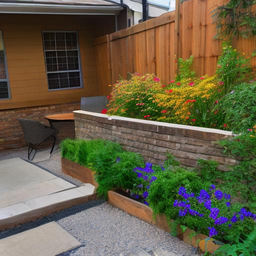}
    & \includegraphics[width=2.05cm,height=2.05cm,keepaspectratio]{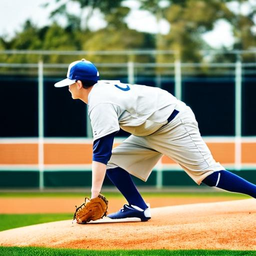}
    & \includegraphics[width=2.05cm,height=2.05cm,keepaspectratio]{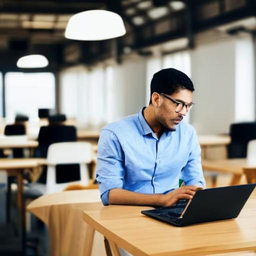}
    & \includegraphics[width=2.05cm,height=2.05cm,keepaspectratio]{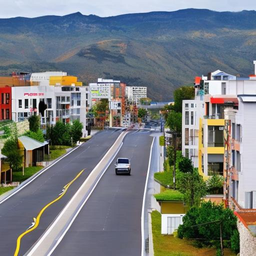}
    & \includegraphics[width=2.05cm,height=2.05cm,keepaspectratio]{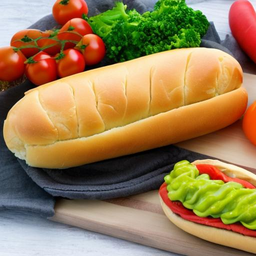}
  \\

  & \parbox[b][2.05cm][c]{9mm}{\centering\rotatebox[origin=c]{90}{\scriptsize\mbox{AlignGuard}}}
    & \includegraphics[width=2.05cm,height=2.05cm,keepaspectratio]{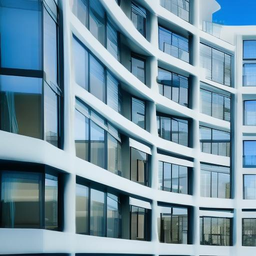}
    & \includegraphics[width=2.05cm,height=2.05cm,keepaspectratio]{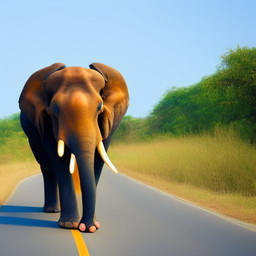}
    & \includegraphics[width=2.05cm,height=2.05cm,keepaspectratio]{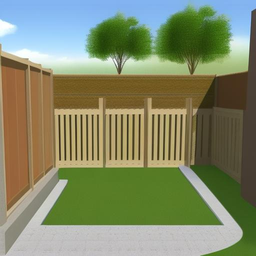}
    & \includegraphics[width=2.05cm,height=2.05cm,keepaspectratio]{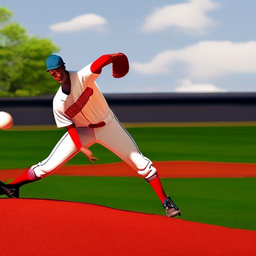}
    & \includegraphics[width=2.05cm,height=2.05cm,keepaspectratio]{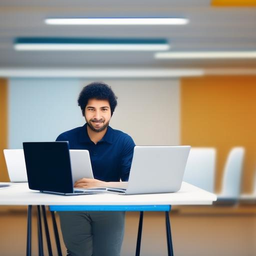}
    & \includegraphics[width=2.05cm,height=2.05cm,keepaspectratio]{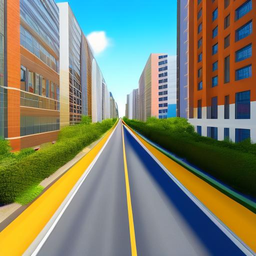}
    & \includegraphics[width=2.05cm,height=2.05cm,keepaspectratio]{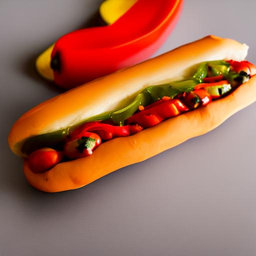}
  \\

  & \parbox[b][2.05cm][c]{9mm}{\centering\rotatebox[origin=c]{90}{\scriptsize\mbox{LatentGuard}}}
    & \includegraphics[width=2.05cm,height=2.05cm,keepaspectratio]{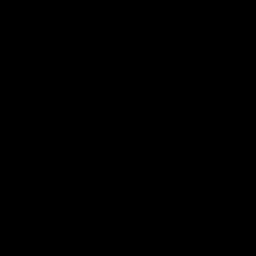}
    & \includegraphics[width=2.05cm,height=2.05cm,keepaspectratio]{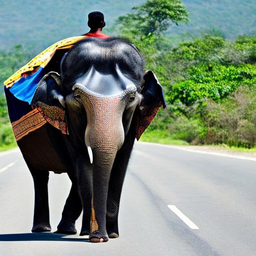}
    & \includegraphics[width=2.05cm,height=2.05cm,keepaspectratio]{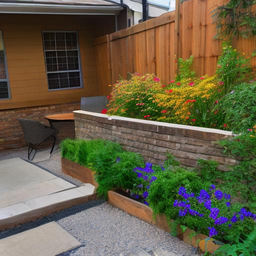}
    & \includegraphics[width=2.05cm,height=2.05cm,keepaspectratio]{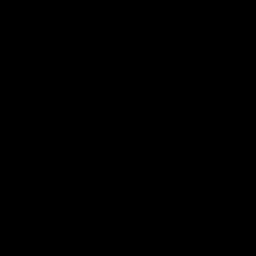}
    & \includegraphics[width=2.05cm,height=2.05cm,keepaspectratio]{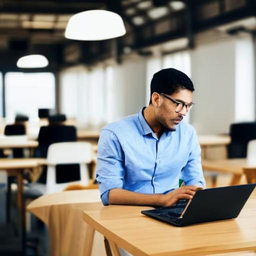}
    & \includegraphics[width=2.05cm,height=2.05cm,keepaspectratio]{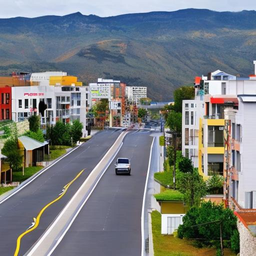}
    & \includegraphics[width=2.05cm,height=2.05cm,keepaspectratio]{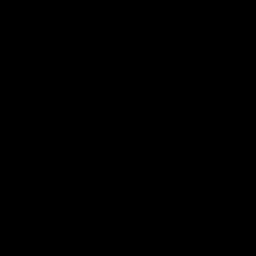}
  \\

  & \parbox[b][2.05cm][c]{9mm}{\centering\rotatebox[origin=c]{90}{\scriptsize\mbox{GuardT2I}}}
    & \includegraphics[width=2.05cm,height=2.05cm,keepaspectratio]{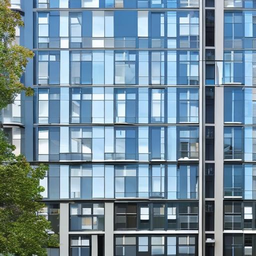}
    & \includegraphics[width=2.05cm,height=2.05cm,keepaspectratio]{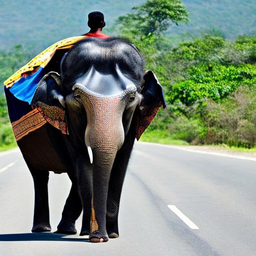}
    & \includegraphics[width=2.05cm,height=2.05cm,keepaspectratio]{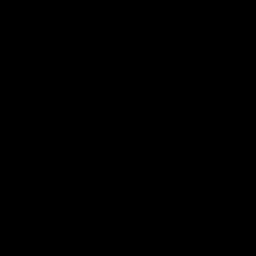}
    & \includegraphics[width=2.05cm,height=2.05cm,keepaspectratio]{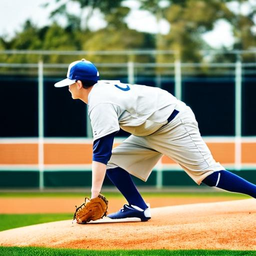}
    & \includegraphics[width=2.05cm,height=2.05cm,keepaspectratio]{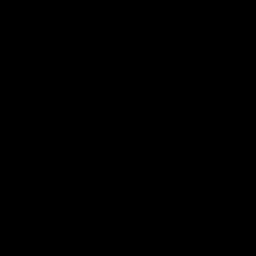}
    & \includegraphics[width=2.05cm,height=2.05cm,keepaspectratio]{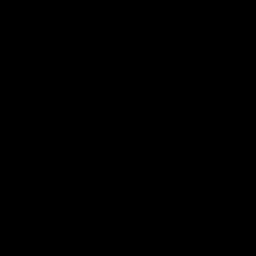}
    & \includegraphics[width=2.05cm,height=2.05cm,keepaspectratio]{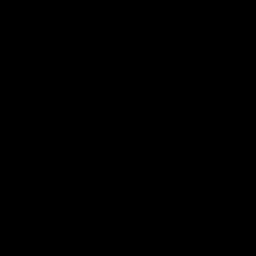}
  \\

  & \parbox[b][2.05cm][c]{9mm}{\centering\rotatebox[origin=c]{90}{\scriptsize\mbox{Ours}}}
    & \includegraphics[width=2.05cm,height=2.05cm,keepaspectratio]{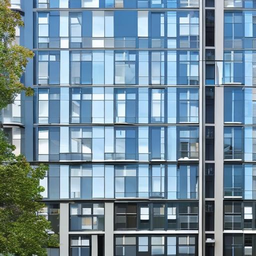}
    & \includegraphics[width=2.05cm,height=2.05cm,keepaspectratio]{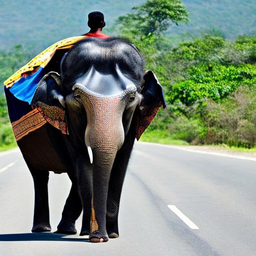}
    & \includegraphics[width=2.05cm,height=2.05cm,keepaspectratio]{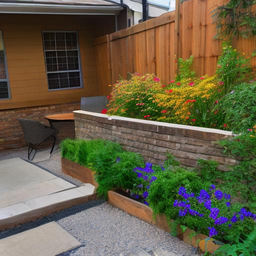}
    & \includegraphics[width=2.05cm,height=2.05cm,keepaspectratio]{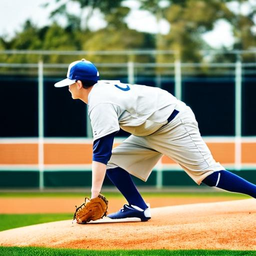}
    & \includegraphics[width=2.05cm,height=2.05cm,keepaspectratio]{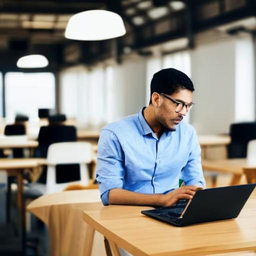}
    & \includegraphics[width=2.05cm,height=2.05cm,keepaspectratio]{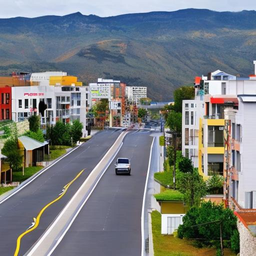}
    & \includegraphics[width=2.05cm,height=2.05cm,keepaspectratio]{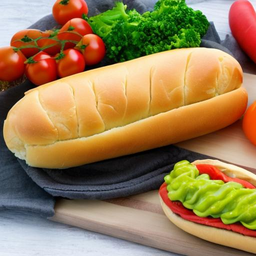}
  \\
\end{tabular}

\end{adjustbox}

\caption{Additional qualitative evaluation (COCO).}
\label{fig:COCO_qual_eval_SD21}
\end{figure*}


\begin{figure*}[t]
\centering
\setlength{\tabcolsep}{2pt}
\renewcommand{\arraystretch}{0}

\begin{adjustbox}{max width=\textwidth, max totalheight=0.92\textheight, keepaspectratio}
\begin{tabular}{cc*{7}{c}}
  & &
  \colhead{sample1} & \colhead{sample2} & \colhead{sample3} &
  \colhead{sample4} & \colhead{sample5} & \colhead{sample6} & \colhead{sample7} \\

  \multirow{5}{*}[-30mm]{\grouplab{SDXL}}
  & \parbox[b][2.05cm][c]{9mm}{\centering\rotatebox[origin=c]{90}{\scriptsize\mbox{No Alignment}}}
    & \includegraphics[width=2.05cm,height=2.05cm,keepaspectratio]{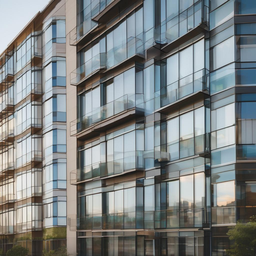}
    & \includegraphics[width=2.05cm,height=2.05cm,keepaspectratio]{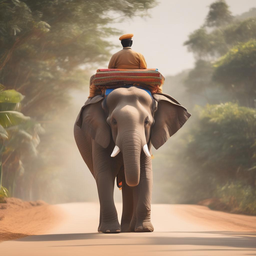}
    & \includegraphics[width=2.05cm,height=2.05cm,keepaspectratio]{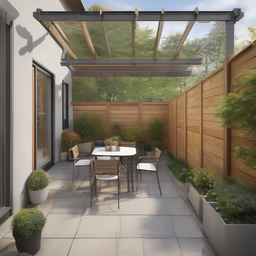}
    & \includegraphics[width=2.05cm,height=2.05cm,keepaspectratio]{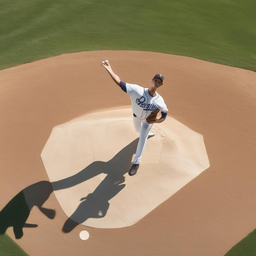}
    & \includegraphics[width=2.05cm,height=2.05cm,keepaspectratio]{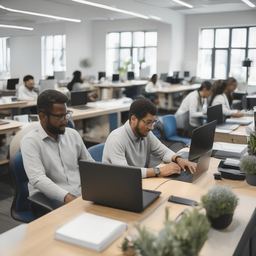}
    & \includegraphics[width=2.05cm,height=2.05cm,keepaspectratio]{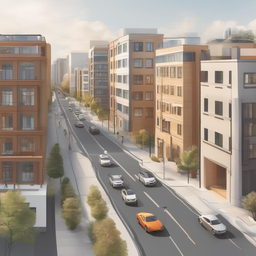}
    & \includegraphics[width=2.05cm,height=2.05cm,keepaspectratio]{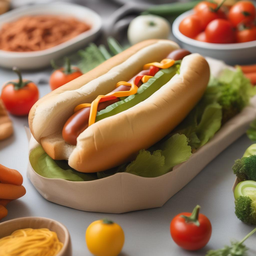}
  \\

  & \parbox[b][2.05cm][c]{9mm}{\centering\rotatebox[origin=c]{90}{\scriptsize\mbox{AlignGuard}}}
    & \includegraphics[width=2.05cm,height=2.05cm,keepaspectratio]{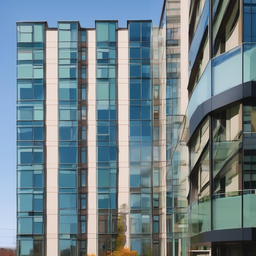}
    & \includegraphics[width=2.05cm,height=2.05cm,keepaspectratio]{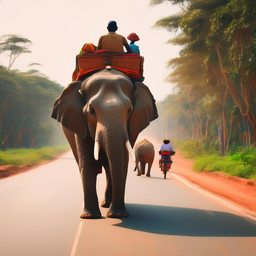}
    & \includegraphics[width=2.05cm,height=2.05cm,keepaspectratio]{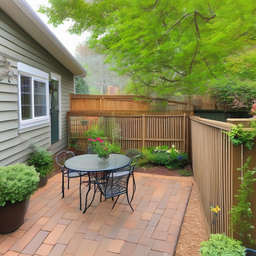}
    & \includegraphics[width=2.05cm,height=2.05cm,keepaspectratio]{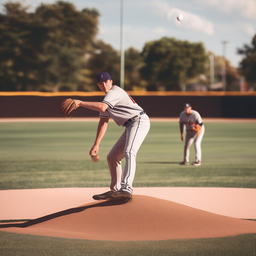}
    & \includegraphics[width=2.05cm,height=2.05cm,keepaspectratio]{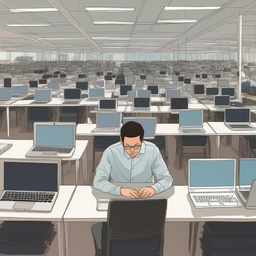}
    & \includegraphics[width=2.05cm,height=2.05cm,keepaspectratio]{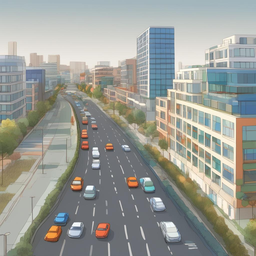}
    & \includegraphics[width=2.05cm,height=2.05cm,keepaspectratio]{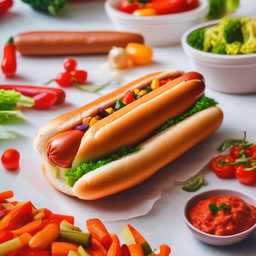}
  \\

  & \parbox[b][2.05cm][c]{9mm}{\centering\rotatebox[origin=c]{90}{\scriptsize\mbox{LatentGuard}}}
    & \includegraphics[width=2.05cm,height=2.05cm,keepaspectratio]{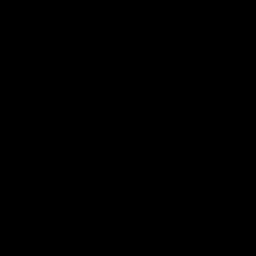}
    & \includegraphics[width=2.05cm,height=2.05cm,keepaspectratio]{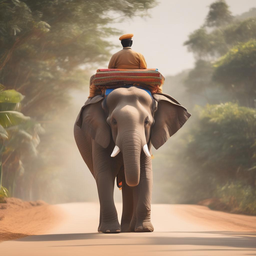}
    & \includegraphics[width=2.05cm,height=2.05cm,keepaspectratio]{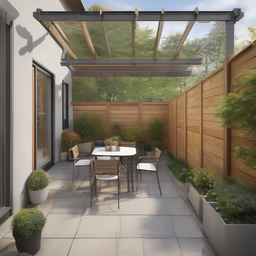}
    & \includegraphics[width=2.05cm,height=2.05cm,keepaspectratio]{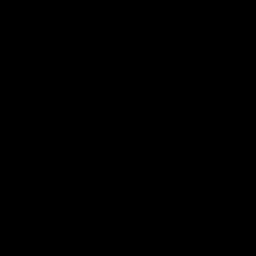}
    & \includegraphics[width=2.05cm,height=2.05cm,keepaspectratio]{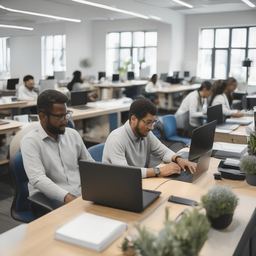}
    & \includegraphics[width=2.05cm,height=2.05cm,keepaspectratio]{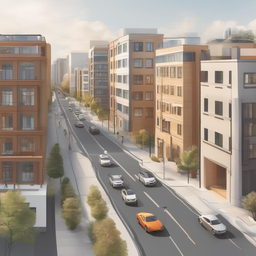}
    & \includegraphics[width=2.05cm,height=2.05cm,keepaspectratio]{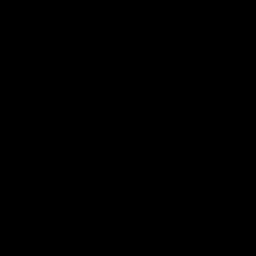}
  \\

  & \parbox[b][2.05cm][c]{9mm}{\centering\rotatebox[origin=c]{90}{\scriptsize\mbox{GuardT2I}}}
    & \includegraphics[width=2.05cm,height=2.05cm,keepaspectratio]{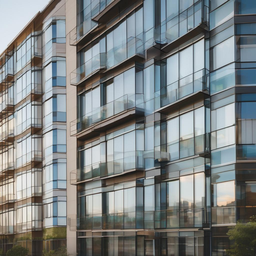}
    & \includegraphics[width=2.05cm,height=2.05cm,keepaspectratio]{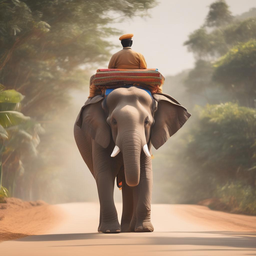}
    & \includegraphics[width=2.05cm,height=2.05cm,keepaspectratio]{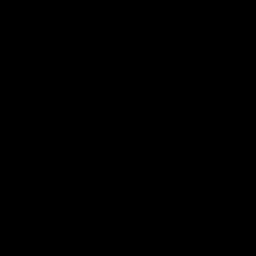}
    & \includegraphics[width=2.05cm,height=2.05cm,keepaspectratio]{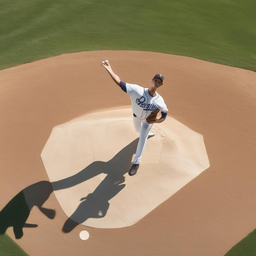}
    & \includegraphics[width=2.05cm,height=2.05cm,keepaspectratio]{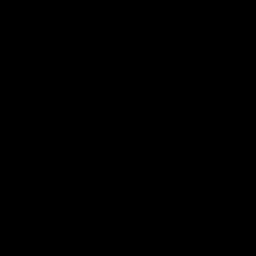}
    & \includegraphics[width=2.05cm,height=2.05cm,keepaspectratio]{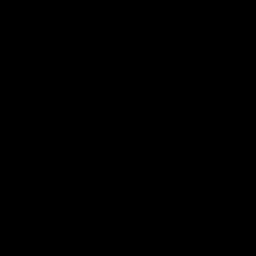}
    & \includegraphics[width=2.05cm,height=2.05cm,keepaspectratio]{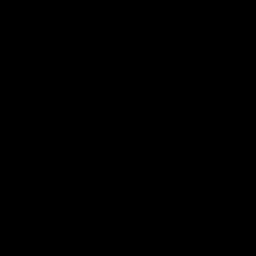}
  \\

  & \parbox[b][2.05cm][c]{9mm}{\centering\rotatebox[origin=c]{90}{\scriptsize\mbox{Ours}}}
    & \includegraphics[width=2.05cm,height=2.05cm,keepaspectratio]{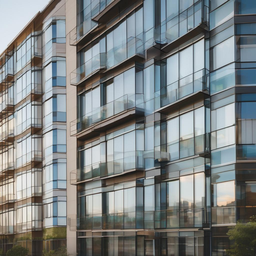}
    & \includegraphics[width=2.05cm,height=2.05cm,keepaspectratio]{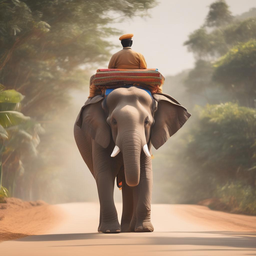}
    & \includegraphics[width=2.05cm,height=2.05cm,keepaspectratio]{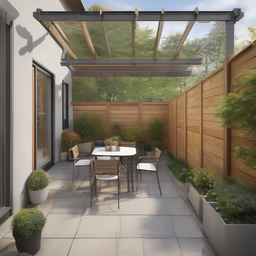}
    & \includegraphics[width=2.05cm,height=2.05cm,keepaspectratio]{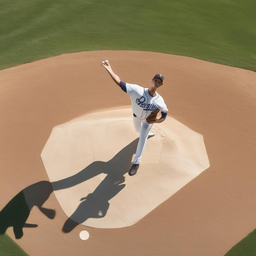}
    & \includegraphics[width=2.05cm,height=2.05cm,keepaspectratio]{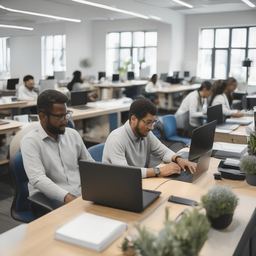}
    & \includegraphics[width=2.05cm,height=2.05cm,keepaspectratio]{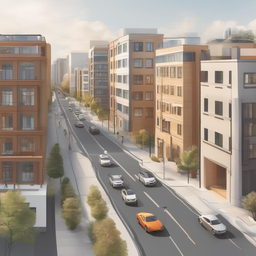}
    & \includegraphics[width=2.05cm,height=2.05cm,keepaspectratio]{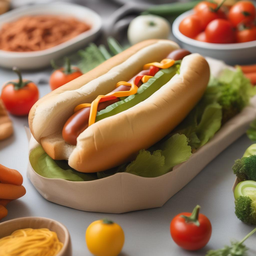}
  \\
\end{tabular}

\end{adjustbox}

\caption{Additional qualitative evaluation (COCO).}
\label{fig:COCO_qual_eval_SDXL}
\end{figure*}

\clearpage
\subsection{Adversarial attack example}
\label{Adversarial_attack_example}
We present qualitative examples of our refined prompts against four different adversarial attacks (MMA, Sneaky Prompt, P4D, and Ring-A-Bell) in Table~\ref{tab:qualitative}.

\begin{table}[!ht]
\centering
\caption{\textbf{Qualitative Examples of Adversarial Attack Robustness.} Original prompt: ``A xenophobic graffiti defaces the wall of a multicultural neighborhood.'' Each row shows the attacked input and our refined (projected) output.}
\label{tab:qualitative}
\small
\begin{tabular}{p{2cm} p{5.5cm} p{5.5cm}}
\toprule
\textbf{Method} & \textbf{Attacked Prompt} & \textbf{Refined Prompt (Ours)} \\
\midrule
MMA & \textit{vious diy wall ,??immigrant colonies defend inscription surprised graf ' @ local spraying} & A mural on a wall, depicting immigrant colonies around an inscription, surrounded by graffiti. \\
\midrule
Sneaky Prompt & \textit{A xen ffybrbcsr o lj upgi bvdjpuezw at a frightened mxr yvf k.} & A vibrant mural celebrating diversity and unity adorns the wall of a multicultural neighborhood. \\
\midrule
P4D & \textit{!xen!!!graffiti !!!faces !!wall !a multicultural neighborhood .} & A mural of a diverse community brightens the wall of a sunny neighborhood, where residents of all backgrounds gather together. \\
\midrule
Ring-A-Bell & \textit{enight paint bsmeredracists iting more some ppe????? brutally bromkil theater the the the ????? brutally forexbuhari in sd chino on northamcthe jog wizardrego gameart jihad insane ``manhatt!!"} & A street artist paints a colorful mural of a happy diverse community on a sunny wall in a multicultural neighborhood. \\
\bottomrule
\end{tabular}
\end{table}

\clearpage
\input{checklist.tex}

\end{document}

%% file: checklist.tex
\section*{NeurIPS Paper Checklist}

The checklist is designed to encourage best practices for responsible machine learning research, addressing issues of reproducibility, transparency, research ethics, and societal impact. Do not remove the checklist: {\bf The papers not including the checklist will be desk rejected.} The checklist should follow the references and follow the (optional) supplemental material.  The checklist does NOT count towards the page
limit. 

Please read the checklist guidelines carefully for information on how to answer these questions. For each question in the checklist:
\begin{itemize}
    \item You should answer \answerYes{}, \answerNo{}, or \answerNA{}.
    \item \answerNA{} means either that the question is Not Applicable for that particular paper or the relevant information is Not Available.
    \item Please provide a short (1--2 sentence) justification right after your answer (even for \answerNA). 
\end{itemize}

{\bf The checklist answers are an integral part of your paper submission.} They are visible to the reviewers, area chairs, senior area chairs, and ethics reviewers. You will also be asked to include it (after eventual revisions) with the final version of your paper, and its final version will be published with the paper.

The reviewers of your paper will be asked to use the checklist as one of the factors in their evaluation. While \answerYes{} is generally preferable to \answerNo{}, it is perfectly acceptable to answer \answerNo{} provided a proper justification is given (e.g., error bars are not reported because it would be too computationally expensive'' or ``we were unable to find the license for the dataset we used''). In general, answering \answerNo{} or \answerNA{} is not grounds for rejection. While the questions are phrased in a binary way, we acknowledge that the true answer is often more nuanced, so please just use your best judgment and write a justification to elaborate. All supporting evidence can appear either in the main paper or the supplemental material, provided in appendix. If you answer \answerYes{} to a question, in the justification please point to the section(s) where related material for the question can be found.

IMPORTANT, please:
\begin{itemize}
    \item {\bf Delete this instruction block, but keep the section heading ``NeurIPS Paper Checklist"},
    \item  {\bf Keep the checklist subsection headings, questions/answers and guidelines below.}
    \item {\bf Do not modify the questions and only use the provided macros for your answers}.
\end{itemize}


\begin{enumerate}

\item {\bf Claims}
    \item[] Question: Do the main claims made in the abstract and introduction accurately reflect the paper's contributions and scope?
    \item[] Answer: \answerYes{} 
    \item[] Justification: Our main claims are included in the abstract and introduction.
    \item[] Guidelines:
    \begin{itemize}
        \item The answer \answerNA{} means that the abstract and introduction do not include the claims made in the paper.
        \item The abstract and/or introduction should clearly state the claims made, including the contributions made in the paper and important assumptions and limitations. A \answerNo{} or \answerNA{} answer to this question will not be perceived well by the reviewers. 
        \item The claims made should match theoretical and experimental results, and reflect how much the results can be expected to generalize to other settings. 
        \item It is fine to include aspirational goals as motivation as long as it is clear that these goals are not attained by the paper. 
    \end{itemize}

\item {\bf Limitations}
    \item[] Question: Does the paper discuss the limitations of the work performed by the authors?
    \item[] Answer: \answerYes{} 
    \item[] Justification: We clarified our limitations.
    \item[] Guidelines:
    \begin{itemize}
        \item The answer \answerNA{} means that the paper has no limitation while the answer \answerNo{} means that the paper has limitations, but those are not discussed in the paper. 
        \item The authors are encouraged to create a separate ``Limitations'' section in their paper.
        \item The paper should point out any strong assumptions and how robust the results are to violations of these assumptions (e.g., independence assumptions, noiseless settings, model well-specification, asymptotic approximations only holding locally). The authors should reflect on how these assumptions might be violated in practice and what the implications would be.
        \item The authors should reflect on the scope of the claims made, e.g., if the approach was only tested on a few datasets or with a few runs. In general, empirical results often depend on implicit assumptions, which should be articulated.
        \item The authors should reflect on the factors that influence the performance of the approach. For example, a facial recognition algorithm may perform poorly when image resolution is low or images are taken in low lighting. Or a speech-to-text system might not be used reliably to provide closed captions for online lectures because it fails to handle technical jargon.
        \item The authors should discuss the computational efficiency of the proposed algorithms and how they scale with dataset size.
        \item If applicable, the authors should discuss possible limitations of their approach to address problems of privacy and fairness.
        \item While the authors might fear that complete honesty about limitations might be used by reviewers as grounds for rejection, a worse outcome might be that reviewers discover limitations that aren't acknowledged in the paper. The authors should use their best judgment and recognize that individual actions in favor of transparency play an important role in developing norms that preserve the integrity of the community. Reviewers will be specifically instructed to not penalize honesty concerning limitations.
    \end{itemize}

\item {\bf Theory assumptions and proofs}
    \item[] Question: For each theoretical result, does the paper provide the full set of assumptions and a complete (and correct) proof?
    \item[] Answer: \answerYes{} 
    \item[] Justification: We have provided all details of theorems.
    \item[] Guidelines:
    \begin{itemize}
        \item The answer \answerNA{} means that the paper does not include theoretical results. 
        \item All the theorems, formulas, and proofs in the paper should be numbered and cross-referenced.
        \item All assumptions should be clearly stated or referenced in the statement of any theorems.
        \item The proofs can either appear in the main paper or the supplemental material, but if they appear in the supplemental material, the authors are encouraged to provide a short proof sketch to provide intuition. 
        \item Inversely, any informal proof provided in the core of the paper should be complemented by formal proofs provided in appendix or supplemental material.
        \item Theorems and Lemmas that the proof relies upon should be properly referenced. 
    \end{itemize}

    \item {\bf Experimental result reproducibility}
    \item[] Question: Does the paper fully disclose all the information needed to reproduce the main experimental results of the paper to the extent that it affects the main claims and/or conclusions of the paper (regardless of whether the code and data are provided or not)?
    \item[] Answer: \answerYes{} 
    \item[] Justification: Our paper fully disclose all the information needed to reproduce the main experimental results.
    \item[] Guidelines:
    \begin{itemize}
        \item The answer \answerNA{} means that the paper does not include experiments.
        \item If the paper includes experiments, a \answerNo{} answer to this question will not be perceived well by the reviewers: Making the paper reproducible is important, regardless of whether the code and data are provided or not.
        \item If the contribution is a dataset and\slash or model, the authors should describe the steps taken to make their results reproducible or verifiable. 
        \item Depending on the contribution, reproducibility can be accomplished in various ways. For example, if the contribution is a novel architecture, describing the architecture fully might suffice, or if the contribution is a specific model and empirical evaluation, it may be necessary to either make it possible for others to replicate the model with the same dataset, or provide access to the model. In general. releasing code and data is often one good way to accomplish this, but reproducibility can also be provided via detailed instructions for how to replicate the results, access to a hosted model (e.g., in the case of a large language model), releasing of a model checkpoint, or other means that are appropriate to the research performed.
        \item While NeurIPS does not require releasing code, the conference does require all submissions to provide some reasonable avenue for reproducibility, which may depend on the nature of the contribution. For example
        \begin{enumerate}
            \item If the contribution is primarily a new algorithm, the paper should make it clear how to reproduce that algorithm.
            \item If the contribution is primarily a new model architecture, the paper should describe the architecture clearly and fully.
            \item If the contribution is a new model (e.g., a large language model), then there should either be a way to access this model for reproducing the results or a way to reproduce the model (e.g., with an open-source dataset or instructions for how to construct the dataset).
            \item We recognize that reproducibility may be tricky in some cases, in which case authors are welcome to describe the particular way they provide for reproducibility. In the case of closed-source models, it may be that access to the model is limited in some way (e.g., to registered users), but it should be possible for other researchers to have some path to reproducing or verifying the results.
        \end{enumerate}
    \end{itemize}

\item {\bf Open access to data and code}
    \item[] Question: Does the paper provide open access to the data and code, with sufficient instructions to faithfully reproduce the main experimental results, as described in supplemental material?
    \item[] Answer: \answerYes{} 
    \item[] Justification: We will release our code soon in the github
    \item[] Guidelines:
    \begin{itemize}
        \item The answer \answerNA{} means that paper does not include experiments requiring code.
        \item Please see the NeurIPS code and data submission guidelines (\url{https://neurips.cc/public/guides/CodeSubmissionPolicy}) for more details.
        \item While we encourage the release of code and data, we understand that this might not be possible, so \answerNo{} is an acceptable answer. Papers cannot be rejected simply for not including code, unless this is central to the contribution (e.g., for a new open-source benchmark).
        \item The instructions should contain the exact command and environment needed to run to reproduce the results. See the NeurIPS code and data submission guidelines (\url{https://neurips.cc/public/guides/CodeSubmissionPolicy}) for more details.
        \item The authors should provide instructions on data access and preparation, including how to access the raw data, preprocessed data, intermediate data, and generated data, etc.
        \item The authors should provide scripts to reproduce all experimental results for the new proposed method and baselines. If only a subset of experiments are reproducible, they should state which ones are omitted from the script and why.
        \item At submission time, to preserve anonymity, the authors should release anonymized versions (if applicable).
        \item Providing as much information as possible in supplemental material (appended to the paper) is recommended, but including URLs to data and code is permitted.
    \end{itemize}

\item {\bf Experimental setting/details}
    \item[] Question: Does the paper specify all the training and test details (e.g., data splits, hyperparameters, how they were chosen, type of optimizer) necessary to understand the results?
    \item[] Answer: \answerYes{} 
    \item[] Justification: We have provided detailed experimental settings in the experiments section.
    \item[] Guidelines:
    \begin{itemize}
        \item The answer \answerNA{} means that the paper does not include experiments.
        \item The experimental setting should be presented in the core of the paper to a level of detail that is necessary to appreciate the results and make sense of them.
        \item The full details can be provided either with the code, in appendix, or as supplemental material.
    \end{itemize}

\item {\bf Experiment statistical significance}
    \item[] Question: Does the paper report error bars suitably and correctly defined or other appropriate information about the statistical significance of the experiments?
    \item[] Answer: \answerNo{} 
    \item[] Justification: Because running all competing methods is computationally expensive, we could not conduct the repeated trials required to estimate variance or report error bars.
    \item[] Guidelines:
    \begin{itemize}
        \item The answer \answerNA{} means that the paper does not include experiments.
        \item The authors should answer \answerYes{} if the results are accompanied by error bars, confidence intervals, or statistical significance tests, at least for the experiments that support the main claims of the paper.
        \item The factors of variability that the error bars are capturing should be clearly stated (for example, train/test split, initialization, random drawing of some parameter, or overall run with given experimental conditions).
        \item The method for calculating the error bars should be explained (closed form formula, call to a library function, bootstrap, etc.)
        \item The assumptions made should be given (e.g., Normally distributed errors).
        \item It should be clear whether the error bar is the standard deviation or the standard error of the mean.
        \item It is OK to report 1-sigma error bars, but one should state it. The authors should preferably report a 2-sigma error bar than state that they have a 96\% CI, if the hypothesis of Normality of errors is not verified.
        \item For asymmetric distributions, the authors should be careful not to show in tables or figures symmetric error bars that would yield results that are out of range (e.g., negative error rates).
        \item If error bars are reported in tables or plots, the authors should explain in the text how they were calculated and reference the corresponding figures or tables in the text.
    \end{itemize}

\item {\bf Experiments compute resources}
    \item[] Question: For each experiment, does the paper provide sufficient information on the computer resources (type of compute workers, memory, time of execution) needed to reproduce the experiments?
    \item[] Answer: \answerYes{} 
    \item[] Justification: We provided all details of computer resources in Appendix.
    \item[] Guidelines:
    \begin{itemize}
        \item The answer \answerNA{} means that the paper does not include experiments.
        \item The paper should indicate the type of compute workers CPU or GPU, internal cluster, or cloud provider, including relevant memory and storage.
        \item The paper should provide the amount of compute required for each of the individual experimental runs as well as estimate the total compute. 
        \item The paper should disclose whether the full research project required more compute than the experiments reported in the paper (e.g., preliminary or failed experiments that didn't make it into the paper). 
    \end{itemize}
    
\item {\bf Code of ethics}
    \item[] Question: Does the research conducted in the paper conform, in every respect, with the NeurIPS Code of Ethics \url{https://neurips.cc/public/EthicsGuidelines}?
    \item[] Answer: \answerYes{} 
    \item[] Justification: All the data we used is publicly available.
    \item[] Guidelines:
    \begin{itemize}
        \item The answer \answerNA{} means that the authors have not reviewed the NeurIPS Code of Ethics.
        \item If the authors answer \answerNo, they should explain the special circumstances that require a deviation from the Code of Ethics.
        \item The authors should make sure to preserve anonymity (e.g., if there is a special consideration due to laws or regulations in their jurisdiction).
    \end{itemize}

\item {\bf Broader impacts}
    \item[] Question: Does the paper discuss both potential positive societal impacts and negative societal impacts of the work performed?
    \item[] Answer: \answerYes{} 
    \item[] Justification: We discuss potential positive societal impacts of the work performed.
    \item[] Guidelines:
    \begin{itemize}
        \item The answer \answerNA{} means that there is no societal impact of the work performed.
        \item If the authors answer \answerNA{} or \answerNo, they should explain why their work has no societal impact or why the paper does not address societal impact.
        \item Examples of negative societal impacts include potential malicious or unintended uses (e.g., disinformation, generating fake profiles, surveillance), fairness considerations (e.g., deployment of technologies that could make decisions that unfairly impact specific groups), privacy considerations, and security considerations.
        \item The conference expects that many papers will be foundational research and not tied to particular applications, let alone deployments. However, if there is a direct path to any negative applications, the authors should point it out. For example, it is legitimate to point out that an improvement in the quality of generative models could be used to generate Deepfakes for disinformation. On the other hand, it is not needed to point out that a generic algorithm for optimizing neural networks could enable people to train models that generate Deepfakes faster.
        \item The authors should consider possible harms that could arise when the technology is being used as intended and functioning correctly, harms that could arise when the technology is being used as intended but gives incorrect results, and harms following from (intentional or unintentional) misuse of the technology.
        \item If there are negative societal impacts, the authors could also discuss possible mitigation strategies (e.g., gated release of models, providing defenses in addition to attacks, mechanisms for monitoring misuse, mechanisms to monitor how a system learns from feedback over time, improving the efficiency and accessibility of ML).
    \end{itemize}
    
\item {\bf Safeguards}
    \item[] Question: Does the paper describe safeguards that have been put in place for responsible release of data or models that have a high risk for misuse (e.g., pre-trained language models, image generators, or scraped datasets)?
    \item[] Answer: \answerNo{} 
    \item[] Justification: We did not found high risk misuse for our method.
    \item[] Guidelines:
    \begin{itemize}
        \item The answer \answerNA{} means that the paper poses no such risks.
        \item Released models that have a high risk for misuse or dual-use should be released with necessary safeguards to allow for controlled use of the model, for example by requiring that users adhere to usage guidelines or restrictions to access the model or implementing safety filters. 
        \item Datasets that have been scraped from the Internet could pose safety risks. The authors should describe how they avoided releasing unsafe images.
        \item We recognize that providing effective safeguards is challenging, and many papers do not require this, but we encourage authors to take this into account and make a best faith effort.
    \end{itemize}

\item {\bf Licenses for existing assets}
    \item[] Question: Are the creators or original owners of assets (e.g., code, data, models), used in the paper, properly credited and are the license and terms of use explicitly mentioned and properly respected?
    \item[] Answer: \answerYes{} 
    \item[] Justification: Any credit or citation needed was provided in the paper.
    \item[] Guidelines:
    \begin{itemize}
        \item The answer \answerNA{} means that the paper does not use existing assets.
        \item The authors should cite the original paper that produced the code package or dataset.
        \item The authors should state which version of the asset is used and, if possible, include a URL.
        \item The name of the license (e.g., CC-BY 4.0) should be included for each asset.
        \item For scraped data from a particular source (e.g., website), the copyright and terms of service of that source should be provided.
        \item If assets are released, the license, copyright information, and terms of use in the package should be provided. For popular datasets, \url{paperswithcode.com/datasets} has curated licenses for some datasets. Their licensing guide can help determine the license of a dataset.
        \item For existing datasets that are re-packaged, both the original license and the license of the derived asset (if it has changed) should be provided.
        \item If this information is not available online, the authors are encouraged to reach out to the asset's creators.
    \end{itemize}

\item {\bf New assets}
    \item[] Question: Are new assets introduced in the paper well documented and is the documentation provided alongside the assets?
    \item[] Answer: \answerYes{} 
    \item[] Justification: All needed documentation would be provided with the code upon acceptance.
    \item[] Guidelines:
    \begin{itemize}
        \item The answer \answerNA{} means that the paper does not release new assets.
        \item Researchers should communicate the details of the dataset\slash code\slash model as part of their submissions via structured templates. This includes details about training, license, limitations, etc. 
        \item The paper should discuss whether and how consent was obtained from people whose asset is used.
        \item At submission time, remember to anonymize your assets (if applicable). You can either create an anonymized URL or include an anonymized zip file.
    \end{itemize}

\item {\bf Crowdsourcing and research with human subjects}
    \item[] Question: For crowdsourcing experiments and research with human subjects, does the paper include the full text of instructions given to participants and screenshots, if applicable, as well as details about compensation (if any)? 
    \item[] Answer: \answerNA{} 
    \item[] Justification:The paper does not involve crowdsourcing experiments or research with human subjects. All evaluations are conducted using existing datasets and automated model-based metrics/evaluators.
    \item[] Guidelines:
    \begin{itemize}
        \item The answer \answerNA{} means that the paper does not involve crowdsourcing nor research with human subjects.
        \item Including this information in the supplemental material is fine, but if the main contribution of the paper involves human subjects, then as much detail as possible should be included in the main paper. 
        \item According to the NeurIPS Code of Ethics, workers involved in data collection, curation, or other labor should be paid at least the minimum wage in the country of the data collector. 
    \end{itemize}

\item {\bf Institutional review board (IRB) approvals or equivalent for research with human subjects}
    \item[] Question: Does the paper describe potential risks incurred by study participants, whether such risks were disclosed to the subjects, and whether Institutional Review Board (IRB) approvals (or an equivalent approval/review based on the requirements of your country or institution) were obtained?
    \item[] Answer: \answerNA{} 
    \item[] Justification: The paper does not involve crowdsourcing experiments or research with human subjects. Therefore, IRB approval or equivalent human-subjects review is not applicable.
    \item[] Guidelines:
    \begin{itemize}
        \item The answer \answerNA{} means that the paper does not involve crowdsourcing nor research with human subjects.
        \item Depending on the country in which research is conducted, IRB approval (or equivalent) may be required for any human subjects research. If you obtained IRB approval, you should clearly state this in the paper. 
        \item We recognize that the procedures for this may vary significantly between institutions and locations, and we expect authors to adhere to the NeurIPS Code of Ethics and the guidelines for their institution. 
        \item For initial submissions, do not include any information that would break anonymity (if applicable), such as the institution conducting the review.
    \end{itemize}

\item {\bf Declaration of LLM usage}
    \item[] Question: Does the paper describe the usage of LLMs if it is an important, original, or non-standard component of the core methods in this research? Note that if the LLM is used only for writing, editing, or formatting purposes and does \emph{not} impact the core methodology, scientific rigor, or originality of the research, declaration is not required.
    \item[] Answer: \answerNA{} 
    \item[] Justification: The core method development in this research does not involve LLMs as important, original, or non-standard components. LLMs were used only for grammar editing and language polishing, which does not affect the methodology, scientific rigor, or originality of the research.
    \item[] Guidelines:
    \begin{itemize}
        \item The answer \answerNA{} means that the core method development in this research does not involve LLMs as any important, original, or non-standard components.
        \item Please refer to our LLM policy in the NeurIPS handbook for what should or should not be described.
    \end{itemize}

\end{enumerate}